\documentclass{article}

\usepackage{booktabs} 

\usepackage{hyperref}

\usepackage[accepted]{icml2019}

\usepackage[table,xcdraw,svgnames]{xcolor}

\usepackage[english]{babel}
\usepackage{amsmath,amsthm,amsfonts,bm}
\usepackage{environ}

\usepackage[utf8]{inputenc} 
\usepackage[T1]{fontenc}    
\usepackage{url}            
\usepackage{nicefrac}       
\usepackage{microtype}      

\renewenvironment{proof}{{\bfseries Proof}}{}

\newtheorem{Theorem}{Theorem}

\newtheorem{Lemma}{Lemma}
\newtheorem{Proposition}[Lemma]{Proposition}
\newtheorem{corollary}[Lemma]{Corollary}

\theoremstyle{definition}

\NewEnviron{medium1}{%
\scalebox{1.1}{\BODY}}
\NewEnviron{medium2}{%
\scalebox{1.05}{\BODY}}

\def\ppc{^{\prime \vss \textnormal{c}}}
\def\pc{^{\textnormal{c}}}

\def\dx{{\dd\rvx}}
\def\dy{{\dd\rvy}}
\def\dz{{\dd\rvz}}

\def\uPHI{{_{{\phi}}}}
\def\uBN{{_{{\BN}}}}
\def\uphi{{_{{\phi}}}}
\def\ubn{^{}_{\BN}}

\def\cbullet{\rvv}

\newcommand{\dd}[1]{\mathop{}\!\mathrm{d}#1}

\def\bigO{{\mathcal{O}}}

\def\eqref#1{equation~\ref{#1}}

\def\1{\bm{1}}

\newcommand{\erfc}{\mathrm{erfc}\,}

\def\reff{{r_\textnormal{eff}}}

\def\udmax{_{2,\textnormal{max}}}
\def\udmin{_{2,\textnormal{min}}}
\def\umax{_\textnormal{max}}
\def\umin{_\textnormal{min}}

\def\largefrac#1#2{\mbox{\large{$\frac{#1}{#2}$}}}
\def\eq#1{Eq.~(\ref{#1})}
\def\Eq#1{Eq.~(\ref{#1})}

\def\fig#1{Fig.~\ref{#1}}

\def\Section#1{Section~\ref{#1}}

\def\theorem#1{Theorem~\ref{#1}}
\def\lemma#1{Lemma~\ref{#1}}

\def\tl{{\theta^l}}

\def\Tl{{\Theta^l}}

\def\Tlm{{\Theta^{l-1}}}
\def\Tlh{{\Theta^{l,h}}}
\def\tlh{{\theta^{l,h}}}
\def\Tlhm{{\Theta^{l,h-1}}}

\def\tlH{{\theta^{l,H}}}

\def\Tkm{{\Theta^{k-1}}}
\def\Tk{{\Theta^k}}
\def\tk{{\theta^k}}

\def\kp{{k^\prime}}

\def\gI{{\mathcal{I}}}

\def\S{{\chi}}

\def\vss{{\hspace{0.5pt}}}

\def\tx{\tilde{\rvx}}
\def\ty{\tilde{\rvy}}
\def\tnu{{\tilde{\vnu}}}

\def\bx{\bar{{\mathbf{x}}}}
\def\by{\bar{{\mathbf{y}}}}
\def\bz{\bar{{\mathbf{z}}}}

\def\bdx{\dd\bar{\rvx}}
\def\bdy{\dd\bar{\rvy}}
\def\bdz{\dd\bar{\rvz}}

\def\plminus{^{l,-}}
\def\plplus{^{l,+}}
\def\db{{\operatorname{dB}}}
\def\BN{{\operatorname{BN}}}
\def\RF{{\operatorname{RF}}}
\def\vec{{\operatorname{vec}}}
\def\snr{{\operatorname{SNR}}}
\def\relu{{\operatorname{ReLU}}}
\def\tanh{{\operatorname{tanh}}}

\def\plm{^{l-1}}

\def\uk{_{k}}
\def\ukN{_{k\in \sN}}
\def\ukm{_{k-1}}

\def\ul{_{l}}
\def\ulm{_{l-1}}

\def\lpar{\left(}
\def\rpar{\right)}

\def\uo{_{0}}
\def\po{^{0}}
\def\pl{^{l}}
\def\pk{^{k}}
\def\pkm{^{k-1}}
\def\plk{^{l,k}}
\def\plmk{^{l-1,k}}
\def\ulk{_{l,k}}
\def\ulmk{_{l-1,k}}
\def\pok{^{0,k}}
\def\plik{^{l,i_k}}

\def\ulik{_{l,i_k}}

\def\umaxk{_{\textnormal{max},k}}
\def\umink{_{\textnormal{min},k}}

\def\pkl{^{k_l}}
\def\ukl{_{k_l}}
\def\pklp{^{k_l+1}}
\def\pklpminus{^{k_l+1,-}}
\def\pklpplus{^{k_l+1,+}}

\def\tklp{{\theta^{k_l+1}}}

\def\pp{^{\prime}}
\def\ppp{^{\prime\prime}}
\def\pppp{^{\prime\prime\prime}}
\def\ppppp{^{\prime\prime\prime\prime}}

\def\plH{^{l,H}}
\def\plHm{^{l,H-1}}

\def\plh{^{l,h}}

\def\plhm{^{l,h-1}}

\def\rf{{\rvr_\textnormal{f}}}
\def\rfc{{\hat{\rvr}_\textnormal{f}}}
\def\fm{{\rvf_\textnormal{m}}}
\def\fmc{{\hat{\rvf}_\textnormal{m}}}

\def\rf{{\rho}}
\def\rfc{{\hat{\rho}}}
\def\fm{{\varphi}}
\def\fmc{{\hat{\varphi}}}

\def\rc{{\textnormal{c}}}

\def\rvb{{\mathbf{b}}}

\def\rvu{{\mathbf{i}}}

\def\rvt{{\mathbf{t}}}
\def\rvu{{\mathbf{u}}}
\def\rvv{{\mathbf{v}}}
\def\rvw{{\mathbf{w}}}
\def\rvx{{\mathbf{x}}}
\def\rvy{{\mathbf{y}}}
\def\rvz{{\mathbf{z}}}

\def\vnu{{\bm{\nu}}}

\def\vomega{{\bm{\omega}}}
\def\vbeta{{\bm{\beta}}}

\def\valpha{{\bm{\alpha}}}
\def\vkappa{{\bm{\kappa}}}

\def\ve{{\bm{e}}}

\def\mC{{\bm{C}}}

\def\mG{{\bm{G}}}

\def\mI{{\bm{I}}}

\def\mW{{\bm{W}}}

\DeclareMathAlphabet{\mathsfit}{\encodingdefault}{\sfdefault}{m}{sl}
\SetMathAlphabet{\mathsfit}{bold}{\encodingdefault}{\sfdefault}{bx}{n}

\def\gI{{\mathcal{I}}}

\def\sN{{\mathbb{N}}}

\def\sP{{\mathbb{P}}}

\def\sR{{\mathbb{R}}}

\newcommand{\X}{\mathcal{X}}
\newcommand{\E}{\mathbb{E}}

\newcommand{\R}{\mathbb{R}}
\newcommand{\Z}{\mathbb{Z}}

\newcommand{\Var}{\mathrm{Var}}

\newcommand{\normltwo}{L^2}

\DeclareMathOperator{\Tr}{Tr}

\usepackage{url}
\usepackage{enumitem}
\usepackage{textcomp}

\usepackage{centernot}
\usepackage{listings}
\usepackage[pdftex]{graphicx} 
\usepackage{caption}
\usepackage{subcaption}
\usepackage{amssymb}
\usepackage{mathtools}

\usepackage{ocgx2}
\usepackage{hyperref}

\newlength\wantedwidth

\icmltitlerunning{Characterizing Well-Behaved vs. Pathological Deep Neural Networks}

\begin{document}

\twocolumn[
\icmltitle{Characterizing Well-Behaved vs. Pathological Deep Neural Networks}



\icmlsetsymbol{equal}{*}

\begin{icmlauthorlist}
\icmlauthor{Antoine Labatie}{labatie}
\end{icmlauthorlist}

\icmlaffiliation{labatie}{Correspondence to}

\icmlcorrespondingauthor{Antoine Labatie}{antoine.labatie@centra-\\liens.net}

\icmlkeywords{Machine Learning, ICML}

\vskip 0.3in
]



\printAffiliationsAndNotice{} 

\begin{abstract}
We introduce a novel approach, requiring only mild assumptions, for the characterization of deep neural networks at initialization. Our approach applies both to fully-connected and convolutional networks and easily incorporates batch normalization and skip-connections. Our key insight is to consider the evolution with depth of statistical moments of signal and noise, thereby characterizing the presence or absence of pathologies in the hypothesis space encoded by the choice of hyperparameters. We establish: (i) for feedforward networks, with and without batch normalization, the multiplicativity of layer composition inevitably leads to ill-behaved moments and pathologies; (ii) for residual networks with batch normalization, on the other hand, skip-connections induce power-law rather than exponential behaviour, leading to well-behaved moments and no pathology.\footnote{Code to reproduce all results in this paper is available at \url{https://github.com/alabatie/moments-dnns}.}
\end{abstract}

\section{Introduction}

The feverish pace of practical applications has led in the recent years to many advances in neural network architectures, initialization and regularization. At the same time, theoretical research has not been able to follow the same pace. In particular, there is still no mature theory able to validate the full choices of hyperparameters leading to state-of-the-art performance. This is unfortunate since such theory could also serve as a guide towards further improvement.

Amidst the research aimed at building this theory, an important branch has focused on neural networks at initialization. Due to the randomness of model parameters at initialization, characterizing neural networks at that time can be seen as characterizing the hypothesis space of input-output mappings that will be favored or reachable during training, i.e. the inductive bias encoded by the choice of hyperparameters. This view has received strong experimental support, with well-behaved input-output mappings at initialization extensively found to be predictive of trainability and post-training performance \citep{Schoenholz17,Yang17,Xiao18,Philipp18b,Yang19a}.

Yet, even this simplifying case of neural networks at initialization is challenging as it notably involves dealing with: (i) the complex interplay of the randomness from input data and from model parameters; (ii) the broad spectrum of potential pathologies; (iii) the finite number of units in each layer; (iv) the difficulty to incorporate convolutional layers, batch normalization and skip-connections. Complexities (i), (ii) typically lead to restricting to specific cases of input data and pathologies, e.g. exploding complexity of data manifolds \citep{Poole16,Raghu17}, exponential correlation or decorrelation of two data points \citep{Schoenholz17,Balduzzi17,Xiao18}, exploding and vanishing gradients \citep{Yang17,Philipp18a,Hanin18a,Yang19a}, exploding and vanishing activations \citep{Hanin18b}. Complexity (iii) commonly leads to making simplifying assumptions, e.g. convergence to Gaussian processes for infinite width \citep{Neal96,Leroux07,Lee18,Matthews18,Borovykh18,Garriga18,Novak19,Yang19b}, “typical” activation patterns \citep{Balduzzi17}. Finally complexity (iv) most often leads to limiting the number of hard-to-model elements incorporated at a time. To the best of our knowledge, all attempts have thus far been limited in either their scope or their simplifying assumptions.

As the first contribution of this paper, we introduce a novel approach for the characterization of deep neural networks at initialization. This approach: (i) offers a unifying treatment of the broad spectrum of pathologies without any restriction on the input data; (ii) requires only mild assumptions; (iii) easily incorporates convolutional layers, batch normalization and skip-connections. 

As the second contribution, we use this approach to characterize deep neural networks with the most common choices of hyperparameters. We identify the multiplicativity of layer composition as the driving force towards pathologies in \mbox{feedforward} networks: either with the neural network having its signal shrunk into a single point or line; or with the neural network behaving as a noise amplifier with sensitivity exploding with depth. In contrast, we identify the combined action of batch normalization and skip-connections as responsible for bypassing this multiplicativity and relieving from pathologies in batch-normalized residual networks.

\section{Propagation}
\label{S:equation_propagation}

We start by formulating the propagation for neural networks with neither batch normalization nor skip-connections, that we refer as \emph{vanilla nets}. We will slightly adapt this formulation in \Section{S:BN_feedforward_nets} with \emph{batch-normalized feedforward nets} and in \Section{S:BN_resnets} with \emph{batch-normalized resnets}. 

\textbf{Clean Propagation.} We first consider a random tensorial input $\rvx \equiv\rvx\po\in \R^{n \times \dots \times n \times N\uo}$, spatially $d$-dimensional with extent $n$ in all spatial dimensions and $N\uo$ channels. This input $\rvx$ is fed into a $d$-dimensional convolutional neural network with periodic boundary conditions, fixed spatial extent $n$, and activation function $\phi$.\footnote{It is possible to relax the assumptions of periodic boundary conditions and constant spatial extent $n$ [\ref{S:relaxing_assumption}]. These assumptions, as well as the assumption of constant width $N\ul$ in \Section{S:BN_resnets}, are only made for simplicity of the analysis.} At each layer $l\geq1$, we denote $N\ul$ the number of channels or \emph{width}, $K\ul$ the convolutional spatial extent, $\rvx\pl, \rvy\pl \in \R^{n \times \dots \times n \times N\ul}$ the post-activations and pre-activations, \mbox{$\vomega\pl \in \R^{K\ul \times \dots \times K\ul \times N\ulm \times N\ul}$} the weights, and $\rvb\pl \in \R^{N\ul}$ the biases. Later in our analysis, the model parameters $\vomega\pl$, $\rvb\pl$ will be considered as random, but for now they are considered as \emph{fixed}. At each layer, the propagation is given by
\begin{align*}
\rvy\pl & = \vomega\pl \ast \rvx\plm + \vbeta\pl, \\
\rvx\pl & = \phi ( \rvy\pl ),
\end{align*}
with $\ast$ the convolution and $\vbeta\pl \in \R^{n \times \dots \times n \times N\ul}$ the tensor with repeated version of $\rvb\pl$ at each spatial position. From now on, we refer to the propagated tensor $\rvx\pl$ as the \emph{signal}. 

\textbf{Noisy Propagation.} To make our setup more realistic, we next suppose that the input signal $\rvx$ is corrupted by an input noise $\dx \equiv \dx\po \in \R^{n \times \dots \times n \times N\uo}$ having small \emph{iid} components such that $\E_{\dx}[ \dx_{i} \dx_{j}]= \sigma^2_{\dx} \delta_{ij}$, with \mbox{$\sigma_{\dx}\ll1$} and $\delta_{ij}$ the Kronecker delta for \mbox{multidimensional indices $i,j$}. We denote $\Phi\ul(\rvx)\equiv\rvx\pl$, with $\Phi\ul$ the neural network mapping from layer $0$ to $l$, and we consider the simultaneous propagation of the signal $\Phi\ul(\rvx)$ and the noise \mbox{$\Phi\ul(\rvx+\dx) - \Phi\ul(\rvx)$}. At each layer, this simultaneous propagation is given at first order by
 \begin{alignat}{4}
\label{E:noisy_propagation1}
\rvy\pl & = \vomega\pl \ast \rvx\plm + \vbeta\pl,& \qquad \dy\pl & = \vomega\pl \ast \dx\plm, \\
\label{E:noisy_propagation2}
\rvx\pl & = \phi ( \rvy\pl ),& \qquad \dx\pl & = \phi\pp ( \rvy\pl ) \odot \dy\pl,\\[-6pt] \nonumber
\end{alignat}

with $\odot$ the element-wise tensor multiplication. The tensor $\dx\pl$ resulting from the simultaneous propagation of $(\rvx\pl,\dx\pl)$ in \eq{E:noisy_propagation1} and \eq{E:noisy_propagation2} approximates arbitrarily well the noise \mbox{$\Phi\ul(\rvx+\dx) - \Phi\ul(\rvx)$} as $\sigma_{\dx}\to0$ [\ref{S:assumption_noise}]. For simplicity, we will keep the terminology of \emph{noise} when referring to $\dx\pl$. 

From \eq{E:noisy_propagation1} and \eq{E:noisy_propagation2}, we see that $\rvx\pl$, $\rvy\pl$ only depend on the input signal $\rvx$, and that $\dx\pl$ depends linearly on the input noise $\dx$ when $\rvx$ is \emph{fixed}. As a consequence, $\dx\pl$ stays centered with respect to $\dx$ such that $\forall \rvx,\valpha,\rc$: $\E_{\dx} [\dx\pl_{\valpha,\rc}]=0$, where from now on the spatial position is denoted as $\alpha$ and the channel as $\rc$. 

\textbf{Scope.} We require two mild assumptions: \mbox{(i) $\rvx$ is not} trivially zero: $\E_{\rvx,\valpha,\rc}[\rvx^2_{\valpha,\rc}]>0$;\footnote{Whenever $\valpha$ and $\rc$ are considered as random variables, they are supposed uniformly sampled among all spatial positions $\{1, \dots, n \}^d$ and all channels $\{1, \dots, N\ul \}$.} (ii) the width $N\ul$ is bounded.

Some results of our analysis will apply for any choice of $\phi$, but unless otherwise stated, we restrict to the most common choice: $\phi(\cdot)\equiv\relu(\cdot)=\max(\cdot,0)$. Even though $\relu$ is not differentiable at $0$, we still define $\dx\pl$ as the result of the simultaneous propagation of $(\rvx\pl,\dx\pl)$ in \eq{E:noisy_propagation1} and \eq{E:noisy_propagation2} with the convention $\phi\pp(0) \equiv 1/2$ [\ref{S:assumption_differentiable}].

Note that fully-connected networks are included in our analysis as the subcase $n=1$. 

\section{Data Randomness}
\label{S:data_randomness}

Now we may turn our attention to the data distributions of signal and noise: $P_{\rvx,\valpha}(\rvx\pl)$, $P_{\rvx,\dx,\valpha}(\dx\pl)$. To outline the importance of these distributions, the output of an \mbox{$L$-layer} neural network can be expressed by layer composition as $(\rvx^L,\dx^L)=\tilde{\Phi}_{l,L}(\rvx\pl, \dx\pl)$, with $\tilde{\Phi}_{l,L}$ the mapping of the signal and noise by the \emph{upper neural network} from layer $l<L$ to layer $L$. The upper neural network thus receives $\rvx\pl$ as input signal and $\dx\pl$ as input noise, implying that it can only have a chance to do any better than random guessing when: (i) $\smash{\rvx\pl}$ is meaningful; (ii) $\smash{\dx\pl}$ is under control. Namely, when $P_{\rvx,\valpha}(\rvx\pl)$, $P_{\rvx,\dx,\valpha}(\dx\pl)$ are not affected by pathologies. We will make this argument as well as the notion of \emph{pathology} more precise in \Section{S:pathologies} after a few prerequisite definitions.

\subsection{Characterizing Data Distributions} 
\label{S:characterizing_distribution}
Using $\cbullet\pl$ as a placeholder for any tensor of layer $l$ in the simultaneous propagation of $(\rvx\pl,\dx\pl)$ -- e.g. $\rvy\pl$, $\rvx\pl$, $\dy\pl$, $\dx\pl$ in \eq{E:noisy_propagation1} and \eq{E:noisy_propagation2} -- we define:

-- The \emph{feature map vector} and \emph{centered feature map vector},\addtocounter{footnote}{+1} \footnotetext{Slightly abusively, the notation $\rvx,\dx,\valpha,\rvv\pl$ is overloaded in the expectation.}\addtocounter{footnote}{-1}
\begin{align*}
\fm(\cbullet\pl,\valpha) \equiv \cbullet\pl_{\valpha,:}, \qquad \fmc(\cbullet\pl,\valpha) \equiv \cbullet\pl_{\valpha,:} - \E_{\rvx,\dx,\valpha}[ \cbullet\pl_{\valpha,:} ],\footnotemark \\[-12pt] \nonumber
\end{align*}
with $\cbullet\pl_{\valpha,:}$ the vectorial slice of $\cbullet\pl$ at spatial position~$\valpha$. Note that $\fm(\cbullet\pl,\valpha)$, $\fmc(\cbullet\pl,\valpha)$ aggregate both the randomness from $(\rvx, \dx)$ which determines the propagation \mbox{up to $\cbullet\pl$}, and the randomness from $\valpha$ which determines the spatial position in $\cbullet\pl$. These random vectors will enable us to circumvent the tensorial structure of $\cbullet\pl$. 

-- The \emph{non-central moment} and \emph{central moment} of order $p$ for given channel $\rc$ and averaged over channels,
\begin{alignat*}{4}
\nu_{p,\rc}(\cbullet\pl) &\equiv \E_{\rvx,\dx,\valpha} \big[ \fm(\cbullet\pl,\valpha)^p_\rc \big], &\quad \nu_{p}(\cbullet\pl) &\equiv \E_{\rc} \big[ \nu_{p,\rc}(\cbullet\pl) \big], \\
\mu_{p,\rc}(\cbullet\pl) &\equiv \E_{\rvx,\dx,\valpha}  \big[ \fmc(\cbullet\pl,\valpha)^p_\rc \big], &\quad \mu_{p}(\cbullet\pl) &\equiv \E_{\rc} \big[ \mu_{p,\rc}(\cbullet\pl) \big].
\end{alignat*}
In the particular case of the noise $\dx\pl$, centered with respect to $\dx$, feature map vectors and centered feature map vectors coincide: $\fm(\dx\pl,\valpha)=\fmc(\dx\pl,\valpha)$, such that non-central moments and central moments also coincide: \mbox{$\nu_{p,\rc}(\dx\pl)=\mu_{p,\rc}(\dx\pl)$} and $\nu_{p}(\dx\pl)=\mu_{p}(\dx\pl)$.

-- The \emph{effective rank} \citep{Vershynin12},
\begin{align*}
\reff(\cbullet\pl) \equiv \frac{\Tr \mC_{\rvx,\dx,\valpha} \big[\fm(\cbullet\pl,\valpha) \big]}{|| \mC_{\rvx,\dx,\valpha}\big[ \fm(\cbullet\pl,\valpha)\big] ||},
\end{align*}
with $\mC_{\rvx,\dx,\valpha}$ the covariance matrix and $|| \cdot ||$ the spectral norm. If we further denote $(\lambda_i)$ the eigenvalues of $\mC_{\rvx,\dx,\valpha} [\fm(\cbullet\pl,\valpha) ]$, then $\reff(\cbullet\pl)= \sum_i \lambda_i / \max_i \lambda_i\geq1$. Intuitively, $ \reff(\cbullet\pl)$ measures the number of effective directions which concentrate the variance of $\fm(\cbullet\pl,\valpha)$.

-- The \emph{normalized sensitivity} -- our key metric -- derived from the moments of $\rvx\pl$ and $\dx\pl$,
\begin{align} 
\S\pl \equiv \Bigg(\frac{ \mu_2(\dx\pl) }{ \mu_2(\rvx\pl) }\Bigg)^\frac{1}{2} \Bigg(\frac{\mu_2(\dx\po)}{\mu_2(\rvx\po)}\Bigg)^{-\frac{1}{2}}. \label{E:definition_S}
\end{align}
To grasp the definition of $\S\pl$, we may consider the signal-to-noise ratio $\snr\pl$ and the noise factor $F\pl$,
\begin{align}
\snr\pl \equiv \frac{ \mu_2(\rvx\pl)}{ \mu_2(\dx\pl) }, \qquad \label{E:factor_noise}
F\pl \equiv \frac{\snr\po}{\snr\pl} =(\S\pl)^2.
\end{align}
We obtain $\snr\pl_\db=\snr\po_\db- 20\log_{10} \S\pl$ in logarithmic decibel scale, i.e. that $\S\pl$ measures how the neural network from layer $0$ to $l$ degrades ($\S\pl > 1$) or enhances ($\S\pl < 1$) the input signal-to-noise ratio. Neural networks with $\S\pl > 1$ are noise amplifiers, while neural networks with $\S\pl < 1$ are noise reducers.

Now, to justify our choice of terminology, let us reason in the case where $\rvx\pl=\Phi\ul(\rvx\po)$ is the output signal at the final layer. Then: (i) the variance $\mu_2(\rvx\pl)$ is typically constrained by the task (e.g. binary classification constrains $\mu_2(\rvx\pl)$ to be roughly equal to $1$); (ii) the constant rescaling $\Psi\ul(\rvx\po)=\sqrt{\mu_2(\rvx\pl)} / \sqrt{\mu_2(\rvx\po)} \cdot \rvx\po$ leads to the same constrained variance: $\mu_2(\Psi\ul(\rvx\po))=\mu_2(\Phi\ul(\rvx\po))$. The normalized sensitivity $\S\pl$ exactly measures the excess root mean square sensitivity of the neural network mapping $\Phi\ul$ relative to the constant rescaling $\Psi\ul$ [\ref{S:S_sensitivity}]. This property is illustrated in \fig{F:chi_visualization}.

\begin{figure}[h]
    \begin{center}
    \includegraphics[width=1.0\linewidth]{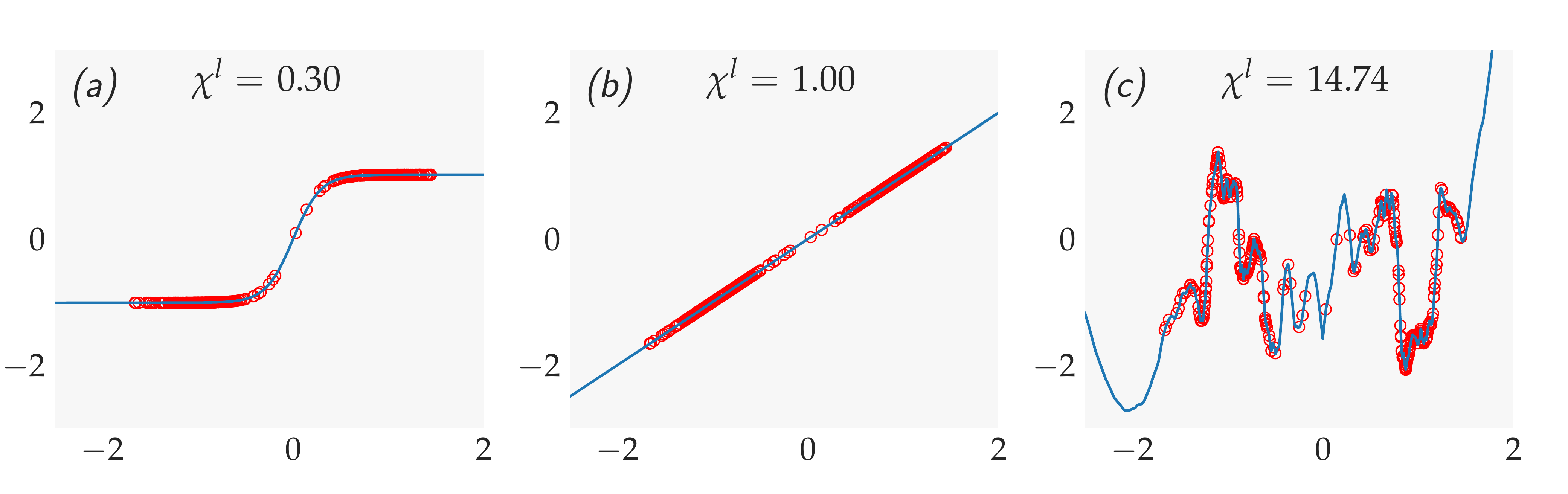}
   \caption{\emph{Illustration of $\S\pl$} in the fully-connected case with one-dimensional input and output, $N\uo=1$, \mbox{$N\ul=1$}. We show the full input-output mapping $\Phi\ul$ \emph{(blue curves)} and randomly sampled input-output data points $\big(\rvx\po,\Phi\ul(\rvx\po)\big)$ \emph{(red circles)} for three different neural networks sharing the same input signal $\rvx\po$ and the same variance in their output signal $\mu_2(\Phi\ul(\rvx\po))$. \emph{(a)} Since input data points $\rvx\po$ appear in flat regions of $\Phi\ul$, the sensitivity is low: $\S\pl<1$. \emph{(b)} $\Phi\ul$ is a constant rescaling: $\S\pl=1$. \emph{(c)} Since $\Phi\ul$ is highly chaotic, the sensitivity is high: $\S\pl >1$.}
    \label{F:chi_visualization}
    \end{center}
\end{figure}

As outlined, $\S\pl$ measures the sensitivity to signal perturbation, which is known for being connected to generalization \citep{Rifai11,Arpit17,Sokolic17,Arora18,Morcos18,Novak18,Philipp18b}. A tightly connected notion is the sensitivity to weight perturbation, also known for being connected to generalization \citep{Hochreiter97,Langford02,Keskar17,Chaudhari17,Smith18,Dziugaite17,Neyshabur17,Neyshabur18,Li18}. The connection is seen by noting the equivalence between a noise $\dd\vomega\pl$ on the weights and a noise $\dy\pl=\dd\vomega\pl \ast \rvx\plm$ and $\dx\pl=\phi\pp(\rvy\pl) \odot \dy\pl$ on the signal in \eq{E:noisy_propagation1} and \eq{E:noisy_propagation2}.

\subsection{Characterizing Pathologies}
\label{S:pathologies}
We are now able to characterize the pathologies, with ill-behaved data distributions, $P_{\rvx,\valpha}(\rvx\pl)$, $P_{\rvx,\dx,\valpha}(\dx\pl)$, that we will encounter: 

-- \emph{Zero-dimensional signal}: $\mu_2(\rvx\pl) / \nu_2(\rvx\pl)\xrightarrow{l\rightarrow\infty}0$. To understand this pathology, let us consider the following mean vectors and rescaling of the signal:
\begin{align*}
 \vnu\pl\equiv\big(\nu_{1,\rc}(\rvx\pl)\big)_{\rc}, 
 \quad \tx\pl\equiv \frac{\rvx\pl}{||\vnu\pl||_2}, 
 \quad \tilde{\vnu}\pl\equiv\big(\nu_{1,\rc}(\tx\pl)\big)_{\rc}.
\end{align*}
The pathology $\mu_2(\rvx\pl) / \nu_2(\rvx\pl)\to0$ implies $\mu_2(\tx\pl)\to0$, meaning that $\fm(\tx\pl,\valpha)$ becomes point-like concentrated at the point $\tilde{\vnu}\pl$ of unit $\normltwo$ norm: $||\tilde{\vnu}\pl||_2=1$ [\ref{S:distributional_pathologies_proof}]. In the limit of strict point-like concentration, the upper neural network from layer $l$ to $L$ is limited to random guessing since it “sees” all inputs the same and cannot distinguish between them.

-- \emph{One-dimensional signal}: $\reff(\rvx\pl)\xrightarrow{l\rightarrow\infty}1$. This pathology implies that the variance of $\fm(\rvx\pl,\valpha)$ becomes concentrated in a single direction, meaning that $\fm(\rvx\pl,\valpha)$ becomes line-like concentrated. In the limit of strict line-like concentration, the upper neural network from layer $l$ to $L$ only “sees” a single feature \mbox{from $\rvx$}.

-- \emph{Exploding sensitivity}: $\S\pl \geq \exp(\gamma l) \xrightarrow{l\rightarrow\infty} \infty$ for some $\gamma >0$. Given the noise factor equivalence of \eq{E:factor_noise}, the pathology $\S\pl \to \infty$ implies $\snr\pl \to 0$, meaning that the clean signal $\rvx\pl$ becomes drowned in the noise $\dx\pl$. In the limit of strictly zero signal-to-noise ratio, the upper neural network from layer $l$ to $L$ is limited to random guessing since it only “sees” noise.

\section{Model Parameters Randomness}
\label{S:model_parameters_randomness}
We now introduce model parameters as the second source of randomness. We consider networks at initialization, which we suppose is \emph{standard} following \citet{He15}: (i) weights are initialized with $\vomega\pl \sim \mathcal{N}\big(0,2\,/\,(K\ul^d N\ulm) \, \mI\big)$, biases are initialized with zeros; (ii) when pre-activations are batch-normalized, scale and shift batch normalization parameters are initialized with ones and zeros respectively.

Considering networks at initialization is justified in two respects. As the first justification, in the context of Bayesian neural networks, the distribution on model parameters at initialization induces a distribution on input-output mappings which can be seen as the prior encoded by the choice of hyperparameters \citep{Neal96,Williams97}. 


As the second justification, even in the standard context of non-Bayesian neural networks, it is likely that pathologies at initialization penalize training by hindering optimization. Let us illustrate this argument in two cases:

-- In the case of zero-dimensional signal, the upper neural network from layer $l$ to $L$ must adjust its bias parameters very precisely in order to center the signal and distinguish between different inputs. This case -- further associated with vanishing gradients for bounded $\phi$ \citep{Schoenholz17} -- is known as the “ordered phase” with unit correlation between different inputs, resulting in untrainability \citep{Schoenholz17,Xiao18}.

-- In the case of exploding sensitivity, the upper neural network from layer $l$ to $L$ only “sees” noise and its backpropagated gradient is purely noise. Gradient descent then performs random steps and training loss is not decreased. This case -- further associated with exploding gradients for batch-normalized $\phi=\relu$ or bounded $\phi$ \citep{Schoenholz17} -- is known as the “chaotic phase” with decorrelation between different inputs, also resulting in untrainability \citep{Schoenholz17,Yang17,Xiao18,Philipp18b,Yang19a}. 

 From now on, our methodology is to consider all moment-related quantities, e.g. $\nu_p(\rvx\pl)$, $\mu_p(\rvx\pl)$, $\mu_p(\dx\pl)$, $\reff(\rvx\pl)$, $\reff(\dx\pl)$, $\S\pl$, as random variables which depend on model parameters. We denote the model parameters as \mbox{$\Tl\equiv(\vomega^{1}, \vbeta^{1}, \dots, \vomega\pl, \vbeta\pl)$} and use $\tl$ as shorthand for $\Tl|\Tlm$. We further denote the geometric increments of $\nu_2(\rvx\pl)$ as $\delta \nu_2(\rvx\pl) \equiv \nu_2(\rvx\pl) / \nu_2(\rvx\plm)$.

\textbf{Evolution with Depth.} The evolution with depth of $\nu_2(\rvx\pl)$ can be written as 
\begin{align*}
\log \Bigg(\frac{\nu_2(\rvx\pl)}{\nu_2(\rvx\po)}\Bigg) =  
\sum_{k\leq l} \underbrace{\log  \delta \nu_2(\rvx\pk) - \E_{\tk}[ \log \delta \nu_2(\rvx\pk)]}_{ \underline{s}[\nu_2(\rvx\pk)] } + \\
\underbrace{\E_{\tk}[ \log  \delta \nu_2(\rvx\pk)] - \log \E_{\tk}[ \delta \nu_2(\rvx\pk)]}_{ \underline{m}[\nu_2(\rvx\pk)] }
+ \underbrace{\log \E_{\tk}[ \delta \nu_2(\rvx\pk)]}_{ \overline{m}[\nu_2(\rvx\pk)] },
\end{align*}
where we used $\log \Big(\frac{\nu_2(\rvx\pl)}{\nu_2(\rvx\po)}\Big) = \log \nu_2(\rvx\pl) - \log \nu_2(\rvx\po)=\sum_{k\leq l} \log \delta \nu_2(\rvx\pk)$ and expressed $\log \delta \nu_2(\rvx\pk)$ with telescoping terms. Denoting $\underline{\delta} \nu_2(\rvx\pk) \equiv \delta \nu_2(\rvx\pk) / \E_{\tk}[ \delta \nu_2(\rvx\pk)]$ the multiplicatively centered increments of $\nu_2(\rvx\pk)$, we get [\ref{S:derivation_mx_sx}]
\begin{align} 
 \overline{m}[\nu_2(\rvx\pk)] & = \log \E_{\tk}[ \delta \nu_2(\rvx\pk)], \label{E:overline_mx}  \\
 \underline{m}[\nu_2(\rvx\pk)] & = \E_{\tk}[ \log \underline{\delta} \nu_2(\rvx\pk) ],\label{E:underline_mx} \\
 \underline{s}[\nu_2(\rvx\pk)] & = \log  \underline{\delta} \nu_2(\rvx\pk) - \E_{\tk}[ \log \underline{\delta} \nu_2(\rvx\pk)].\label{E:underline_sx}
\end{align}
\textbf{Discussion.} We directly note that: (i) $\overline{m}[\nu_2(\rvx\pk)]$ and $\underline{m}[\nu_2(\rvx\pk)]$ are random variables which depend on $\Tkm$, while $\underline{s}[\nu_2(\rvx\pk)]$ is a random variable which depends \mbox{on $\Tk$}; (ii) $\underline{m}[\nu_2(\rvx\pk)] < 0$ by log-concavity; (iii) $\underline{s}[\nu_2(\rvx\pk)]$ is centered with $\E_{\tk}[ \vss \underline{s}[\nu_2(\rvx\pk)] \vss]=0$ and $\E_{\Tk}[ \vss \underline{s}[\nu_2(\rvx\pk)] \vss]=0$. 

We further note that each channel provides an independent contribution to $\nu_2(\rvx\pk)= \frac{1}{N_k}\sum_{\rc}\nu_{2,\rc}(\rvx\pk)$, implying for large $N_k$ that $\underline{\delta} \nu_2(\rvx\pk)$ has low expected deviation to $1$ and that $|\log \underline{\delta} \nu_2(\rvx\pk)| \ll 1$, $|\underline{m}[\nu_2(\rvx\pk)]| \ll 1$, $|\underline{s}[\nu_2(\rvx\pk)]| \ll 1$ with high probability. The term $\overline{m}[\nu_2(\rvx\pk)]$ is thus dominating as long as it is not vanishing. The same reasoning applies to other positive moments, e.g. $\mu_2(\rvx\pl)$, $\mu_2(\dx\pl)$.

\textbf{Further Notation.} From now on, the geometric increment of any quantity is denoted with $\delta$. The definitions of $\overline{m}$, $\underline{m}$ and $\underline{s}$ in \eq{E:overline_mx},~(\ref{E:underline_mx})~and~(\ref{E:underline_sx}) are extended to other positive moments of signal and noise, as well as $\S\pl$ with  
\begin{align*}
\overline{m}[\S\pl] &\equiv\tfrac{1}{2}\big(\overline{m}[\mu_2(\dx\pl)]   - \overline{m}[\mu_2(\rvx\pl)]  \big), \\ 
\underline{m}[\S\pl] &\equiv\tfrac{1}{2}\big(\underline{m}[\mu_2(\dx\pl)]  - \underline{m}[\mu_2(\rvx\pl)] \big), \\
\underline{s}[\S\pl] &\equiv\tfrac{1}{2}\big(\underline{s}[\mu_2(\dx\pl)] - \underline{s}[\mu_2(\rvx\pl)]  \big).
\end{align*}
We introduce the notation $a\simeq b$ when $a(1+\epsilon_a) = b (1+\epsilon_b)$ with $|\epsilon_a|\ll 1$, $|\epsilon_b| \ll 1$ with high probability. And the notation $a\lesssim b$ when $a(1+\epsilon_a) \leq b (1+\epsilon_b)$ with $|\epsilon_a|\ll 1$, $|\epsilon_b| \ll 1$ with high probability. From now on, we assume that the \emph{width is large}, implying
\begin{align*}
\delta \S\pl = \exp\big( \overline{m}[\S\pl] + \underline{m}[\S\pl] + \underline{s}[\S\pl]\big) \simeq \exp\big( \overline{m}[\S\pl]\big).
\end{align*}


We stress the \emph{layer-wise} character of this approximation, whose validity only requires $N\ul\gg1$, independently of the depth $l$. This contrasts with the \emph{aggregated} character \mbox{(up to layer $l$)} of the mean field approximation of $\rvy\pl$ as a Gaussian process, whose validity requires not only $N\ul\gg1$ but also -- as we will see -- that the depth $l$ remains sufficiently small with respect to $N\ul$.

\section{Vanilla Nets}
\label{S:vanilla_nets}
We are fully equipped to characterize deep neural networks at initialization. We start by analyzing vanilla nets which correspond to the propagation introduced in \Section{S:equation_propagation}.

\begin{Theorem}[moments of vanilla nets]
\emph{[\ref{S:proof_moments_vanilla}]}
\label{T:moments_vanilla}
There exist small constants $1\gg m\umin, m\umax, v\umin, v\umax >0$, random variables $m\ul, m\pp\ul$, $s\ul,s\pp\ul$ and events $A\ul$, $A\pp\ul$ of probabilities equal to $\prod^l_{k=1} ( 1 - 2^{-N_k})$ such that 
\begin{alignat*}{3}
& \textnormal{Under $A\ul$: } &  \quad  \log \nu_2(\rvx\pl)  & = - l m\ul + \sqrt{l} s\ul + \log \nu_2(\rvx\po), \\
& \textnormal{Under $A\pp\ul$: } & \quad \log \mu_2(\dx\pl) & = - l m\pp\ul + \sqrt{l} s\pp\ul + \log \mu_2(\dx\po). \\[-38pt]
\end{alignat*}
\begin{align*}
\hspace{27pt}\underbrace{\qquad\qquad\quad\qquad\qquad\qquad\qquad\qquad\qquad\;\;\;}_\text{\parbox{8.2cm}{
\hspace{-24pt} $m\umin \leq m\ul \leq m\umax$, \;\;
$\E_{\Tl|A\ul} \big[s\ul\big]=0$,  \;\;
$v\umin \leq \Var_{\Tl|A\ul}\big[s\ul\big]\leq v\umax$ \\ 
\hspace*{-24pt} $m\umin \leq m\pp\ul \leq m\umax$, \;\;
$\E_{\Tl|A\pp\ul}\big[s\pp\ul\big]=0$,  \;\;
$v\umin \leq \Var_{\Tl|A\pp\ul}\big[s\pp\ul\big]\leq v\umax$ }} \\[-21pt]
\end{align*}

\end{Theorem}
\textbf{Discussion.} The conditionality on $A\ul$, $A\pp\ul$ is necessary to exclude the collapse: $\nu_2(\rvx\pl)=0$, $\mu_2(\dx\pl)=0$, with undefined $\log \nu_2(\rvx\pl)$, $\log \mu_2(\dx\pl)$, occurring e.g. when all elements of $\vomega\pl$ are strictly negative \citep{Lu18}. In practice, this conditionality is highly negligible since the probabilities of the complementary events $A\pc\ul$, $A\ppc\ul$ decay exponentially in the width $N\ul$ [\ref{S:negligible_conditionality}].

Now let us look at the evolution of $\log \nu_2(\rvx\pl)$, $\log \mu_2(\dx\pl)$ under $A\ul$, $A\pp\ul$. The initialization \citet{He15} enforces $\E_\tl[\nu_2(\rvx\pl)]=\nu_2(\rvx\plm)$ and $\E_\tl[\mu_2(\dx\pl)]=\mu_2(\dx\plm)$ such that: (i) $\E_\Tl[\nu_2(\rvx\pl)]$, $\E_\Tl[\mu_2(\dx\pl)]$ are kept stable during propagation; (ii) $\overline{m} [ \nu_2(\rvx\pl) ]$, $\overline{m} [ \mu_2(\dx\pl) ]$ vanish and $\log \nu_2(\rvx\pl)$, $\log \mu_2(\dx\pl)$ are subject to a slow diffusion with small negative drift terms: \mbox{$\underline{m} [\nu_2(\rvx\pl) ]<0$, $\underline{m} [\mu_2(\dx\pl) ]<0$}, and small diffusion terms: $\underline{s} [\nu_2(\rvx\pl) ]$, $\underline{s} [\mu_2(\dx\pl) ]$ [\ref{S:moments_vanilla_terms_evolution}].\footnote{Any deviation from \citet{He15} leads, on the other hand, to pathologies orthogonal to the pathologies of Section \ref{S:pathologies}, with either exploding or vanishing constant scalings of $(\rvx\pl,\dx\pl)$.} The diffusion happens in log-space since layer composition amounts to a multiplicative random effect in real space. It is a finite-width effect since the terms $\underline{m} [ \nu_2(\rvx\pl) ]$, $\underline{m} [ \mu_2(\dx\pl) ]$, $\underline{s} [\nu_2(\rvx\pl) ]$, $\underline{s} [\mu_2(\dx\pl) ]$ also vanish for infinite width.

\fig{F:vanilla_histo} illustrates the slowly decreasing negative expectation and slowly increasing variance of $\log \nu_2(\rvx\pl)$, $\log \mu_2(\dx\pl)$, caused by the small negative drift and diffusion terms. \fig{F:vanilla_histo} also indicates that $\log \nu_2(\rvx\pl)$, $\log \mu_2(\dx\pl)$ are nearly Gaussian, implying that $\nu_2(\rvx\pl)$, $\mu_2(\dx\pl)$ are nearly lognormal. Two important insights are then provided by the expressions of the expectation: $\exp(\mu+ \sigma^2 / 2)$ and the kurtosis:~$\exp(4\sigma^2)+2\exp(3\sigma^2)+3\exp(2\sigma^2)-3$ of a lognormal variable $\exp(X)$ with \mbox{$X\sim\mathcal{N}(\mu,\sigma^2)$}. Firstly, the decreasing negative expectation and increasing variance of $\log \nu_2(\rvx\pl)$, $\log \mu_2(\dx\pl)$ act as opposing forces in order to ensure the stabilization of $\E_\Tl[\nu_2(\rvx\pl)]$, $\E_\Tl[\mu_2(\dx\pl)]$. Secondly, $\nu_2(\rvx\pl)$, $\mu_2(\dx\pl)$ are stabilized only in terms of expectation and they become fat-tailed distributed as $l\to\infty$.

\begin{figure}[h]
    \begin{center}
       \captionsetup[subfigure]{labelformat=empty}
  \begin{subfigure}{0pt}
  \caption{\label{F:vanilla_histo_a}}
  \end{subfigure}
  \begin{subfigure}{0pt}
  \caption{\label{F:vanilla_histo_b}}
  \end{subfigure}
    \includegraphics[width=1.0\linewidth]{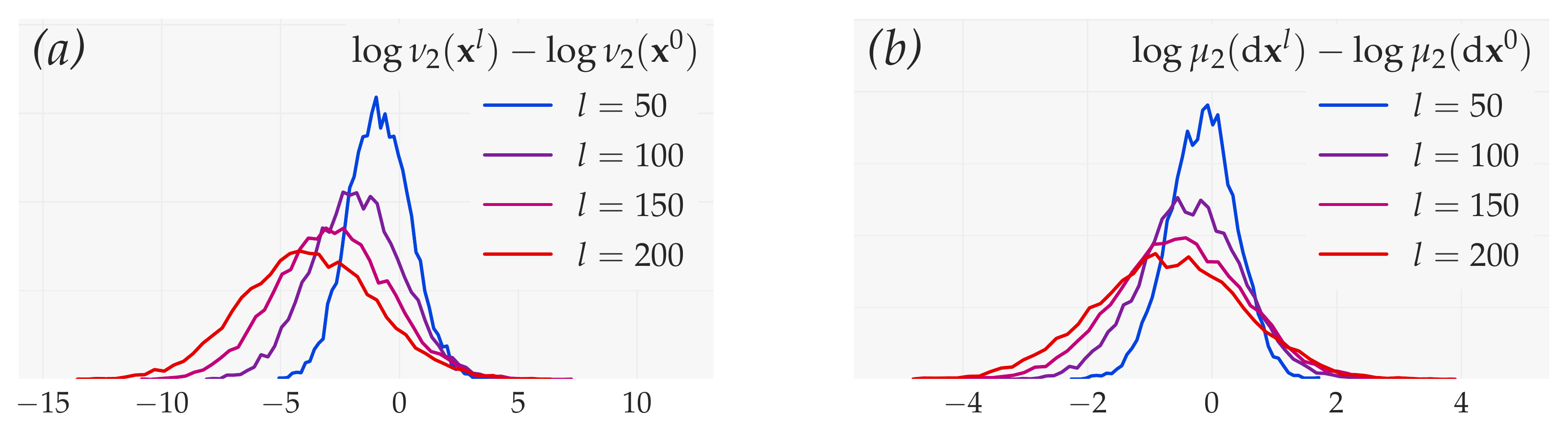}
    \caption{\emph{Slowly diffusing moments of vanilla nets} with $L=200$ layers of width $N\ul=128$. \emph{(a)} Distribution of $\log \nu_2(\rvx\pl)-\log \nu_2(\rvx\po)$ for $l=50,100,150,200$. \emph{(b)} Same for $\log \mu_2(\dx\pl)-\log \mu_2(\dx\po)$.}
    \label{F:vanilla_histo}
    \end{center}
\end{figure}

\begin{Theorem}[normalized sensitivity increments of vanilla nets]
\label{T:S_increments_vanilla}
\emph{[\ref{S:proof_S_increments_vanilla}]} Denoting $\rvy^{l,\pm}\equiv\max\big(\pm\rvy\pl,0\big)$, the dominating term under $\{ \mu_2(\rvx\plm) > 0\}$ \mbox{in the evolution of $\S\pl$ is}
\begin{align}
\delta \S\pl & \simeq \Bigg( 1 - \E_{\rc,\tl} \Bigg[ \frac{\nu_{1,\rc}  \big( \rvy\plplus \big) \nu_{1,\rc} \big( \rvy\plminus \big)}{\mu_{2}(\rvx\plm)} \Bigg] \Bigg)^{-\frac{1}{2}}.\label{E:S_increments_vanilla} \\[-36pt] \nonumber
\end{align}
\begin{align*}
\hspace{14pt}\underbrace{\qquad\qquad\qquad\qquad\qquad\qquad\qquad\quad\;\;}_{
\hspace{2pt} \in [1,\sqrt{2}]} \\[-21pt]
\end{align*}
\end{Theorem}

\textbf{Discussion.}
A first consequence is that $\S\pl$ always increases with depth. Another consequence is that only two possibilities of evolution which both lead to pathologies are allowed:

-- If sensitivity is exploding: $\S\pl \geq \exp(\gamma l) \to \infty$ with exponential drift $\gamma$ stronger than the slow diffusion of \mbox{\theorem{T:moments_vanilla}} and if $\nu_2(\rvx\pl)$, $\mu_2(\dx\pl)$ are lognormally distributed as supported by \fig{F:vanilla_histo}, then \mbox{\theorem{T:moments_vanilla}} implies the a.s. convergence to the pathology of zero-dimensional signal: $\mu_2(\rvx\pl)/\nu_2(\rvx\pl) \to 0$ [\ref{S:ratio_cv_as}].

-- Otherwise, geometric increments $\delta \S\pl$ are strongly limited. In the limit \mbox{$\delta \S\pl \simeq \exp\big( \overline{m}[\S\pl]\big)\to1$}, if the moments of $\tx\pl\equiv\rvx\pl/\sqrt{\mu_2(\rvx\pl)}$ remain bounded, then \theorem{T:S_increments_vanilla} implies the convergence to the pathology of one-dimensional signal: $\reff(\rvx\pl) \to1$~[\ref{S:one_dimensional_pathology}] and the convergence to pseudo-linearity, with each additional layer $l$ becoming arbitrarily well approximated by a linear mapping~[\ref{S:pseudo_linearity}].

\begin{figure}[h]
    \begin{center}
       \captionsetup[subfigure]{labelformat=empty}
	\begin{subfigure}{0pt}
	\caption{\label{F:vanilla_a}}
	\end{subfigure}
	\begin{subfigure}{0pt}
	\caption{\label{F:vanilla_b}}
	\end{subfigure}
    \includegraphics[width=1.0\linewidth]{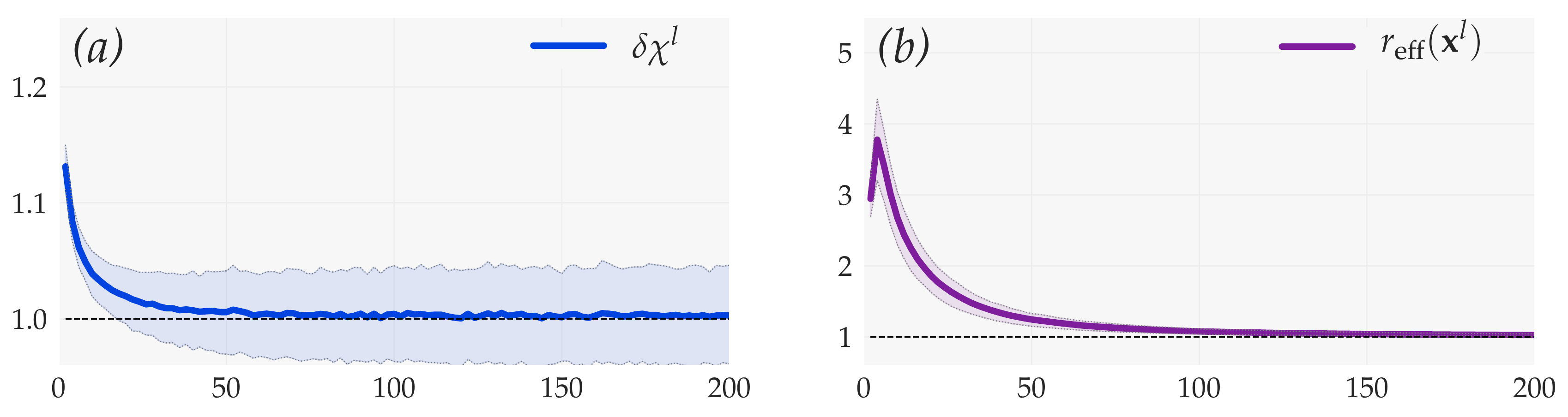}
    \caption{\emph{Pathology of one-dimensional signal for vanilla nets} with $L=200$ layers of width $N\ul=512$. \emph{(a)} $\delta \S\pl$ such that $\delta \S\pl \simeq \exp\big( \overline{m}[\S\pl]\big)\to1$. \emph{(b)} $\reff(\rvx\pl)$ indicates one-dimensional signal pathology: $\reff(\rvx\pl)\to 1$.}
    \label{F:vanilla}
    \end{center}
\end{figure}

\textbf{Experimental Verification.} The evolution with depth of vanilla nets is shown in \fig{F:vanilla}. From the two possibilities, we observe the case with limited geometric increments: \mbox{$\delta \S\pl \simeq \exp\big( \overline{m}[\S\pl]\big) \to 1$}, the convergence to the pathology of one-dimensional signal: $\reff(\rvx\pl) \to1$, and the convergence to pseudo-linearity. 

The only way that the neural network can achieve pseudo-linearity is by having each one of its $\relu$ units either always active or always inactive, i.e. behaving either as zero or as the identity. Our analysis offers theoretical insight into this coactivation phenomenon, previously observed experimentally \citep{Balduzzi17,Philipp18a}.

\section{Batch-Normalized Feedforward Nets}
\label{S:BN_feedforward_nets}
Next we incorporate batch normalization \citep{Ioffe15}, which we denote as $\BN$. For simplicity, we only consider the test mode which consists in subtracting $\nu_{1,\rc}(\rvy\pl)$ and dividing by $\sqrt{\mu_{2,\rc}(\rvy\pl)}$ for each channel $\rc$ in $\rvy\pl$. The propagation is given by
\begin{alignat}{4}
& \rvy\pl && = \vomega\pl \ast \rvx\plm + \vbeta\pl, \qquad &&  \dy\pl && = \vomega\pl \ast \dx\plm, \label{E:bn_propagation1} \\
& \rvz\pl && = \BN(\rvy\pl), \qquad && \dz\pl && = \BN\pp (\rvy\pl) \odot \dy\pl, \label{E:bn_propagation2} \\
& \rvx\pl && = \phi(\rvz\pl), \qquad &&  \dx\pl && = \phi\pp(\rvz\pl) \odot \dz\pl.  \label{E:bn_propagation3}
\end{alignat}

\begin{Theorem}[normalized sensitivity increments of batch-normalized feedforward nets]
\label{T:S_increments_bn_ff}
\emph{[\ref{S:proof_S_increments_bn_ff}]} 
The dominating term in the evolution of $\S\pl$ can be decomposed as
\begin{align*}
\delta \S\pl = \delta\ubn \S\pl \;\cdot\; &  \delta\uphi \S\pl \simeq \exp\big( \overline{m}\uBN[ \S\pl ] \big) \cdot \exp\big( \overline{m}\uPHI[ \S\pl ] \big), \\
\exp\big( \overline{m}\uBN[ \S\pl ] \big) & \equiv \Bigg(\frac{\mu_2(\dx\plm)}{\mu_2(\rvx\plm)}\Bigg)^{-\frac{1}{2}} \E_{\rc,\tl} \Bigg[ \frac{\mu_{2,\rc}(\dy\pl)}{\mu_{2,\rc}(\rvy\pl)} \Bigg]^\frac{1}{2}, \\
\exp\big( \overline{m}\uPHI[ \S\pl ] \big) & \equiv \Big( 1- 2 \vss \E_{\rc,\tl} [ \nu_{1,\rc}  ( \rvz\plplus  ) \nu_{1,\rc} ( \rvz\plminus ) ] \Big)^{-\frac{1}{2}}.\\[-38pt]
\end{align*} 
\begin{align*}
\hspace{72pt}\underbrace{\qquad\qquad\qquad\qquad\qquad\qquad\qquad\;\;}_{
\hspace{2pt} \in [1,\sqrt{2}]} \\[-15pt]
\end{align*}
\end{Theorem}

\textbf{Effect of Batch Normalization.} The batch normalization term is such that $\exp(\overline{m}\uBN[\S\pl]) \simeq \delta\ubn \S\pl$, with $\delta\ubn \S\pl$ defined as the increment of $\S\pl$ in the convolution and batch normalization steps of \eq{E:bn_propagation1} and \eq{E:bn_propagation2}. The expression of $\exp(\overline{m}\uBN[\S\pl])$ holds for any choice of $\phi$.

This term can be understood intuitively by seeing the different channels $\rc$ in $\rvy\pl$ as $N\ul$ random projections of $\rvx\plm$ and batch normalization as a modulation of the magnitude for each projection. Since batch normalization uses $\sqrt{\mu_{2,\rc}(\rvy\pl)}$ as normalization factor, directions of high signal variance are dampened, while directions of low signal variance are amplified. This preferential exploration of low signal directions naturally deteriorates the signal-to-noise ratio and amplifies $\S\pl$ owing to the noise factor equivalence of \eq{E:factor_noise}.

Now let us look directly at $\exp(\overline{m}\uBN[\S\pl])$ in \theorem{T:S_increments_bn_ff}. If we define the event under which the vectorized weights in channel $\rc$ have $\normltwo$ norm equal to $r$: $W^{l,c}_{r}\equiv\big\{ ||\vec(\vomega\pl_{:,:,\rc})||_2=r \big\}$, then spherical symmetry implies that variance increments in channel $\rc$ from $\rvx\plm$ to $\rvy\pl$ and from $\dx\plm$ to $\dy\pl$ have equal expectation under $W^{l,c}_{r}$:
\begin{align*}
\frac{\E_{\tl | W^{l,c}_{r}} [\mu_{2,\rc}(\rvy\pl)]}{ \mu_{2}(\rvx\plm)}=\frac{\E_{\tl | W^{l,c}_{r}} [\mu_{2,\rc}(\dy\pl)]}{\mu_{2}(\dx\plm)}. 
\end{align*} 
On the other hand, the variance of these increments depends on the fluctuation of signal and noise in the random direction generated by $\vec(\vomega\pl_{:,:,\rc})/||\vec(\vomega\pl_{:,:,\rc})||_2$. This depends on the conditioning of signal and noise, i.e. on the magnitude of $\reff(\rvx\plm)$, $\reff(\dx\plm)$. If we assume that $\dx\plm$ is well-conditioned, then $\mu_{2,\rc}(\dy\pl)/ \mu_{2}(\dx\plm)$ can be treated as a constant and by convexity of the function $x\mapsto1/x$:
\begin{align*}
\Bigg(\frac{\mu_2(\dx\plm)}{\mu_2(\rvx\plm)} \Bigg)^{-1} \E_{\tl| W^{l,c}_{r}} \Bigg[ \frac{\mu_{2,\rc}(\dy\pl)}{\mu_{2,\rc}(\rvy\pl)} \Bigg] \gtrsim 1, 
\end{align*} 
which in turn implies $\exp( \overline{m}\uBN[\S\pl]) \gtrsim 1$. The worse the conditioning of $\rvx\plm$, i.e. the smaller $\reff(\rvx\plm)$, the larger the variance of $\mu_{2,\rc}(\rvy\pl)$ at the denominator and the impact of the convexity. Thus the smaller $\reff(\rvx\plm)$ and the larger $\exp( \overline{m}\uBN[\S\pl])$. This argument is strictly valid for the first step of the propagation wherein the noise has perfect conditioning, resulting in $\exp( \overline{m}\uBN[\S^1]) \geq 1$ [\ref{S:bn_first_step}]. 

\textbf{Effect of the Nonlinearity.} The nonlinearity term is such that $\exp(\overline{m}\uPHI[\S\pl])\simeq\delta\uphi \S\pl$, with $\delta\uphi \S\pl$ defined as the increment of $\S\pl$ in the nonlinearity step of \eq{E:bn_propagation3}. This term is analogous to the term of \eq{E:S_increments_vanilla} for vanilla nets, except that $\nu_{1,\rc}  ( \rvz\plplus  ) \nu_{1,\rc} ( \rvz\plminus )$ is less likely to vanish than $\nu_{1,\rc}(\rvy\plplus) \nu_{1,\rc}(\rvy\plminus)/\mu_{2}(\rvx\plm)$ in \eq{E:S_increments_vanilla} since batch normalization now keeps the signal centered around zero.

\textbf{Experimental Verification.} In \fig{F:BN-FF}, we confirm experimentally the pathology of exploding sensitivity: $\S\pl\geq\exp(\gamma l) \to\infty$ for some $\gamma >0$. We also confirm that: (i) $\dx\pl$ remains well-conditioned, while $\rvx\pl$ becomes ill-conditioned; (ii) $\reff(\rvx\pl)$ and $\delta\ubn \S\pl$ are inversely correlated. 

Interestingly, $\delta\uphi \S\pl$ becomes subdominant with respect to $\delta\ubn \S\pl$ at large depth. This stems from the fact that $\rvz\pl$ becomes fat-tailed distributed with respect to $\rvx$, $\valpha$, with large $\mu_4(\rvz\pl)$ and small $\nu_1(|\rvz\pl|)$. Combined with \mbox{$\nu_1(\rvz\plplus)\leq\nu_1(|\rvz\pl|)$} and \mbox{$\nu_1(\rvz\plminus)\leq\nu_1(|\rvz\pl|)$}, this explains the decay of $|\exp( \overline{m}\uPHI[ \S\pl ])-1|$ and thus of $|\delta\uphi \S\pl-1|$.

\begin{figure*}
\begin{minipage}[t]{0.475\linewidth}
    \begin{center}
      \captionsetup[subfigure]{labelformat=empty}
  \begin{subfigure}{0pt}
  \caption{\label{F:BN-FF_a}}
  \end{subfigure}
  \begin{subfigure}{0pt}
  \caption{\label{F:BN-FF_b}}
  \end{subfigure}
  \begin{subfigure}{0pt}
  \caption{\label{F:BN-FF_c}}
  \end{subfigure}
  \begin{subfigure}{0pt}
  \caption{\label{F:BN-FF_d}}
  \end{subfigure}
    \includegraphics[width=1.\linewidth]{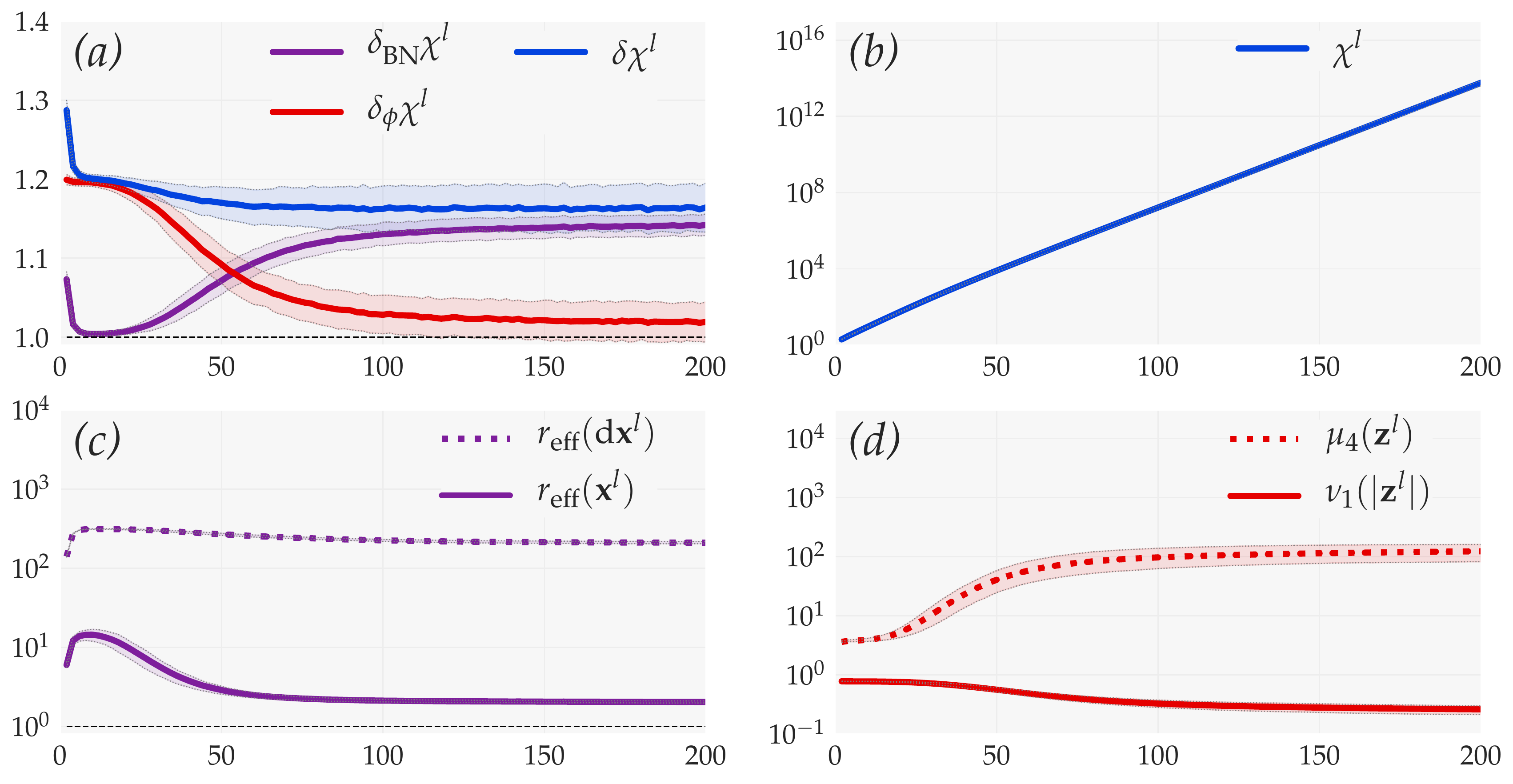}
    \caption{\emph{Pathology of exploding sensitivity for batch-normalized feedforward nets} with $L=200$ layers of width $N\ul=512$. \emph{(a)} Geometric increments $\delta \S\pl$ decomposed as the product of $\delta\ubn \S\pl$ defined as the increment from $(\rvx\plm,\dx\plm)$ to $(\rvz\pl,\dz\pl)$, and $\delta\uphi \S\pl$ defined as the increment from $(\rvz\pl,\dz\pl)$ to $(\rvx\pl,\dx\pl)$. \emph{(b)} The growth of $\S\pl$ indicates exploding sensitivity pathology: $\S\pl\geq \exp(\gamma l) \to\infty$ for some $\gamma >0$. \emph{(c)} $\rvx\pl$ becomes ill-conditioned with small $\reff(\rvx\pl)$. \emph{(d)} $\rvz\pl$ becomes fat-tailed distributed with respect to $\rvx$, $\valpha$, with large $\mu_4(\rvz\pl)$ and small $\nu_1(|\rvz\pl|)$.}
    \label{F:BN-FF}
    \end{center}
\end{minipage}\qquad
\begin{minipage}[t]{0.475\linewidth}
    \begin{center}
      \captionsetup[subfigure]{labelformat=empty}
	\begin{subfigure}{0pt}
	\caption{\label{F:BN-Res_a}}
	\end{subfigure}
	\begin{subfigure}{0pt}
	\caption{\label{F:BN-Res_b}}
	\end{subfigure}
	\begin{subfigure}{0pt}
	\caption{\label{F:BN-Res_c}}
	\end{subfigure}
	\begin{subfigure}{0pt}
	\caption{\label{F:BN-Res_d}}
	\end{subfigure}
    \includegraphics[width=1.\linewidth]{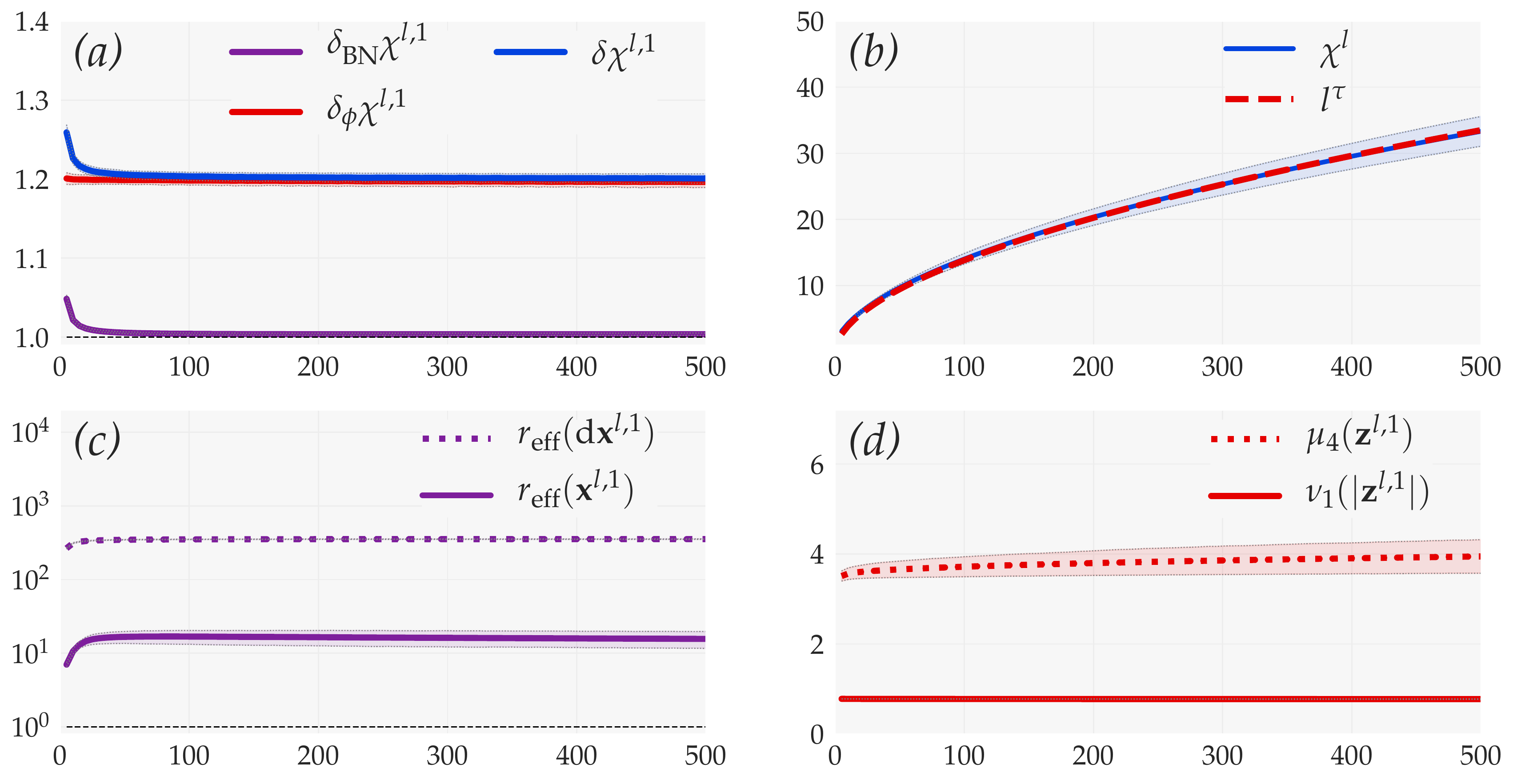}
      \caption{\emph{Well-behaved evolution of batch-normalized resnets} with $L=500$ residual units comprised of $H=2$ layers of width $N=512$. \emph{(a)} Geometric feedforward increments $\delta \S^{l,1}$ decomposed as the product of $\delta\ubn \S^{l,1}$ defined as the increment from $(\rvy^{l,0},\dy^{l,0})$ to $(\rvz^{l,1},\dz^{l,1})$, and $\delta\uphi \S^{l,1}$ defined as the increment from $(\rvz^{l,1},\dz^{l,1})$ to $(\rvy^{l,1},\dy^{l,1})$. \emph{(b)} $\S\pl$ has power-law growth. \emph{(c)} $\reff(\rvx^{l,1})$ indicates that many directions of signal variance are preserved. \emph{(d)} $\mu_4(\rvz^{l,1})$, $\nu_1(|\rvz^{l,1}|)$ indicate that $\rvz^{l,1}$ has close to Gaussian data distribution.}
      \label{F:BN-Res}
    \end{center}
\end{minipage}
\end{figure*}

\section{Batch-Normalized ResNets}
\label{S:BN_resnets}
We finish our exploration of deep neural network architectures with the incorporation of skip-connections. From now on, we assume that the width is constant, $N\ul=N$, and following \citet{He16}, we adopt the perspective of pre-activation units. The propagation is given by
\begin{align}
 (\rvy^{l},\dy^{l}) & = (\rvy^{l-1}, \dy^{l-1}) +  (\rvy^{l,H},\dy^{l,H}), \label{E:bn_res_propagation1}\\[-22pt] \nonumber
\end{align}
\begin{alignat*}{4}
&  \rvz\plh && \hspace{1.2pt}{=}\hspace{2.5pt} \BN(\rvy\plhm), \quad &&  \dz\plh && \hspace{1.2pt}{=}\hspace{2.5pt} \BN\pp(\rvy\plhm) \odot \dy\plhm,  \\
&  \rvx\plh && \hspace{1.2pt}{=}\hspace{2.5pt} \phi(\rvz\plh), \quad && \dx\plh && \hspace{1.2pt}{=}\hspace{2.5pt} \phi\pp (\rvz\plh) \odot \dz\plh, \\
& \rvy\plh && \hspace{1.2pt}{=}\hspace{2.5pt} \vomega\plh \ast \rvx\plh \hspace{1.2pt}{+}\hspace{2.5pt} \vbeta\plh,\;\; && \dy\plh && \hspace{1.2pt}{=}\hspace{2.5pt} \vomega\plh \ast \dx\plh.\\[-32pt]
\end{alignat*}
\begin{align*}
\underbrace{\qquad\qquad\qquad\qquad\qquad\qquad\qquad\qquad\qquad\qquad\qquad\quad\;\;}_\text{\parbox{6.5cm}{\centering $1 \leq h \leq H$, with $H$ the number of layers inside residual units and with $(\rvy^{l,0},\dy^{l,0})\equiv(\rvy\plm,\dy\plm)$}}
\end{align*}
If we adopt the convention $(\rvy^{0,H},\dy^{0,H}) \equiv (\rvy\po,\dy\po)$, then \eq{E:bn_res_propagation1} can be expanded as 
\begin{align}
  (\rvy^{l},\dy^{l}) & = \sum\nolimits^l_{k=0} (\rvy^{k,H},\dy^{k,H}). \label{E:bn_res_propagation4} 
\end{align}
For consistency reasons, we redefine the inputs of the propagation as $(\rvy,\dy) \equiv (\rvy\po,\dy\po)$ and the normalized sensitivity and its increments as
\begin{alignat*}{5}
  \S\plh & \hspace{2.2pt}{\equiv}\hspace{2.2pt}\Bigg( \frac{\mu_2(\dy\plh)}{\mu_2(\rvy\plh)} \Bigg)^\frac{1}{2} & \Bigg( \frac{\mu_2(\dy\po)}{\mu_2(\rvy\po)} \Bigg)^{-\frac{1}{2}}&, & \quad \delta \S\plh  & \hspace{2.2pt}{\equiv}\hspace{3.0pt} \frac{\S\plh}{\S\plhm}, \\
    \S\pl  & \hspace{2.5pt}{\equiv}\hspace{2.2pt}\Bigg( \frac{\mu_2(\dy\pl)}{\mu_2(\rvy\pl)} \Bigg)^\frac{1}{2} & \Bigg( \frac{\mu_2(\dy\po)}{\mu_2(\rvy\po)} \Bigg)^{-\frac{1}{2}}&, & \quad \delta \S\pl  & \hspace{2.5pt}{\equiv}\hspace{3.0pt} \frac{\S\pl}{\S\plm}.
\end{alignat*}

\begin{Theorem}[normalized sensitivity increments of batch-normalized resnets]
\label{T:S_increments_bn_res}
\emph{[\ref{S:proof_S_increments_bn_res}]} 
Suppose that we can bound signal variances: $\mu\udmin \lesssim \mu_2(\rvy\plH) \lesssim \mu\udmax$ and feedforward increments: $\delta\umin \lesssim \delta \S^{l,h} \lesssim \delta\umax$ for all $l,h$. Further denote $\eta\umin\equiv\big( (\delta\umin)^{2H} \mu\udmin - \mu\udmax \big) / \mu\udmax$ and $\eta\umax\equiv\big( (\delta\umax)^{2H} \mu\udmax - \mu\udmin \big)/ \mu\udmin$, as well as $\tau\umin\equiv\tfrac{1}{2}\eta\umin$ and $\tau\umax\equiv\tfrac{1}{2}\eta\umax$. Then there exist positive constants $C\umin$, $C\umax>0$ such that
\begin{alignat}{4}
 \Big( 1 + \frac{\eta\umin}{l+1} \Big)^\frac{1}{2}
 & \lesssim && \,\, \delta \S\pl && \lesssim \Big( 1 + \frac{\eta\umax}{l+1}  \Big)^\frac{1}{2},  \label{E:S_increments_bn_res_Theorem1}  \\
C\umin l^{\tau\umin} & \lesssim && \;\; \S\pl \, && \lesssim C\umax l^{\tau\umax}. \label{E:S_increments_bn_res_Theorem2}
\end{alignat}
\end{Theorem}

\textbf{Discussion.}
First let us note that \theorem{T:S_increments_bn_res} remarkably holds for any choice of $\phi$, with and without batch normalization, as long as the existence of $\mu\udmin$, $\mu\udmax$, $\delta\umin$, $\delta\umax$ is ensured. In the case $\phi=\relu$, the existence of $\delta\umin$, $\delta\umax$ is always ensured but the existence of $\mu\udmin$, $\mu\udmax$ is only ensured when batch normalization controls signal variance inside residual units: $\mu_{2,\rc}(\rvz\plH)=1$ [\ref{S:increments_bn_res_other_phi}].

Now let us get a better grasp of \theorem{T:S_increments_bn_res}. We see in \eq{E:S_increments_bn_res_Theorem1} that the evolution remains exponential inside residual units since $\eta\umin$, $\eta\umax$ have an exponential dependence in $H$. However, it is slowed down by the factor $1/(l+1)$ between successive residual units. This stems from the dilution \citep{Philipp18a} of the residual path $(\rvy^{l,H},\dy^{l,H})$ into the skip-connection path $(\rvy^{l-1}, \dy^{l-1})$ with ratio of signal variances $\mu_2(\rvy^{l,H}) / \big( \mu_2(\rvy^{l,H})+\mu_2(\rvy^{l-1}) \big)$, decaying as $1/(l+1)$. If we remove the dilution effect by multiplying the skip-connection branch by $0$ (i.e. replacing the scaling in $1/(l+1)$ by a scaling in $1$) and if we set $\mu\udmin=\mu\udmax$, then \eq{E:S_increments_bn_res_Theorem1} recovers the feedforward evolution $(\delta\umin)^H \lesssim \delta \S\pl \lesssim (\delta\umax)^H$. The dilution is clearly visible in \eq{E:bn_res_propagation4}. Namely, each residual unit adds a term $(\rvy^{l,H},\dy^{l,H})$ of increased $\S^{l,H}$ but its relative contribution to the aggregation gets smaller and smaller with~$l$, so that the growth of $\S\pl$ gets slower and slower with~$l$. 

Since $\frac{1}{2}\log(1+\frac{\eta}{x})\simeq\frac{\eta}{2x}$ and $\int_{x_0}^x \frac{\eta}{2x\pp} dx\pp\simeq\log x^\frac{\eta}{2}$ for $x\gg1$, the bounds on $\S\pl=\prod_k \delta \S\pk=\exp\big(\sum_k \log \delta \S\pl \big)$ in \eq{E:S_increments_bn_res_Theorem2} are obtained by integrating the bounds on the logarithm of \eq{E:S_increments_bn_res_Theorem1}. A direct consequence of the dilution is thus the power-law evolution of $\S\pl$ instead of the exponential evolution for feedforward nets. Equivalently, when rewriting \eq{E:S_increments_bn_res_Theorem2} as
\begin{align*}
C\umin \exp(\tau\umin\log l )\lesssim\S\pl\lesssim C\umax \exp(\tau\umax \log l), 
\end{align*}
the evolution of $\S\pl$ for resnets is equivalent to the evolution of $\S^{\tau \log l}$ for some $\tau>0$ for feedforward nets. In other words, the evolution with depth of resnets is the \emph{logarithmic} version of the evolution with depth of feedforward nets.

\textbf{Experimental Verification.} The evolution with depth of batch-normalized resnets is shown in \fig{F:BN-Res}. There is a clear parallel between the evolution for $l\leq 500$ in \fig{F:BN-Res} and the evolution for $l\lesssim 15$ in \fig{F:BN-FF}. This confirms that batch-normalized resnets are slower-to-evolve variants of batch-normalized feedforward nets.

The exponent in the power-law fit of \fig{F:BN-Res_b} is notably set to $\tau\equiv\tfrac{1}{2}(\langle \delta \chi^{l,1} \rangle^{2H}-1)$, with the feedforward increment $\langle \delta \chi^{l,1} \rangle$ averaged over the whole evolution. This means that \eq{E:S_increments_bn_res_Theorem2} very well describes the evolution of $\S\pl$ in practice. 

Contrary to batch-normalized feedforward nets, the signal remains well-behaved with: (i) many directions of signal variance preserved in $\reff(\rvx^{l,1})$; (ii) close to Gaussian data distribution, as indicated e.g. by $\mu_4(\rvz^{l,1})$ close to the Gaussian kurtosis of $3$. No pathology occurs.

\section{Discussion and Summary}

The novel approach that we introduced for the characterization of deep neural networks at initialization brings three main contributions: (i) it offers a unifying treatment of the broad spectrum of pathologies; (ii) it relies on mild assumptions; (iii) it easily incorporates convolutional layers, batch normalization and skip connections.

Most studies on the convergence of neural networks to Gaussian processes have until now considered the maximal depth $L$ as constant and the width in the limit $N\ul\to\infty$ for $l\leq L$. We reversed this perspective by considering the width $N\ul$ as large but still bounded and the depth in the limit $l\to \infty$. Then the mean-field approximation of $\rvy\pl$ as a Gaussian process indexed by $\rvx,\valpha$ eventually becomes invalid: 

-- In the context of vanilla nets, with e.g. an input $\fm(\rvx,\valpha)$ constant with respect to $\valpha$ and reduced to a single point of $\R^{N\uo}$ such that $\fm(\rvx\pl,\valpha)$ remains a single point of $\R^{N\ul}$. Given the evolution of \fig{F:vanilla_histo}, the $\normltwo$ norm $||\fm(\rvx\pl,\valpha)||_2^2=N\ul\nu_2(\rvx\pl)$ becomes fat-tailed distributed as $l\to\infty$. For given $\rvx,\valpha,\rc$, this means that $\rvx\pl_{\valpha,\rc}$ and $\rvy\pl_{\valpha,\rc}$ become fat-tailed distributed as $l\to\infty$.

-- In the context of batch-normalized feedforward nets, with e.g. an input $\fm(\rvx,\valpha)$ constant with respect to $\valpha$ and uniformly sampled among $M$ points positioned spherically symmetrically in $\R^{N\uo}$. Given the evolution of \fig{F:BN-FF}, spherical symmetry, together with batch normalization, implies that for any given $\rvx,\valpha,\rc$: $\E_\Tl[\rvz\pl_{\valpha,\rc}]=\E_\Tl[\nu_{1,\rc}(\rvz\pl)]=0$, $\E_\Tl[(\rvz\pl_{\valpha,\rc})^2]=\E_\Tl[\mu_{2,\rc}(\rvz\pl)]=1$, and $\E_\Tl[(\rvz\pl_{\valpha,\rc})^4]=\E_\Tl[\mu_{4}(\rvz\pl)]\gg1$. For given $\rvx,\valpha,\rc$, this means that $\rvz\pl_{\valpha,\rc}$ and $\rvy\pl_{\valpha,\rc}$ become fat-tailed distributed as $l\to\infty$.

Similar observations were made in previous works. \citet{Duvenaud14} found that the composition of Gaussian processes eventually leads to lognormal and ill-behaved derivatives; \citet{Matthews18} found that the convergence to Gaussianity as $N\ul\to\infty$ becomes slower with respect to $N\ul$ as the depth $l$ grows. This stems from the fact that the affine transform at each layer is \emph{additive} with respect to the width dimension, but layer composition is \emph{multiplicative} with respect to the depth dimension. Intuitively, the Central Limit Theorem implies that $\rvy\pl$ becomes normally distributed as $N\ul\to \infty$, but lognormally distributed (with fat-tail) as $l\to\infty$.

Beside from this insight, our approach enabled us to characterize deep neural networks with the most common choices of hyperparameters:

-- In the case of vanilla nets, the initialization \citet{He15} limits the evolution of second-order moments of signal and noise. Combined with the limited growth of $\S\pl$, this results in the convergence to the pathology of one-dimensional signal: $\reff(\rvx\pl) \to 1$ and the convergence to neural network pseudo-linearity, with each additional layer $l$ becoming arbitrarily well approximated by a linear mapping.

-- In the case of batch-normalized feedforward nets, the pathology of exploding sensitivity: $\S\pl \geq \exp(\gamma l) \to \infty$ for some $\gamma>0$ has two origins: on the one hand, batch normalization which upweights low-signal pre-activation directions; on the other hand, the nonlinearity $\phi$. 

-- Finally in the case of resnets, $\S\pl$ only grows as a power-law. Equivalently, the evolution with depth of resnets is the logarithmic version of the evolution with depth of feedforward nets. The underlying phenomenon is the dilution of the residual path into the skip-connection path with ratio of signal variances decaying as $1/(l+1)$. This mechanism is responsible for breaking the circle of depth multiplicativity which causes pathologies for feedforward nets.

\section*{Acknowledgements}

Many thanks are due to Jean-Baptiste Fiot for his precious feedback on initial drafts and to the anonymous reviewers for their insightful comments.

\allowdisplaybreaks
\expandafter\def\expandafter\UrlBreaks\expandafter{\UrlBreaks
  \do\a\do\b\do\c\do\d\do\e\do\f\do\g\do\h\do\i\do\j%
  \do\k\do\l\do\m\do\n\do\o\do\p\do\q\do\r\do\s\do\t%
  \do\u\do\v\do\w\do\x\do\y\do\z\do\A\do\B\do\C\do\D%
  \do\E\do\F\do\G\do\H\do\I\do\J\do\K\do\L\do\M\do\N%
  \do\O\do\P\do\Q\do\R\do\S\do\T\do\U\do\V\do\W\do\X%
  \do\Y\do\Z}
  
\bibliography{characterizing_DNN}
\bibliographystyle{icml2019}

\newpage
\onecolumn
\appendix

\section{Details of the Experiments}
\label{S:experimental_details}
\fig{F:chi_visualization} considered an input $\rvx\po$ as a Gaussian mixture, with $\rvx\po \sim \mathcal{N}(-1, 0.3^2)$ with probability $1/2$ and $\rvx\po \sim \mathcal{N}(1, 0.3^2)$ with probability $1/2$. This input $\rvx\po$ was propagated into: (a) a single layer with $\phi=\verb|tanh|$; (b) a single layer with $\phi$ linear; (c) a batch-normalized feedforward net with: $\phi=\relu$ and $N\uk=100$ for $1\leq k< 10$; $\phi$ linear and $N\ul=1$ for $l=10$.

The experiments of \fig{F:vanilla_histo},~\ref{F:vanilla},~\ref{F:BN-FF},~\ref{F:BN-Res} were made on \verb|cifar10| with a random initial convolution of stride $2$ reducing the spatial dimension from $32$ to $n=16$ and increasing the width from $3$ to $N\uo$. In each case, we considered the convolutional extent $K\ul=3$ and periodic boundary conditions. 

In \fig{F:vanilla_histo}, we considered the width $N\ul=128$ and the total depth $L=200$. For each realization, we randomly initialized model parameters following \citet{He15} and randomly sampled $M=1,024$ images to constitute the input data distribution. For each realization, we then computed the evolution with depth of $\log \nu_2(\rvx\pl)-\log \nu_2(\rvx\po)$ and $\log \mu_2(\dx\pl)-\log \mu_2(\dx\po)$. The distributions of $\log \nu_2(\rvx\pl)-\log \nu_2(\rvx\po)$ and $\log \mu_2(\dx\pl)-\log \mu_2(\dx\po)$ shown in \fig{F:vanilla_histo} were estimated using $10,000$ such realizations. The limited width -- slightly smaller than standard values -- had the purpose of limiting computation time in order to gather more realizations.

In \fig{F:vanilla},~\ref{F:BN-FF},~\ref{F:BN-Res}, we increased the width to $N\ul=512$. For each realization, we randomly initialized model parameters following \citet{He15} and randomly sampled $M=64$ images to constitute the input data distribution. We then computed the evolution with depth of all moment-related quantites. For each quantity, the expectation as well as the $1\sigma$ intervals displayed in \fig{F:vanilla},~\ref{F:BN-FF},~\ref{F:BN-Res} were estimated using $1,000$ such realizations. 

Let us make a few remarks:

-- The limited number of images $M$ for each experiment enabled to reduce the computation time, in particular penalized by the computation of $\reff(\rvx\pl)$, $\reff(\dx\pl)$, $\reff(\rvx^{l,1})$, $\reff(\dx^{l,1})$ in \fig{F:vanilla},~\ref{F:BN-FF},~\ref{F:BN-Res}. For batch-normalized feedforward nets and batch-normalized resnets, choosing $M$ in the range of standard batch sizes also had the advantage that our setup of batch normalization in \emph{test mode} matched the usual setup of batch normalization in \emph{training mode}.

For vanilla nets in \fig{F:vanilla_histo},~\ref{F:vanilla} and batch-normalized resnets in \fig{F:BN-Res}, this reduction of $M$ had very little impact. For batch-normalized feedforward nets in \fig{F:BN-FF}, on the other hand, this reduction of $M$ had the effect of limiting pathologies in the signal. This can be understood by considering $M\pp$ batch-normalized random points $(\rvz_{0},\dots,\rvz_{M\pp})$. In our case, $M\pp$ is proportional to $M$ but $M\pp>M$ since the data distribution depends on the input $\rvx$ \emph{and} the \mbox{spatial position $\valpha$}. By considering the worst-case scenario such that $(\rvz_{0},\dots,\rvz_{M\pp})=(-a,\dots,-a,b,-a,\dots,-a)$:
\begin{gather*}
\frac{1}{M\pp} \sum_i \rvz_{i} = \frac{- (M\pp-1)a + b}{M\pp},
\quad \frac{1}{M\pp} \sum_i (\rvz_{i})^2 = \frac{(M\pp-1)a^2 + b^2}{M\pp},
\quad \frac{1}{M\pp} \sum_i (\rvz_{i})^4 = \frac{(M\pp-1)a^4 + b^4}{M\pp}, \\
 \frac{1}{M\pp} \sum_i \rvz_{i} = 0, 
\quad \frac{1}{M\pp} \sum_i (\rvz_{i})^2 = 1
\implies a = \frac{1}{\sqrt{M\pp-1}}, \quad b = \sqrt{M\pp-1},
\quad \frac{1}{M\pp} \sum_i (\rvz_{i})^4 = \frac{1+(M\pp-1)^3}{M\pp(M\pp-1)}.
\end{gather*}

This shows that the empirical kurtosis of $(\rvz_{0},\dots,\rvz_{M\pp})$ is roughly bounded by $M\pp$, i.e. that the pathologies of the signal are naturally limited by the number of input images $M$. As a result, for larger $M$ we found that: (i) $\reff(\rvx\pl)$ gets closer to $1$; (ii) $\mu_4(\rvz\pl)$ gets even larger and $\nu_1(|\rvz\pl|)$ gets even smaller; (iii) $|\exp( \overline{m}\uBN[ \S\pl ] )-1|$ and $|\delta\ubn \S\pl-1|$ get larger; (iv) $|\exp( \overline{m}\uPHI[ \S\pl ] )-1|$ and $|\delta\uphi \S\pl-1|$ get even smaller.

-- The dynamics of $|\exp( \overline{m}\uBN[ \S\pl ] )-1|$ at very low depth in \fig{F:BN-FF},~\ref{F:BN-Res} stems from the input images from \verb|cifar10| having a number of channels equal to $3 \ll N\ul=512$. The signal is therefore ill-conditioned at very low depth and quickly gets better conditioned, implying that $|\exp( \overline{m}\uBN[ \S\pl ] )-1|$ is non-negligible at very low depth and quickly gets vanishing. This dynamics is brief and occurs before the settling of the main dynamics which leads in particular to the conditioning of the signal degrading again in \fig{F:BN-FF}.

-- We tested to set more realistic values for the width $N\ul$ in the experiment of \fig{F:vanilla_histo}. We always observed an absolutely equivalent behaviour apart from the diffusion getting slower with larger $N\ul$.

-- We tested to change the boundary conditions from periodic to reflective and to zero-padding. We always observed an equivalent behaviour with reflective conditions. As for zero-padding conditions: (i) the evolution of vanilla nets was slightly changed with $\reff(\rvx\pl)$ converging to a value of roughly $2$ instead of $1$ due to the creation of new signal directions by zero-padding; (ii) the evolution of batch-normalized feedforward nets and batch-normalized resnets were always equivalent.

-- We tested to change the dataset from \verb|cifar10| to \verb|mnist|, with the random initial convolution of stride $2$ reducing the spatial dimension from $28$ to $n=14$ and increasing the width from $1$ to $N\uo$. We observed an equivalent behaviour apart from the signal being slightly more fat-tailed at low depth due to the original images being more fat-tailed in \verb|mnist| than~in~\verb|cifar10|.

-- Finally we tested to change the fuzz parameter $\epsilon$ of batch normalization. The experiments of \fig{F:BN-FF},~\ref{F:BN-Res} used the standard value $\epsilon=0.001$ but we observed an indistinguishable behaviour when using the value $\epsilon=0$.

\section{Complementary Definitions and Notations}
\label{S:complementary_definitions_notations}

In this section, we use again $\cbullet\pl$ as placeholder for any tensor of layer $l$ in the simultaneous propagation of $(\rvx\pl,\dx\pl)$.

\subsection{Receptive Field}

\textbf{Receptive Field Mapping.} Let us consider the convolution at layer $l$ of an input $\cbullet\plm \in \R^{n \times \dots \times n \times N\ulm}$ from layer $l-1$. The output feature map of the convolution $(\vomega\pl \ast \cbullet\plm)_{\valpha,:}$ at position $\valpha \in \{1, \dots, n \}^d$ is obtained by the application of the convolution kernel $\vomega\pl$ over a local input region from $\cbullet\plm$ of size $K\ul^d N\ulm$, with $K\ul^d$ the spatial extent and $N\ulm$ the channel extent. The local input region is called the \emph{receptive field} of $\vomega\pl \ast \cbullet\plm$ at spatial position $\valpha$.

The \emph{receptive field mapping} $\RF$ associates $\cbullet\plm$ to the tensor $\RF(\cbullet\plm) \in \R^{n \times \dots \times n \times K\ul^d N\ulm}$, with $\RF(\cbullet\plm)_{\valpha,:}$ the reshaped vectorial form of the receptive field of $\vomega\pl \ast \cbullet\plm$ at spatial position $\valpha$. We denote $R\ul=K\ul^d N\ulm$ the dimensionality of $\RF(\cbullet\plm)_{\valpha,:}$ and $\gI\pl_\rc$ the set of indices in $\RF(\cbullet\plm)_{\valpha,:}$ corresponding to elements in channel $\rc$ in $\cbullet\plm$. Strictly speaking, $\RF$ depends on $l$ but this is implied by the argument, so we write $\RF$ for simplicity. \\[-6pt]

\textbf{Receptive Field Vectors.} The \emph{receptive field vector} and \emph{centered receptive field vector} associated with $\cbullet\plm$ are defined as  \\[-6pt]
\begin{align*}
\rf(\cbullet\plm,\valpha) \equiv \RF(\cbullet\plm)_{\valpha,:} \; \;\;   \textnormal{and} \; \;\;\;
\rfc(\cbullet\plm,\valpha)  \equiv \RF(\cbullet\plm)_{\valpha,:} - \E_{\rvx,\dx,\valpha}[ \RF(\cbullet\plm)_{\valpha,:} ], \\[-15pt]\nonumber
\end{align*}

where, slightly abusively, we overloaded the notation $\rvx,\dx,\valpha,\cbullet\plm$ in the expectation. Again, strictly speaking, $\rf$ and $\rfc$ depend on $l$ but this is implied by the argument.  \\[-6pt]

\subsection{Propagation with Receptive Field Formulation}

\textbf{Equation of Propagation.} Using the definition of $\RF$, the affine transformation from the receptive field $\RF(\rvx\plm)_{\valpha, :}$ to the feature map in the next layer $\rvy\pl_{\valpha, :}$ can be written as   \\[-9pt]
\begin{align}
\label{E:affine_transformation}
\rvy\pl_{\valpha, :} = \mW\pl \RF(\rvx\plm)_{\valpha, :} + \rvb\pl= \mW\pl \RF(\rvx\plm)_{\valpha, :} + \vbeta\pl_{\valpha,:},
\end{align}

with $\mW\pl \in \R^{N\ul \times R\ul}$ the suitably reshaped matricial form of $\vomega\pl$. To lighten notation, we write $\displaystyle \rvy\pl = \mW\pl \RF(\rvx\plm) + \vbeta\pl$ as a short for the affine transformation of \eq{E:affine_transformation} occuring at all spatial positions $\valpha$. We have the following equivalence between the notations with receptive field and convolution:  \\[-6pt]
\begin{align*}
\mW\pl \RF(\rvx\plm)  + \vbeta\pl = \vomega\pl \ast \rvx\plm + \vbeta\pl. \\[-15pt] \nonumber
\end{align*}

For vanilla nets, the simultaneous propagation of $(\rvx\pl,\dx\pl)$ can be written as  \\[-9pt]
\begin{alignat*}{4}
& \rvy\pl && = \mW\pl \RF ( \rvx\plm ) + \vbeta\pl, \qquad && \dy\pl && = \mW\pl \RF ( \dx\plm ), \\
& \rvx\pl && = \phi(\rvy\pl), \qquad && \dx\pl && = \phi\pp(\rvy\pl) \odot \dy\pl. \\[-15pt] \nonumber
\end{alignat*}

For batch-normalized feedforward nets, the simultaneous propagation of $(\rvx\pl,\dx\pl)$ can be written as   \\[-9pt]
\begin{alignat*}{4}
& \rvy\pl && = \mW\pl \RF ( \rvx\plm ) + \vbeta\pl, \qquad &&  \dy\pl && = \mW\pl \RF ( \dx\plm ), \\
& \rvz\pl && = \BN(\rvy\pl), \qquad && \dz\pl && = \BN\pp (\rvy\pl) \odot \dy\pl, \\
& \rvx\pl && = \phi(\rvz\pl), \qquad && \dx\pl && = \phi\pp(\rvz\pl) \odot \dz\pl. 
\end{alignat*}

\subsection{Symmetric Propagation}

\textbf{Symmetric Propagation for Vanilla Nets.} We define additional tensors obtained by \emph{symmetric propagation} at each layer $l$. For vanilla nets, they are given by  \\[-9pt]
\begin{alignat*}{4}
& \by\pl && = -\mW\pl \RF ( \rvx\plm ) - \vbeta\pl, \qquad && \bdy\pl && = -\mW\pl \RF ( \dx\plm ), \\
& \bx\pl && = \phi(\by\pl), \qquad &&   \bdx\pl && = \phi\pp(\by\pl) \odot \bdy\pl. \nonumber \\[-15pt] \nonumber
\end{alignat*}

Under standard initialization, \emph{the tensor moments have the same distribution with respect to $\theta\pl$ for both propagations}. Furthermore, $\forall \valpha,\rc$: $\rvx\pl_{\valpha,\rc}+\bx\pl_{\valpha,\rc}=\big| \rvy\pl_{\valpha,\rc} \big|$ and $\rvx\pl_{\valpha,\rc} \vss \bx\pl_{\valpha,\rc}=0$, implying that $\forall \valpha,\rc$: $( \rvx\pl_{\valpha,\rc})^2+( \bx\pl_{\valpha,\rc})^2=( \rvy\pl_{\valpha,\rc})^2$. Thus $\forall \rc$:   \\[-9pt]
\begin{align}
 \nu_{2,\rc}( \rvx\pl ) + \nu_{2,\rc}( \bx\pl ) & = \nu_{2,\rc}(\rvy\pl). \label{E:symmetric_nux}  \\[-15pt] \nonumber
\end{align}

Now let us consider the second-order moments of the noise tensor:  \\[-6pt] 
\begin{align}
(\dx\pl_{\valpha,\rc})^2 + (\bdx\pl_{\valpha,\rc})^2 & = (\dy\pl_{\valpha,\rc})^2 \phi\pp(\rvy_{\valpha,\rc})^2 + (\bdy\pl_{\valpha,\rc})^2 \phi\pp(\by_{\valpha,\rc})^2 = (\dy\pl_{\valpha,\rc})^2 [\phi\pp(\rvy_{\valpha,\rc})^2 + \phi\pp(\by_{\valpha,\rc})^2] = (\dy\pl_{\valpha,\rc})^2, \label{E:symmetric_mus1} \\[-15pt] \nonumber
\end{align} 

where \eq{E:symmetric_mus1} was obtained using $\bdy\pl_{\valpha,\rc}=-\dy\pl_{\valpha,\rc}$ and $\rvy\pl_{\valpha,\rc}=-\by\pl_{\valpha,\rc}$, as well as the convention $\phi\pp(0)\equiv 1/2$. Since $\dx\pl$, $\bdx\pl$, $\dy\pl$ are centered, it follows that $\forall \rc$:  \\[-9pt]
\begin{align}
 \mu_{2,\rc}( \dx\pl ) + \mu_{2,\rc}( \bdx\pl )=\nu_{2,\rc}( \dx\pl ) + \nu_{2,\rc}( \bdx\pl )
 = \nu_{2,\rc}(\dy\pl) = \mu_{2,\rc}(\dy\pl). \label{E:symmetric_mus} \\[-12pt] \nonumber
\end{align}

\textbf{Symmetric Propagation for Batch-Normalized Feedforward Nets.} For batch-normalized feedforward nets, the symmetric propagation at each layer $l$ is given by   \\[-9pt]
\begin{alignat}{6}
& \by\pl && = -\mW\pl \RF ( \rvx\plm ) - \vbeta\pl, \qquad &&  \bdy\pl && = -\mW\pl \RF ( \dx\plm ), \label{E:rf_propagation_symmetric1} \\
& \bz\pl && = \BN(\by\pl), \qquad && \bdz\pl && = \BN\pp (\by\pl) \odot \bdy\pl, \label{E:rf_propagation_symmetric2} \\
& \bx\pl && = \phi(\bz\pl), \qquad &&  \bdx\pl && = \phi\pp(\bz\pl) \odot \bdz\pl. \label{E:rf_propagation_symmetric3}  \\[-15pt] \nonumber
\end{alignat}

$\BN$ in \eq{E:rf_propagation_symmetric2} uses the statistics of $\by\pl$ such that, under standard initialization, \emph{the tensor moments have the same distribution with respect to $\theta\pl$ for both propagations}. We then simply have  \\[-9pt]
\begin{align}
\bz\pl=-\rvz\pl, \qquad \bdz\pl=-\dz\pl. \label{E:symmetric_bn_identity} \\[-15pt]
\end{align} 

The same analysis as before gives $\forall \rc$:   \\[-9pt]
\begin{align}
\nu_{2,\rc}( \rvx\pl ) + \nu_{2,\rc}( \bx\pl ) & = \nu_{2,\rc}(\rvz\pl), \label{E:symmetric_bn_nux} \\
\mu_{2,\rc}(\dx\pl)+\mu_{2,\rc}(\bdx\pl) & = \mu_{2,\rc}(\dz\pl). \label{E:symmetric_bn_mus}
\end{align}

\subsection{Gramian and Covariance Matrices}

We adopt the standard definition of the \emph{Gramian matrices} of $\fm(\cbullet\plm,\valpha)$, $\fmc(\cbullet\plm,\valpha)$, $\rf(\cbullet\plm,\valpha)$, $\rfc(\cbullet\plm,\valpha)$:  \\[-9pt]
\begin{align*}
\mG_{\rvx,\dx,\valpha}[\fm(\cbullet\plm,\valpha)] & \equiv \E_{\rvx,\dx,\valpha}\big[ \fm(\cbullet\plm,\valpha) \fm(\cbullet\plm,\valpha)^T  \big], \\
\mG_{\rvx,\dx,\valpha}[\fmc(\cbullet\plm,\valpha)] & \equiv \E_{\rvx,\dx,\valpha}\big[ \fmc(\cbullet\plm,\valpha) \fmc(\cbullet\plm,\valpha)^T  \big], \\
\mG_{\rvx,\dx,\valpha}[\rf(\cbullet\plm,\valpha)] & \equiv \E_{\rvx,\dx,\valpha}\big[ \rf(\cbullet\plm,\valpha) \rf(\cbullet\plm,\valpha)^T  \big], \\
 \mG_{\rvx,\dx,\valpha}[\rfc(\cbullet\plm,\valpha)] & \equiv \E_{\rvx,\dx,\valpha}\big[ \rfc(\cbullet\plm,\valpha) \rfc(\cbullet\plm,\valpha)^T  \big]. \\[-12pt]
\end{align*}

Then, the \emph{covariance matrices} of $\fm(\cbullet\plm,\valpha)$, $\fmc(\cbullet\plm,\valpha)$, $\rf(\cbullet\plm,\valpha)$, $\rfc(\cbullet\plm,\valpha)$ are defined as  \\[-9pt]
\begin{alignat*}{3}
& \mC_{\rvx,\dx,\valpha}[\fm(\cbullet\plm,\valpha)] && =\mC_{\rvx,\dx,\valpha}[\fmc(\cbullet\plm,\valpha)] && = \mG_{\rvx,\dx,\valpha}[\fmc(\cbullet\plm,\valpha)], \\ 
& \mC_{\rvx,\dx,\valpha}[\rf(\cbullet\plm,\valpha)] && =\mC_{\rvx,\dx,\valpha}[\rfc(\cbullet\plm,\valpha)] && = \mG_{\rvx,\dx,\valpha}[\rfc(\cbullet\plm,\valpha)].\\[-6pt] 
\end{alignat*}

\subsection{Statistics-Preserving Property}
\label{S:relaxing_assumption}

\textbf{Statistics-Preserving Property.} $\RF$ is \emph{statistics-preserving} with respect to $\cbullet\plm$ if for any channel $\rc$ and any index $i_\rc \in \gI\pl_\rc$, the random variables $\RF(\cbullet\plm)_{\valpha,i_\rc}=\rf(\cbullet\plm,\valpha)_{i_\rc}$ and $\cbullet\plm_{\valpha, \rc}=\fm(\cbullet\plm,\valpha)_{\rc}$, which depend on $\rvx$, $\dx$, $\valpha$, have the same distribution: $\RF(\cbullet\plm)_{\valpha,i_\rc}=\rf(\cbullet\plm,\valpha)_{i_\rc}~\sim_{\rvx,\dx,\valpha}~\cbullet\plm_{\valpha, \rc}=\fm(\cbullet\plm,\valpha)_{\rc}$.  

First we will prove that $\RF$ is statistics-preserving with respect to $\rvx\plm$, $\dx\plm$ when convolutions have periodic boundary conditions and the global spatial extent $n$ is constant. Afterwards, we will provide a possible relaxation of these assumptions. The global spatial extent will be denoted as $n\ul$ when it is non-constant.

\subsubsection{Case of Periodic Boundary Conditions and Constant Spatial Extent $n\ul=n$}

\begin{Lemma}
\label{L:statistics_preserving_constant}
If convolutions have periodic boundary conditions and the global spatial extent $n$ is constant, then $\RF$ is statistics-preserving with respect to any input $\cbullet\plm$ from layer $l-1$. 
\end{Lemma}

\begin{proof}\textbf{.}
Fix a channel $\rc$ in $\cbullet\plm$, an index $i_\rc \in \gI\pl_\rc$, and consider the tensors $\cbullet\plm_{:,\rc}$, $\RF(\cbullet\plm)_{:,i_\rc} \in \R^{n \times \dots \times n}$. The index $i_\rc$ corresponds to a given convolution kernel position $\vkappa \in \{1, \dots, K\ul \}^d$. Under periodic boundary conditions, this fixed kernel position $\vkappa$ implies that each position $\valpha$ in $\RF(\cbullet\plm)_{\valpha,i_\rc}$ originates from a different position $\valpha\pp$ in the tensor $\cbullet\plm_{\valpha\pp,\rc}$. Therefore the index mapping $f: \valpha \rightarrow \valpha\pp$ from $\{1, \dots, n\}^d$ to $\{1, \dots, n \}^d$ is bijective. We then have $\RF(\cbullet\plm)_{\valpha, i_\rc} = \cbullet\plm_{f(\valpha),\rc} \sim_\valpha \cbullet\plm_{\valpha, \rc}$ when $\cbullet\plm$ is deterministic and $\valpha$ is random. In turn, this implies that $\RF(\cbullet\plm)_{\valpha, i_\rc} \sim_{\rvx,\dx,\valpha} \cbullet\plm_{\valpha, \rc}$, when $\rvx,\dx,\valpha$ are random. \qed
\end{proof} \\[-6pt]

\begin{Proposition}
\label{P:statistics_preserving_constant}
If convolutions have periodic boundary conditions and the global spatial extent $n$ is constant, then $\RF$ is statistics-preserving with respect to $\rvx\plm$ and $\dx\plm$.
\end{Proposition}

\begin{proof}\textbf{.} This follows immediately from Lemma \ref{L:statistics_preserving_constant}. \qed
\end{proof} \\[-6pt]

\begin{corollary}
\label{C:statistics_preserving_constant}
For any channel $\rc$ and $\displaystyle i_\rc \in \gI\pl_\rc$, we have $\rf(\rvx\plm,\valpha)_{i_\rc} \sim_{\rvx, \valpha} \fm(\rvx\plm,\valpha)_\rc$ and $\rf(\dx\plm,\valpha)_{i_\rc}\sim_{\rvx,\dx,\valpha} \fm(\dx\plm,\valpha)_\rc$. Since the cardinality $|\gI\pl_\rc|=K^d\ul$ is the same for all channels $\rc$, it follows that \\[-6pt]
\begin{align*}
 \nu_2(\rvx\plm) &= \frac{1}{N\ulm} \Tr \mG_{\rvx,\valpha}[\fm(\rvx\plm,\valpha)] = \frac{1}{R\ul} \Tr \mG_{\rvx,\valpha}[\rf(\rvx\plm,\valpha)], \\
 \mu_2(\rvx\plm) &=  \frac{1}{N\ulm} \Tr \mC_{\rvx,\valpha}[\fm(\rvx\plm,\valpha)] =  \frac{1}{R\ul} \Tr \mC_{\rvx,\valpha}[\rf(\rvx\plm,\valpha)], \\
 \nu_2(\dx\plm)=\mu_2(\dx\plm) &= \frac{1}{N\ulm} \Tr \mC_{\rvx,\dx,\valpha}[\fm(\dx\plm,\valpha)] = \frac{1}{R\ul} \Tr \mC_{\rvx,\dx,\valpha}[\rf(\dx\plm,\valpha)].  \\[-15pt]
\end{align*}

Note that this result always holds in the fully-connected case $n\ul=1$, characterized by $\rf(\rvx\plm,\valpha)=\fm(\rvx\plm,\valpha)$, $\rf(\dx\plm,\valpha)=\fm(\dx\plm,\valpha)$ and $R\ul=N\ulm$.  \\[-12pt]
\end{corollary}

\subsubsection{Case of Large Spatial Extent $n\ul \gg K\ul$}
\begin{Proposition}
If the convolution stride is one (i.e. $n\ulm=n\ul$) in most layers and the global spatial extent is much larger than the convolutional spatial extent (i.e. $n\ul \gg K\ul$) in most layers, then, for any boundary conditions, $\RF$ is approximately statistics-preserving with respect to $\rvx\plm$ and $\dx\plm$.
\end{Proposition}

\begin{proof}\textbf{.} Fix a layer $l-1$ such that $n\ulm=n\ul$ and $n\ul \gg K\ul$. Denote $\RF^{(\textnormal{p})}$ the receptive field mapping associated with periodic boundary conditions. Since $n\ulm=n\ul \gg K\ul$ the receptive fields $\RF(\rvx\plm)_{\valpha,:},\RF(\dx\plm)_{\valpha,:}$ and $\RF^{(\textnormal{p})}(\rvx\plm)_{\valpha,:}, \RF^{(\textnormal{p})}(\dx\plm)_{\valpha,:}$ do not intersect boundary regions for most $\valpha$, implying for most $\valpha$: \\[-6pt]
\begin{align*}
\RF(\rvx\plm)_{\valpha,:}=\RF^{(\textnormal{p})}(\rvx\plm)_{\valpha,:}, \qquad \RF(\dx\plm)_{\valpha,:}=\RF^{(\textnormal{p})}(\dx\plm)_{\valpha,:}. \\[-15pt]
\end{align*}

This implies for any index $i_\rc$ that $P_{\rvx,\valpha}\big[ \RF(\rvx\plm)_{\valpha,i_\rc} \big] \simeq P_{\rvx,\valpha}\big[ \RF^{(\textnormal{p})}(\rvx\plm)_{\valpha,i_\rc} \big]$ and $P_{\rvx,\dx,\valpha}\big[ \RF(\dx\plm)_{\valpha,i_\rc} \big] \simeq P_{\rvx,\dx,\valpha}\big[ \RF^{(\textnormal{p})}(\dx\plm)_{\valpha,i_\rc} \big]$.

Since $\RF^{(\textnormal{p})}$ is statistics-preserving with respect to $\rvx\plm$ and $\dx\plm$ by Lemma \ref{L:statistics_preserving_constant}, it follows for any channel $\rc$ and index $\displaystyle i_\rc \in \gI\pl_\rc$ that $P_{\rvx,\valpha}\big[ \RF^{(\textnormal{p})}(\rvx\plm)_{\valpha,i_\rc} \big]=P_{\rvx,\valpha}\big[ \rvx\plm_{\valpha,\rc} \big]$ and $P_{\rvx,\dx,\valpha}\big[ \RF^{(\textnormal{p})}(\dx\plm)_{\valpha,i_\rc} \big]=P_{\rvx,\dx,\valpha}\big[ \dx\plm_{\valpha,\rc} \big]$. We then deduce that $P_{\rvx,\valpha}\big[ \RF(\rvx\plm)_{\valpha,i_\rc} \big] \simeq P_{\rvx,\valpha}\big[ \rvx\plm_{\valpha,\rc} \big]$ and $P_{\rvx,\dx,\valpha}\big[ \RF(\dx\plm)_{\valpha,i_\rc} \big] \simeq P_{\rvx,\dx,\valpha}\big[ \dx\plm_{\valpha,\rc} \big]$, meaning that $\RF$ is approximately statistics-preserving with respect to $\rvx\plm$ and $\dx\plm$. \qed
\end{proof}

\section{Details of \Section{S:data_randomness} and \Section{S:model_parameters_randomness}}

\subsection{Approximation of \mbox{$\Phi\ul(\rvx+\dx) - \Phi\ul(\rvx)$} by $\dx\pl$}
\label{S:assumption_noise}

We use the definitions and notations from \Section{S:complementary_definitions_notations} in the context of the propagation of \eq{E:noisy_propagation1} and \eq{E:noisy_propagation2}. We further suppose that a.s. with respect to $\rvx$: $\exists r>0$ such that $\Phi\ul$ is differentiable in the open ball $\mathcal{B}_r(\rvx)$ of radius $r$ at point $\rvx$ (see \Section{S:assumption_differentiable} for the justification).

We will prove that \\[-9pt]
\begin{align} 
\frac{\mu_2( \Phi\ul(\rvx+\dx) - \Phi\ul(\rvx) - \dx\pl )}{\mu_2(\dx\pl)} \to 0 \textnormal{ \;\;as $\sigma_{\dx}\to0$ (with fixed distributions of $\rvx$ and $\dx/\sigma_{\dx}$)}. \label{E:assumption_noise1}
\end{align}

Due to the $1$-Lipschitzness of $\phi=\relu$, under periodic boundary conditions, we have that $\forall \rvt, \rvu, \rvv,\rvw$: \\[-9pt]
\begin{align*} 
\Big(\phi \big( \mW\pl_{\rc,:} \rf(\rvt, \valpha) + \vbeta\pl \big) - \phi \big( \mW\pl_{\rc,:} \rf(\rvu, \valpha) + \vbeta\pl \big) \Big)^2 & \leq || \mW\pl ||^2 \cdot || \rf(\rvt, \valpha) - \rf(\rvu, \valpha) ||^2_2 \leq || \mW\pl ||^2 \cdot || \vec(\rvt-\rvu) ||^2_2, \\
\Big(\phi\pp(\rvv_{\valpha,\rc}) \cdot \mW\pl_{\rc,:} \rf(\rvw, \valpha) \Big)^2 & \leq || \mW\pl ||^2 \cdot  || \rf(\rvw, \valpha) ||^2_2 \leq || \mW\pl ||^2 \cdot || \vec(\rvw) ||^2_2,
\end{align*}

with $|| \mW\pl ||$ the spectral norm of $\mW\pl$. It follows that $\forall \rvx, \dx$:  \\[-9pt]
\begin{align*} 
|| \vec(\Phi\ul(\rvx+\dx) - \Phi\ul(\rvx)) ||^2_2 \leq \Bigg( \prod^l_{k=1} n^{d} N\ul || \mW\pl ||^2 \Bigg) \cdot || \vec(\dx) ||^2_2, \\
|| \vec(\dx\pl) ||^2_2 \leq \Bigg( \prod^l_{k=1} n^{d} N\ul || \mW\pl ||^2  \Bigg) \cdot || \vec(\dx) ||^2_2.
\end{align*}

This gives: \\[-9pt]
\begin{align*} 
|| \vec(\Phi\ul(\rvx+\dx) - \Phi\ul(\rvx) - \dx\pl) ||^2_2 
 & \leq 2 || \vec(\Phi\ul(\rvx+\dx) - \Phi\ul(\rvx)) ||^2_2 + 2 || \vec(\dx\pl) ||^2_2  \\
 & \leq 4 \Bigg( \prod^l_{k=1} n^{d} N\ul || \mW\pl ||^2  \Bigg) \cdot || \vec(\dx) ||^2_2 \\
 & \leq C || \vec(\dx) ||^2_2, 
\end{align*}

with $C = 4 \cdot \prod^l_{k=1} n^{d} N\ul || \mW\pl ||^2$.

The assumption on the differentiability of $\Phi\ul$ implies that $\forall \epsilon >0$, $\exists \eta_\epsilon>0$: $\sP_{\rvx}\big[ \Phi\ul \textnormal{ is differentiable in } \mathcal{B}_{\eta_\epsilon}(\rvx) \big] \geq 1-\epsilon$. Markov's inequality applied to $||\vec(\dx)||^2_2$ further implies that \\[-9pt]
\begin{align*} 
\sP_\dx \big[ ||\vec(\dx)||_2 > \eta_\epsilon \big] = \sP_\dx \big[ ||\vec(\dx)||^2_2 > \eta^2_\epsilon \big] \leq \frac{n^d N\uo \sigma^2_{\dx}}{\eta^2_\epsilon}.
\end{align*}

It then follows that $\forall \epsilon>0$, $\exists \eta_\epsilon,\sigma_\epsilon>0$ such that $\forall \sigma_{\dx} < \sigma_\epsilon$: \\[-9pt]
\begin{align*} 
\sP_{\rvx,\dx} [ A_\epsilon ] \geq 1-2\epsilon,  \\[-12pt]
\end{align*}

with $A_\epsilon = \big\{ ||\vec(\dx)||_2 \leq {\eta_\epsilon} \big\} \cap \big\{ \Phi\ul \textnormal{ is differentiable in } \mathcal{B}_{\eta_\epsilon}(\rvx) \big\}$.

Denoting $A\pc_\epsilon$ the complementary event of $A_\epsilon$, we deduce that $\forall \sigma_{\dx} < \sigma_\epsilon$:  \\[-9pt]
\begin{align} 
\E_{\rvx,\dx} \big[ \1_{A\pc_\epsilon} || \vec\big(\Phi\ul(\rvx+\dx) - \Phi\ul(\rvx) - \dx\pl\big) ||^2_2 \big] 
& \leq \E_{\rvx,\dx} \big[ \1_{A\pc_\epsilon} C || \vec(\dx) ||^2_2 \big] \nonumber \\
& \leq C \sigma^2_\dx \E_{\rvx,\dx} \big[ \1_{A\pc_\epsilon} || \vec(\dx/ \sigma_\dx) ||^2_2 \big] \nonumber \\
& \leq C \sigma^2_\dx \sP_{\rvx,\dx}[{A\pc_\epsilon}]^\frac{1}{2} \E_{\rvx,\dx} \big[ || \vec(\dx/ \sigma_\dx) ||_2^4 \big]^\frac{1}{2} \label{E:assumption_noise2} \\
& \leq C \sigma^2_\dx (2\epsilon)^\frac{1}{2} \E_{\rvx,\dx} \big[ || \vec(\dx / \sigma_\dx) ||_2^4 \big]^\frac{1}{2}, \nonumber
\end{align}

where we used Cauchy-Schwarz inequality in \eq{E:assumption_noise2}. 

Since $\Phi\ul(\rvx+\dx) - \Phi\ul(\rvx) - \dx\pl=0$ under $A_\epsilon$, it follows that $\forall \sigma_{\dx} < \sigma_\epsilon$: \\[-9pt]
\begin{align} 
\frac{\mu_2\big( \Phi\ul(\rvx+\dx) - \Phi\ul(\rvx) - \dx\pl \big)}{\mu_2(\dx\pl)}
& = \frac{\frac{1}{n^d N\ul} \E_{\rvx,\dx} \big[ || \vec\big(\Phi\ul(\rvx+\dx) - \Phi\ul(\rvx) - \dx\pl\big) ||^2_2 \big]}{ \mu_2(\dx) \cdot \E_{\rvx,\valpha,\rc} \big[ || \vec(\nabla_\rvx \rvx\pl_{\valpha,\rc}) ||^2_2 \big]} \label{E:assumption_noise3} \\
& \leq \frac{\frac{1}{n^d N\ul} C (2\epsilon)^\frac{1}{2} \E_{\rvx,\dx} \big[ || \vec(\dx / \sigma_\dx) ||_2^4 \big]^\frac{1}{2}}{\E_{\rvx,\valpha,\rc} \big[ || \vec(\nabla_\rvx \rvx\pl_{\valpha,\rc}) ||^2_2 \big]} 
\leq C\pp \epsilon^\frac{1}{2}, \nonumber
\end{align}

where we used Proposition \ref{P:sensitivity_derivative} in \eq{E:assumption_noise3} and appropriately defined the constant $C\pp$. 

Let us finally consider $\epsilon\pp >0$ and $\epsilon$ such that $C\pp \epsilon^\frac{1}{2}=\epsilon\pp$. Then $\exists \sigma_\epsilon>0$ such that $\forall \sigma_{\dx} < \sigma_\epsilon$: \\[-9pt]
\begin{align*} 
\frac{\mu_2\big( \Phi\ul(\rvx+\dx) - \Phi\ul(\rvx) - \dx\pl \big)}{\mu_2(\dx\pl)}
\leq \epsilon\pp,
\end{align*}

which proves \eq{E:assumption_noise1}.

\subsection{Assumption that $\Phi\ul$ is Differentiable a.s. with respect to $\rvx$}
\label{S:assumption_differentiable}

The \emph{sensitivity equivalence} detailed in \Section{S:S_sensitivity} relies on the assumption that $\Phi\ul(\rvx)$ is differentiable \emph{surely} with respect to $\rvx$. If $\Phi\ul(\rvx)$ is differentiable a.s. with respect to $\rvx$, this can be relaxed using subdifferentials by noting that moments with respect to $\rvx,\dx,\valpha$ are left unchanged when ignoring zero-probability events.

Now let us justify the assumption that $\Phi\ul(\rvx)$ is differentiable a.s. with respect to $\rvx$ in the context of the propagation of \eq{E:noisy_propagation1} and \eq{E:noisy_propagation2}. We denote the receptive field vectors $\rf(\rvx\pkm,\valpha)$ as in \Section{S:complementary_definitions_notations}, and we denote $\Tl\equiv(\vomega^{1}, \vbeta^{1}, \dots, \vomega\pl, \vbeta\pl)$ as in \Section{S:model_parameters_randomness}. We further assume standard initialization. 

Let $A \equiv \big\{ \exists r>0 \textnormal{ such that } \Phi\ul \textnormal{ is differentiable in the open ball $\mathcal{B}_r(\rvx)$ of radius $r$ at point $\rvx$} \big\}$ be an event depending on $\rvx$, $\Tl$, and let $A^{\textnormal{c}}$ be the complementary event. We will prove that $\sP_{\rvx|\Tl}[A]=1$ with probability $1$ with respect to $\Tl$.

For given $\rvx$ such that $\forall \valpha$: $\rvx_{\valpha,:} \neq 0$, it is easy to see that \\[-9pt]
\begin{align*}
A^{\textnormal{c}} \implies \exists k\leq l, \exists \valpha,\rc: \; \rf(\rvx\pkm,\valpha) \neq 0 \textnormal{ and }\rvx\pk_{\valpha,\rc}=0.
\end{align*}

Under standard initialization, this corresponds to a zero-probability event with respect to $\Tl$, meaning that $\sP_{\Tl | \rvx}[A]=1-\sP_{\Tl | \rvx}[A^{\textnormal{c}}]=1$. 

Now considering $\rvx$ again as random, using Fubini's Theorem and making the assumption that $\rvx_{\valpha,:}\neq 0$ a.s. with respect to $\rvx,\valpha$ (which is the case e.g. if $\rvx_{\valpha,:}$ has well-defined probability density function):   \\[-6pt]
\begin{align}
\E_{\Tl} \sP_{\rvx|\Tl} [A] = \E_{\Tl} \E_{\rvx|\Tl} [\1_A] = \E_\rvx \E_{\Tl|\rvx} [\1_A]
  = \E_\rvx \sP_{\Tl|\rvx}[A]=1. \label{E:assumption_differentiable} \\[-15pt] \nonumber
\end{align}

By contradiction, if there would be non-zero probability with respect to $\Tl$ that $\sP_{\rvx|\Tl} [A]\neq1$, then \eq{E:assumption_differentiable} would not hold. Therefore with probability $1$ with respect to $\Tl$, $\sP_{\rvx|\Tl} [A]=1$, implying that with probability $1$ with respect to $\Tl$, $\Phi\ul(\rvx)$ is differentiable a.s. with respect to $\rvx$.

\subsection{Property of Normalized Sensitivity}
 \label{S:S_sensitivity}

\begin{Proposition}
\label{P:sensitivity_derivative}
The noise tensor $\dx\pl$ and the vectorized version of the tensor $\nabla_{\rvx} \rvx\pl_{\valpha,\rc}$, containing for given $\valpha,\rc$ the derivatives of $\rvx\pl_{\valpha,\rc}$ with respect to $\rvx=\rvx\po$, are related by: $\E_{\rvx,\valpha,\rc} \big[ || \vec(\nabla_\rvx \rvx\pl_{\valpha,\rc}) ||^2_2 \big]^\frac{1}{2}=\sqrt{\mu_2(\dx\pl)} / \sqrt{\mu_2(\dx)}=\sqrt{\mu_2(\dx\pl)} / \sqrt{\mu_2(\dx\po)}$. 
\end{Proposition}

\begin{proof}\textbf{.} Due to the definition of $\dx\pl$ as the first-order approximation of $\Phi\ul(\rvx+\dx) - \Phi\ul(\rvx)$:  \\[-12pt]
\begin{align*}
\dx\pl_{\valpha,\rc} = \big\langle \vec(\nabla_\rvx \rvx\pl_{\valpha,\rc}), \vec(\dx) \big\rangle = \big\langle \vec(\nabla_\rvx \rvx\pl_{\valpha,\rc}), \vec(\dx\po) \big\rangle, \\[-15pt] \nonumber
\end{align*}

with $\langle \, , \rangle$ the standard dot product in $\R^{n^d N\uo}$. 

Then due to the white noise property: $\E_{\dx}[ \dx_{i} \dx_{j}]= \sigma^2_{\dx} \delta_{ij}=\mu_2(\dx) \delta_{ij}=\mu_2(\dx\po) \delta_{ij}$, we deduce that  \\[-9pt]
\begin{align*}
\E_{\dx} \big[ (\dx\pl_{\valpha,\rc})^2 \big] & = \mu_2(\dx) \cdot || \vec(\nabla_\rvx \rvx\pl_{\valpha,\rc}) ||^2_2, \\
\E_{\rvx,\dx,\valpha,\rc} \big[ (\dx\pl_{\valpha,\rc})^2 \big] & = \mu_2(\dx) \cdot \E_{\rvx,\valpha,\rc} \big[ || \vec(\nabla_\rvx \rvx\pl_{\valpha,\rc}) ||^2_2 \big], \\
\Bigg( \frac{\mu_2(\dx\pl)}{\mu_2(\dx)} \Bigg)^\frac{1}{2} = \Bigg( \frac{\mu_2(\dx\pl)}{\mu_2(\dx\po)} \Bigg)^\frac{1}{2}  & = \E_{\rvx,\valpha,\rc} \big[ || \vec(\nabla_\rvx \rvx\pl_{\valpha,\rc}) ||^2_2 \big]^\frac{1}{2}. \tag*{\qed}
\end{align*}
\end{proof}

\begin{Proposition}
\label{P:normalized sensitivity}
Denoting the neural network mapping $\rvx\pl=\Phi\ul(\rvx)=\Phi\ul(\rvx\po)$ and the constant rescaling $\Psi\ul(\rvx)=\sqrt{\mu_2(\rvx\pl)} / \sqrt{\mu_2(\rvx\po)} \cdot \rvx\po=\sqrt{\mu_2(\rvx\pl)} / \sqrt{\mu_2(\rvx)} \cdot \rvx$ \; leading to the same signal variance: $\mu_2(\Psi\ul(\rvx))=\mu_2(\Phi\ul(\rvx))$, the normalized sensitivity $\S\pl$ exactly measures the excess root mean square sensitivty of the neural network mapping $\Phi\ul$ relative to the constant rescaling $\Psi\ul$:  \\[-9pt]
\begin{align*}
\S\pl =\frac{\E_{\rvx,\valpha,\rc} \big[ || \vec(\nabla_\rvx \Phi\ul(\rvx)_{\valpha,\rc}) ||^2_2 \big]^\frac{1}{2}}
{\E_{\rvx,\valpha,\rc} \big[ || \vec(\nabla_\rvx \Psi\ul(\rvx)_{\valpha,\rc}) ||^2_2 \big]^\frac{1}{2}}
=\frac{\E_{\rvx,\valpha,\rc} \big[ || \vec(\nabla_\rvx \rvx\pl_{\valpha,\rc}) ||^2_2 \big]^\frac{1}{2}}
{\E_{\rvx,\valpha,\rc} \big[ || \vec(\nabla_\rvx \Psi\ul(\rvx)_{\valpha,\rc}) ||^2_2 \big]^\frac{1}{2}}.
\end{align*}
\end{Proposition}

\begin{proof}\textbf{.} This directly follows from: (i) the definition of $\S\pl$; (ii) the result from Proposition \ref{P:sensitivity_derivative}; (iii) the fact that the constant rescaling $\Psi\ul$ has root mean square sensitivitiy equal to 
$\E_{\rvx,\valpha,\rc} \big[ || \vec(\nabla_\rvx \Psi\ul(\rvx)_{\valpha,\rc}) ||^2_2 \big]^\frac{1}{2}=\sqrt{\mu_2(\rvx\pl)} / \sqrt{\mu_2(\rvx\po)}$. \qed

\end{proof}

\subsection{Characterizing Pathologies}
 \label{S:distributional_pathologies_proof}

We consider the following mean vectors and rescaling of the signal:  \\[-6pt]
\begin{align*}
 \vnu\pl\equiv(\nu_{1,\rc}(\rvx\pl))_{1\leq \rc \leq N\ul}, 
 \qquad \tx\pl\equiv \frac{1}{||\vnu\pl||_2} \rvx\pl, 
 \qquad \tilde{\vnu}\pl\equiv(\nu_{1,\rc}(\tx\pl))_{1\leq \rc \leq N\ul} = \frac{\vnu\pl}{||\vnu\pl||_2}. \\[-15pt] \nonumber
\end{align*}

We immediately have $||\tilde{\vnu}\pl||_2=1$. Furthermore we have   \\[-6pt]
\begin{align*}
\nu_2(\rvx\pl) & =\frac{1}{N\ul} \sum_\rc \E_{\rvx,\valpha}\big[\fm(\rvx\pl,\valpha)_\rc^2 \big] \\
 & =\frac{1}{N\ul}  \Big( \sum_\rc \Var_{\rvx,\valpha}\big[\fm(\rvx\pl,\valpha)_\rc \big] + \E_{\rvx,\valpha}\big[\fm(\rvx\pl,\valpha)_\rc \big]^2  \Big) \\
 & = \frac{1}{N\ul} \Big( \sum_\rc \mu_{2,\rc}(\rvx\pl)+\nu_{1,\rc}(\rvx\pl)^2 \Big) \\
 & =\mu_2(\rvx\pl)+\frac{1}{N\ul} ||\vnu\pl||_2^2.  \\[-15pt] \nonumber 
\end{align*}

The pathology $\mu_2(\rvx\pl) / \nu_2(\rvx\pl)\xrightarrow{l\rightarrow\infty}0$ implies $ ||\vnu\pl||_2^2 / \big(N\ul \nu_2(\rvx\pl) \big) \xrightarrow{l\rightarrow\infty}1$, which in turn implies $\mu_2(\rvx\pl) / ||\vnu\pl||_2^2\xrightarrow{l\rightarrow\infty}0$, i.e. $\mu_2(\tx\pl)\xrightarrow{l\rightarrow\infty}0$. It follows that $\fm(\tx\pl,\valpha)$ becomes \emph{point-like} concentrated at point $\tilde{\vnu}\pl$ of unit $\normltwo$ norm.

\subsection{Derivation of \eq{E:overline_mx}, (\ref{E:underline_mx}) and (\ref{E:underline_sx})}
 \label{S:derivation_mx_sx}

The quantities $\overline{m}[\nu_2(\rvx\pk)]$, $\underline{m}[\nu_2(\rvx\pk)]$ and $\underline{s}[\nu_2(\rvx\pk)]$ are defined as   \\[-6pt]
\begin{align*}
\overline{m}[\nu_2(\rvx\pk)] & \equiv \log \E_{\tk}[ \delta \nu_2(\rvx\pk)],\\
\underline{m}[\nu_2(\rvx\pk)] & \equiv \E_{\tk}[ \log  \delta \nu_2(\rvx\pk)] - \log \E_{\tk}[ \delta \nu_2(\rvx\pk)],\\
\underline{s}[\nu_2(\rvx\pk)]  & \equiv \log  \delta \nu_2(\rvx\pk) - \E_{\tk}[ \log \delta \nu_2(\rvx\pk)]. \\[-15pt] \nonumber
\end{align*}

Denoting $\underline{\delta} \nu_2(\rvx\pk) \equiv \delta \nu_2(\rvx\pk) / \E_{\tk}[ \delta \nu_2(\rvx\pk)]$ the multiplicatively centered increments of $\nu_2(\rvx\pk)$, the term $\underline{m}[\nu_2(\rvx\pk)]$ can be expressed as  \\[-9pt]
\begin{align}
\underline{m}[\nu_2(\rvx\pk)] & = \E_{\tk}\Big[ \log  \Big(\underline{\delta} \nu_2(\rvx\pk) \, \E_{\tk}[ \delta \nu_2(\rvx\pk)] \Big) \Big] - \log \E_{\tk}[ \delta \nu_2(\rvx\pk)] \nonumber \\
& = \E_{\tk}[ \log \underline{\delta} \nu_2(\rvx\pk) ] + \log \E_{\tk}[ \delta \nu_2(\rvx\pk)] - \log \E_{\tk}[ \delta \nu_2(\rvx\pk)] \label{E:derivation_mx_sx1}\\
& = \E_{\tk}[ \log \underline{\delta} \nu_2(\rvx\pk) ], \nonumber  \\[-15pt] \nonumber
\end{align}

where we used $\E_{\tk}[ \log  \E_{\tk}[ \delta \nu_2(\rvx\pk)] ]=\log  \E_{\tk}[ \delta \nu_2(\rvx\pk)]$ in \eq{E:derivation_mx_sx1}. The term $\underline{s}[\nu_2(\rvx\pk)]$ can be expressed as  \\[-6pt]
\begin{align}
\underline{s}[\nu_2(\rvx\pk)] & = \log  \Big( \underline{\delta} \nu_2(\rvx\pk) \, \E_{\tk}[ \delta \nu_2(\rvx\pk)] \Big)  - \E_{\tk}\Big[ \log  \Big(  \underline{\delta} \nu_2(\rvx\pk) \, \E_{\tk}[ \delta \nu_2(\rvx\pk)]  \Big)  \Big] \nonumber \\
& = \log  \underline{\delta} \nu_2(\rvx\pk) + \log \E_{\tk}[ \delta \nu_2(\rvx\pk)]  - \E_{\tk}[ \log  \underline{\delta} \nu_2(\rvx\pk) ]  - \log \E_{\tk}[ \delta \nu_2(\rvx\pk)] \label{E:derivation_mx_sx2} \\
& = \log  \underline{\delta} \nu_2(\rvx\pk) - \E_{\tk} [ \log  \underline{\delta} \nu_2(\rvx\pk) ], \nonumber \\[-15pt] \nonumber
\end{align}

where we used again $\E_{\tk}[ \log  \E_{\tk}[ \delta \nu_2(\rvx\pk)] ]=\log  \E_{\tk}[ \delta \nu_2(\rvx\pk)]$ in \eq{E:derivation_mx_sx2}.

\section{Details of \Section{S:vanilla_nets}}
\subsection{Lemmas on Weak Convergence}
\label{weak_convergence}

\textbf{Weak Convergence.} The sequence of random variables $(X\uk)\ukN$ \emph{converges weakly} to the random variable $X$ if $\sP[ X_k~\leq~a ]~\xrightarrow{k\rightarrow\infty}~\sP[ X~\leq~a ]$ for every continuity point $a$ of the function $x\mapsto \sP[ X \leq x ]$. We then write $X\uk \Rightarrow X$. \\ 

\textbf{Tightness.} 
The sequence of random variables $(X\uk)\ukN$ is \emph{tight} if \\[-9pt] 
\begin{align*}
\forall \epsilon, \exists a_\epsilon, b_\epsilon \in \sR: \inf\nolimits_k \sP\big[ X_k \in ]a_\epsilon, b_\epsilon] \big] \geq 1-\epsilon.\\[-9pt] 
\end{align*}

\textbf{Uniform Integrability.} 
The sequence of random variables $(X\uk)\ukN$ is \emph{uniformly integrable} if \\[-9pt] 
\begin{align*}
\sup\nolimits_k \E\big[ \1_{\{|X\uk| \geq M \}} |X\uk| \big]  \xrightarrow{M\rightarrow\infty} 0.\\[-6pt] 
\end{align*}

\begin{Lemma}[Theorem 25.7 in \citet{Billingsley95}] Consider a real-valued function $h$, continuous everywhere apart from a finite set of discontinuity points $D_h=\{x_1, \dots, x_p\}$. Then $h$ is measurable and if $X\uk \Rightarrow X$ with $\sP[X \in D_h]=0$, then $h(X\uk) \Rightarrow h(X)$.   \\[-3pt] 
\label{L:function_weak_convergence}
\end{Lemma}

\begin{Lemma}[Theorem 25.10 in \citet{Billingsley95}, known as Prokhorov's theorem]
\label{L:prokhorov_weak_convergence}
If the sequence of random variables $(X\uk)\ukN$ is tight, then it admits a weakly convergent subsequence, i.e. there exists a sequence $(i_k)\ukN$ of strictly increasing indices and a random variable $X$ such that $X_{i_k} \Rightarrow X$. \\[-3pt] 
\end{Lemma}

\begin{Lemma}[Theorem 25.12 in \citet{Billingsley95}]
If the sequence of random variables $(X\uk)\ukN$ is uniformly integrable and if $X\uk \Rightarrow X$, then $X$ has well-defined expectation and $\E[X\uk] \xrightarrow{k\rightarrow\infty} \E[X]$. \\[-3pt] 
\label{L:integrable_weak_convergence}
\end{Lemma}

\subsection{Lemma on the Sum of Increments}
 \label{S:Lemma_moments_vanilla}

\begin{Lemma}
\label{L:moments_vanilla}
Let us consider a sequence $(X\uk)\ukN$ of random variables and a decreasing sequence of events $(A\uk)\ukN$, which both depend on $\Tk$. Let us suppose that $\sP_{\tk|A\ukm}[A\uk]$ does not depend on $\Tkm$ and let us denote under $A\uk$:\\[-9pt]
\begin{align*}
Y_k \equiv \E_{\tk|A\uk}[X_k], \qquad Z_k \equiv X_k-\E_{\tk|A\uk}[X_k].  \\[-18pt] \nonumber
\end{align*}

Let us further suppose that there exist constants $m\umin$, $m\umax$, $v\umin$, $v\umax$ such that $\forall k$, under $A\uk$:\\[-9pt]
\begin{align*}
m\umin \leq Y_k \leq m\umax, \qquad  v\umin \leq \Var_{\tk|A\uk}[ Z_k ] \leq v\umax.  \\[-21pt]
\end{align*}

Then it follows that \\[-15pt]
\begin{enumerate}[itemsep=0pt,parsep=0pt,label={(\roman*)}]
\item The random variables $Z_k$ are centered and non-correlated such that $\forall k$, $\forall k\pp \neq k$:   \\[-9pt]
\begin{align*}
\E_{\Tk|A\uk} [ Z_k ] = 0, \qquad \E_{\Theta^{\max(k,k\pp)}|A_{\max(k,k\pp)}} [ Z_k Z_{k\pp} ] = 0.   \\[-12pt]
\end{align*}
\item There exist random variables $m\ul$ and $s\ul$ such that under $A\ul$:  \\[-9pt]
\begin{align*}
\sum^l_{k=1} X_k = l m\ul + \sqrt{l} s\ul, \quad m\umin \leq m\ul \leq m\umax, \quad \E_{\Tl|A\ul}[s\ul] = 0, \quad v\umin \leq \Var_{\Tl|A\ul}[s\ul] \leq v\umax.
\end{align*}
\end{enumerate}
\end{Lemma}

\begin{proof}\textbf{ of (i).} 
First we show that $Z_k$ is centered under $A\uk$:  \\[-6pt]
\begin{align}
\E_{\tk|A\uk}[ Z_k ]& = \E_{\tk|A\uk}[X_k]-\E_{\tk|A\uk}[X_k]=0, \label{E:moments_vanilla_Lemma1} \\
\E_{\Tk|A\uk}[ Z_k ]& = \frac{\E_{\Tkm|A\ukm}[ \E_{\tk|A\ukm}[\1_{A\uk} Z_k] ] }{\sP_{\Tk|A\ukm}[A\uk]}
= \E_{\Tkm|A\ukm}\Bigg[ \frac{\E_{\tk|A\ukm}[\1_{A\uk} Z_k]}{\sP_{\tk|A\ukm}[A\uk]} \Bigg] 
= \E_{\Tkm|A\ukm}\big[ \E_{\tk|A\uk}[Z_k] \big] = 0. \nonumber
\end{align}

Now for $k < \kp$, we have $k \leq \kp-1$ and thus $Z_k$ is fully determined by $\Theta^{\kp-1}$. Then we can write  \\[-6pt]
\begin{align}
\E_{\Theta^{\kp}|A_{\kp}} [ Z_k Z_\kp ] 
& = \E_{\Theta^{\kp-1}|A_{\kp-1}} \E_{\theta^\kp|A_\kp} [ Z_k Z_\kp ]
 = \E_{\Theta^{\kp-1}|A_{\kp-1}} \big[ Z_k \, \E_{\theta^{\kp}|A_\kp}[ Z_\kp ] \big] = 0, \nonumber \\[-15pt] \nonumber
\end{align}

where we used \eq{E:moments_vanilla_Lemma1}. \qed
\end{proof}

\begin{proof}
\textbf{of (ii).} First we note that  \\[-6pt]
\begin{align*}
\Var_{\Tk|A\uk} \big[ Z_k \big]=\E_{\Tk|A\uk} \big[ Z^2_k \big]=\E_{\Tkm|A\ukm} \E_{\tk|A\uk} \big[ Z^2_k \big] = \E_{\Tkm|A\ukm} \Var_{\tk|A\uk} \big[ Z_k \big]. \\[-15pt] \nonumber
\end{align*}

Combined with the hypothesis that $v\umin \leq \Var_{\tk|A\uk}[ Z_k ] \leq v\umax$, we deduce that  \\[-6pt]
\begin{align}
v\umin \leq \Var_{\Tk|A\uk} \big[ Z_k \big] \leq v\umax. \label{E:moments_vanilla_Lemma3} \\[-15pt] \nonumber
\end{align}

Now let us denote $M_l \equiv \sum^l_{k=1} Y_k$ and $S_l \equiv \sum^l_{k=1} Z_k$. Then, using (i), we get that  \\[-6pt]
\begin{align}
\E_{\Tl|A\ul}\big[ S\ul \big] & = \sum\nolimits_k \E_{\Tl|A\ul}[Z_k] = 0, \nonumber \\
\Var_{\Tl|A\ul}\big[ S\ul \big] = \E_{\Tl|A\ul}\big[ S^2\ul \big] & = \sum\nolimits_{k,k\pp} \E_{\Tl|A\ul}\big[ Z_k Z_{k\pp} \big] \nonumber \\
        &  = \sum\nolimits_k \E_{\Tk|A\uk} \big[ Z^2_k \big] 
         = \sum\nolimits_k \Var_{\Tk|A\uk} \big[ Z_k \big]. \label{E:moments_vanilla_Lemma4} \\[-15pt] \nonumber
\end{align}

The hypothesis implies under $A\ul$ that $l \vss m\umin \leq M\ul \leq l \vss m\umax$, while \eq{E:moments_vanilla_Lemma3} and \eq{E:moments_vanilla_Lemma4} together imply that \mbox{$l \vss v\umin \leq \Var_{\Tl|A\ul}[S\ul] \leq l \vss v\umax$}. If we define $m\ul \equiv M\ul / l$ and $s\ul \equiv S\ul / \sqrt{l}$, then $\sum^l_{k=1} X_k = \sum^l_{k=1} Y_k + \sum^l_{k=1} Z_k$ can be written as required under $A\ul$:  \\[-6pt]
\begin{align*}
\sum\pl_{k=1} X_k = M\ul + S\ul = l m\ul + \sqrt{l} s\ul, \quad m\umin \leq m\ul \leq m\umax, \quad \E_{\Tl|A\ul}[s\ul] = 0, \quad v\umin \leq \Var_{\Tl|A\ul}[s\ul] \leq v\umax. \tag*{\qed}
\end{align*}
\end{proof}

\subsection{Proof of \theorem{T:moments_vanilla}}
 \label{S:proof_moments_vanilla}

\textbf{\theorem{T:moments_vanilla}} (moments of vanilla nets)\textbf{.} \emph{
There exist small constants $1\gg m\umin, m\umax, v\umin, v\umax >0$, random variables $m\ul, m\pp\ul$, $s\ul,s\pp\ul$ and events $A\ul$, $A\pp\ul$ of probabilities equal to $\prod^l_{k=1} ( 1 - 2^{-N_k+1 })$ such that}
\begin{alignat*}{6}
& \textnormal{Under $A\ul$: } &  \;\;  \log \Bigg( \frac{\nu_2(\rvx\pl)}{\nu_2(\rvx\po)}  \Bigg) & = - l m\ul + \sqrt{l} s\ul, \quad &
 m\umin \leq m\ul \leq m\umax, & \quad
 \E_{\Tl|A\ul} \big[s\ul\big]=0, & \quad 
 v\umin \leq \Var_{\Tl|A\ul}\big[s\ul\big]\leq v\umax,\\
& \textnormal{Under $A\pp\ul$: } & \;\; \log \Bigg( \frac{\mu_2(\dx\pl)}{\mu_2(\dx\po)}  \Bigg)  & = - l m\pp\ul + \sqrt{l} s\pp\ul, \quad &
m\umin \leq m\pp\ul \leq m\umax,& \quad
\E_{\Tl|A\pp\ul}\big[s\pp\ul\big]=0,  & \quad
v\umin \leq \Var_{\Tl|A\pp\ul}\big[s\pp\ul\big]\leq v\umax.  \\[-15pt] \nonumber
\end{alignat*}

\subsubsection{Proof Introduction}

 Using the definitions and notations from \Section{S:complementary_definitions_notations}, denoting $(\ve_1,\dots,\ve_{R\ul})$ and $(\lambda_1,\dots,\lambda_{R\ul})$ respectively the orthogonal eigenvectors and eigenvalues of $\mG_{\rvx,\valpha}[\rf(\rvx\plm,\valpha)]$ and denoting $\hat{\mW}\pl\equiv\mW\pl (\ve_1,\dots,\ve_{R\ul})$, we get that $\forall \rc$: \\[-6pt]
 \begin{align*}
\nu_{2,\rc}(\rvy\pl) & = \E_{\rvx,\valpha}\big[ (\rvy\pl_{\valpha,\rc})^2 \big]
  = \E_{\rvx,\valpha} \Big[ \big(\mW\pl_{\rc,:} \rf(\rvx\plm,\valpha) \big)^2 \Big] \\
 & = \sum_i \big( \hat{\mW}\pl_{\rc,i} \big)^2 \lambda_i
  = R\ul \, \nu_2(\rvx\plm) \sum_i \big( \hat{\mW}\pl_{\rc,i} \big)^2 \hat{\lambda}_i, \\[-15pt] \nonumber
\end{align*}
where we defined $\hat{\lambda}_i \equiv \lambda_i \, / \, \sum_j \lambda_j$ and used $\sum_j \lambda_j=\Tr \mG_{\rvx,\valpha}\big[\rf(\rvx\plm,\valpha)\big]=R\ul \nu_2(\rvx\plm)$ by Corollary \ref{C:statistics_preserving_constant}. 

Let us further define  \\[-9pt]
\begin{align*}
u\pl_\rc \equiv \left\{
    \begin{array}{cl}
        \largefrac{\nu_{2,\rc}( \rvx\pl )}{\nu_{2,\rc}( \rvx\pl ) +  \nu_{2,\rc}( \bx\pl )} & \mbox{if } \nu_{2,\rc}(\rvx\pl)+ \nu_{2,\rc}(\bx\pl)>0\vspace{2pt} \\
        \frac{1}{2} & \mbox{otherwise.} 
    \end{array}
\right.   \\[-15pt] \nonumber
\end{align*}

Combined with $\nu_{2,\rc}( \rvy\pl ) = \nu_{2,\rc}( \rvx\pl ) +  \nu_{2,\rc}( \bx\pl )$ by \eq{E:symmetric_nux}, we get that $\forall \rc$, under $\big\{ \nu_2(\rvx\plm) \neq 0 \big\}$:  \\[-6pt]
\begin{align}
 \nu_{2,\rc}(\rvx\pl)+\nu_{2,\rc}(\bx\pl) & = R\ul \, \nu_2(\rvx\plm) \,  \sum_i \big(\hat{\mW}\pl_{\rc,i} \big)^2 \hat{\lambda}_i, \label{E:moments_vanilla_Theorem4}  \\
 \nu_{2,\rc}(\rvx\pl) & =  u\pl_\rc \, R\ul \, \nu_2(\rvx\plm) \sum_i \big( \hat{\mW}\pl_{\rc,i} \big)^2 \hat{\lambda}_i.
  \label{E:moments_vanilla_Theorem2} \\[-15pt] \nonumber
\end{align}

Now combining \eq{E:moments_vanilla_Theorem4} with the symmetry of the propagation: $\nu_{2,\rc}(\rvx\pl) \sim_{\tl} \nu_{2,\rc}(\bx\pl)$, and the assumption of standard initialization: $\mW\pl_{\rc,:} \sim_\tl \hat{\mW}\pl_{\rc,:} \sim_\tl \mathcal{N}(0,2\,/\,R\ul \mI)$, we get that $\forall \rc$, under $\big\{ \nu_2(\rvx\plm) \neq 0 \big\}$:  \\[-6pt]
\begin{align*}
 2 \E_\tl \big[ \nu_{2,\rc}(\rvx\pl) \big] & = \E_\tl \big[ \nu_{2,\rc}(\rvx\pl)+\nu_{2,\rc}(\bx\pl) \big]
 = \E_\tl \Big[ R\ul \nu_2(\rvx\plm) \sum_i \big( \hat{\mW}\pl_{\rc,i} \big)^2 \hat{\lambda}_i \Big] \\
& = R\ul \nu_2(\rvx\plm) \frac{2}{R\ul}  \sum_i \hat{\lambda}_i 
 = 2 \nu_2(\rvx\plm). \\[-15pt] \nonumber
\end{align*}

Thus $\forall \rc: \; \E_\tl[ \nu_{2,\rc}(\rvx\pl) ] = \nu_2(\rvx\plm)$ and $\E_\tl[ \nu_2(\rvx\pl) ] = \nu_2(\rvx\plm)$, i.e. that under $\big\{ \nu_2(\rvx\plm) \neq 0 \big\}$:  \\[-9pt]
\begin{align}
\E_\tl[ \delta \nu_2(\rvx\pl) ] = 1. \label{E:moments_vanilla_Theorem5}  \\[-15pt] \nonumber
\end{align}

Next, let us define  \\[-9pt]
\begin{align*}
v\pl_\rc \equiv \left\{
    \begin{array}{cl}
        0 & \mbox{if }  u\pl_\rc < \frac{1}{2} \vspace{3pt} \\
        1 & \mbox{if }  u\pl_\rc > \frac{1}{2} \;, \vspace{3pt} \\
        \tilde{v}\pl_\rc & \mbox{if }  u\pl_\rc = \frac{1}{2} \\
    \end{array}
\right. \\[-15pt] \nonumber
\end{align*}

with $\tilde{v}\pl_\rc \sim \operatorname{Bernouilli}(1/2)$ independent of $\vomega\pl$ and $\vbeta\pl$. 

Conditionally on $u\pl_\rc = 1/2$: $v\pl_\rc~\sim~\operatorname{Bernouilli}(1/2)$, independently of $\nu_{2,\rc}(\rvy\pl)$ and $||\mW\pl_{\rc,:}||_2$. And conditionally on $u\pl_\rc~\neq~1/2$: $v\pl_\rc~\sim~\operatorname{Bernouilli}(1/2)$, independently of $\nu_{2,\rc}(\rvy\pl)$ and $||\mW\pl_{\rc,:}||_2$. It follows that $v\pl_\rc~\sim~\operatorname{Bernouilli}(1/2)$, independently of $\nu_{2,\rc}(\rvy\pl)$ and $||\mW\pl_{\rc,:}||_2$. \\[-6pt]

Defining $B\ul\equiv\{ \exists \rc: v\pl_\rc = 1 \}$ we get that $\sP_\tl [ B\ul] = 1-2^{-N\ul}$. We also get that $B\ul$ is independent of $\big(||\mW\pl_{\rc,:}||_2\big)_{1\leq\rc\leq N\ul}$ and thus of $||\mW\pl||_F$. \emph{This will be useful later in the course of this proof}.

Denoting $A\ul=\bigcap^l_{k=1}\Big(B\uk \cap \big\{ \nu_2(\rvx\pk) \neq 0 \big\} \Big)$, we have that \\[-6pt]
\begin{align*}
\sP_{\tl| A\ulm} \big[ A\ul \big] & = \sP_{\tl| A\ulm} \big[ \vss B\ul \cap \{ \nu_2(\rvx\pl) \neq 0 \} \vss \big]  = \sP_{\tl| A\ulm} \big[ B\ul   \big] = 1 - 2^{-N\ul}, \\
\sP_\Tl\big[A\ul\big] & = \sP_\Tl\Big[\bigcap\nolimits^l_{k=1} A\uk \Big] = \prod\nolimits^l_{k=1} \sP_{\tk| A\ukm} \big[A\uk\big] =  \prod\nolimits^l_{k=1} \big(1 - 2^{-N\uk} \big).\nonumber \\[-15pt] \nonumber
\end{align*}

where we used $\sP_{\tl| A\ulm} \big[ B\ul \cap \big\{ \nu_2(\rvx\pl) \neq 0 \big\} \big]=\sP_{\tl| A\ulm} \big[ B\ul \big]$ due to $\sP_{\tl | B\ul \cap A\ulm}\big[ \nu_2(\rvx\pl) \neq 0 \big] = 1$.

Now since $(\nu_{2,\rc}(\rvy\pl))_{1\leq\rc\leq N\ul}$ and $(v\pl_\rc)_{1\leq\rc\leq N\ul}$ are independent, \eq{E:moments_vanilla_Theorem2} implies that $\exists (w_{i})_{1\leq i \leq R\ul}$ such that under $B\ul \cap \big\{ \nu_2(\rvx\plm) \neq 0 \big\}$:   \\[-6pt]
\begin{alignat*}{3}
& (w_{i})_{1\leq i \leq R\ul} \sim \mathcal{N}(0,2\,/\,R\ul \mI), & \qquad \frac{1}{N\ul}  \Bigg(\frac{1}{2}\Bigg) R\ul \nu_2(\rvx\plm) \,  \sum^{R\ul}_{i=1} w_{i}^2 \hat{\lambda}_i & \leq  \frac{1}{N\ul} \sum^{N\ul}_{\rc=1} \nu_{2,\rc}(\rvx\pl), \\
& (w_{i})_{1\leq i \leq R\ul} \sim \mathcal{N}(0,2\,/\,R\ul \mI), & \qquad \frac{R\ul}{2N\ul} \sum^{R\ul}_{i=1} w_{i}^2 \hat{\lambda}_i & \leq  \delta \nu_2(\rvx\pl). \\[-15pt] \nonumber
\end{alignat*}

On the other hand, $\exists (w_{i,j})_{1\leq i \leq R\ul, \, 1 \leq j \leq N\ul}$ such that under $B\ul \cap \big\{ \nu_2(\rvx\plm) \neq 0 \big\}$:  \\[-6pt]
\begin{align*}
(w_{i,j})_{1\leq i \leq R\ul, \, 1 \leq j \leq N\ul} \sim \mathcal{N}(0,2\,/\,R\ul \mI): \qquad  \delta \nu_2(\rvx\pl) \leq \frac{R\ul}{N\ul} \sum^{N\ul}_{j=1} \sum^{R\ul}_{i=1} w_{i,j}^2 \hat{\lambda}_i \leq \frac{R\ul}{N\ul} \sum^{N\ul}_{j=1} \sum^{R\ul}_{i=1} w_{i,j}^2. \\[-15pt] \nonumber
\end{align*}

Denoting $\operatorname{Chi-Squared}(1)$ and $\operatorname{Chi-Squared}(N\ul R\ul)$ the chi-squared distributions with $1$ and $N\ul R\ul$ degrees of freedom respectively, $\exists w\umin, w\umax$ such that under $B\ul \cap \big\{ \nu_2(\rvx\plm) \neq 0 \big\}$: \\[-6pt]
\begin{alignat}{5}
w\umin & \sim \frac{R\ul}{2 N\ul} \frac{2}{R\ul} \frac{1}{R\ul} \operatorname{Chi-Squared}(1), & \quad w\umax & \sim \frac{R\ul}{N\ul} \frac{2}{R\ul} \operatorname{Chi-Squared}(N\ul R\ul), & \quad w\umin \leq \delta \nu_2(\rvx\pl) \leq w\umax, \nonumber \\
w\umin & \sim \frac{1}{N\ul R\ul} \operatorname{Chi-Squared}(1),& \quad w\umax & \sim \frac{2}{N\ul} \operatorname{Chi-Squared}(N\ul R\ul),& \quad w\umin \leq \delta \nu_2(\rvx\pl) \leq w\umax, \label{E:moments_vanilla_Theorem6} \\[-15pt] \nonumber
\end{alignat}

where we used $\max_i \hat{\lambda}_i \geq  \frac{1}{R\ul}$.

Simply replacing $\rvx\pl$ by $\dx\pl$, $\rvy\pl$ by $\dy\pl$, $\mG_{\rvx,\valpha}$ by $\mC_{\rvx,\dx,\valpha}$, using \eq{E:symmetric_mus} instead of \eq{E:symmetric_nux} and the identity with $\mu_2(\dx\plm)$ instead of $\nu_2(\rvx\plm)$ in Corollary~\ref{C:statistics_preserving_constant}, we get that under $\big\{ \mu_2(\dx\plm) \neq 0 \big\}$:  \\[-9pt]
\begin{align}
\E_\tl[ \delta \mu_2(\dx\pl) ] = 1. \label{E:moments_vanilla_Theorem7} \\[-15pt] \nonumber
\end{align}

Furthermore $\exists B\pp\ul$, independent of $||\mW\pl||_F$, such that $\sP_\tl\big[B\pp\ul\big]=1-2^{-N\ul}$, and $\exists w\pp\umin, w\pp\umax$ such that under $B\pp\ul \cap \big\{ \mu_2(\dx\plm) \neq 0 \big\}$: \\[-6pt]
\begin{align}
w\pp\umin \sim \frac{1}{N\ul R\ul} \operatorname{Chi-Squared}(1), \quad w\pp\umax \sim \frac{2}{N\ul} \operatorname{Chi-Squared}(N\ul R\ul), \qquad w\pp\umin \leq \delta \mu_2(\dx\pl) \leq w\pp\umax.  \label{E:moments_vanilla_Theorem8}
\end{align}

Denoting $A\pp\ul=\bigcap^l_{k=1}\Big(B\pp\uk \cap \big\{ \mu_2(\dx\pk) \neq 0 \big\} \Big)$, we also have \\[-6pt]
\begin{align*}
\sP_\Tl\big[A\pp\ul\big] & = \prod\nolimits^l_{k=1} \big(1 - 2^{-N\uk} \big). \\[-15pt] \nonumber
\end{align*}

Both $\log x$ and $(\log x)^2$ are integrable at $0$ since $\int \log x \, dx = x \log x - x$ and $\int (\log x)^2 dx = x(\log x)^2 - 2x \log x +2x$. By \eq{E:moments_vanilla_Theorem6} and \eq{E:moments_vanilla_Theorem8}, it then follows that $\log \delta \nu_2(\rvx\pl)$ and $\log \delta \mu_2(\rvx\pl)$ have well-defined expectation and variance under $A\ul$ and $A\pp\ul$ respectively.

\emph{Now, crucially, let us note that the distributions of $\delta \nu_2(\rvx\pl)$ with respect to $\tl | A\ul$ and $\delta \mu_2(\rvx\pl)$ with respect to $\tl | A\pp\ul$ are fully determined by: (i) the input distributions $P_{\rvx}(\rvx)=P_{\rvx\po}(\rvx\po)$ and $P_{\dx}(\dx)=P_{\dx\po}(\dx\po)$; (ii) the model parameters $\Tlm$ up to layer $l-1$.}

We are thus interested in the following infima and suprima:\\[-9pt]
\begin{alignat}{3}
& \inf_{P_{\rvx\po}(\rvx\po),P_{\dx\po}(\dx\po),\Tlm} \E_{\tl | A\ul}[ -\log \delta \nu_2(\rvx\pl) ], &\qquad & \sup_{P_{\rvx\po}(\rvx\po),P_{\dx\po}(\dx\po),\Tlm} \E_{\tl | A\ul}[ -\log \delta \nu_2(\rvx\pl) ], \label{E:inf_sup_expectation1} \\
& \inf_{P_{\rvx\po}(\rvx\po),P_{\dx\po}(\dx\po),\Tlm} \Var_{\tl | A\ul}[ \log \delta \nu_2(\rvx\pl) ],& \qquad & \sup_{P_{\rvx\po}(\rvx\po),P_{\dx\po}(\dx\po),\Tlm} \Var_{\tl | A\ul}[ \log \delta \nu_2(\rvx\pl) ], \label{E:inf_sup_variance1} \\
& \inf_{P_{\rvx\po}(\rvx\po),P_{\dx\po}(\dx\po),\Tlm} \E_{\tl | A\pp\ul}[ -\log \delta \mu_2(\dx\pl) ], &\qquad & \sup_{P_{\rvx\po}(\rvx\po),P_{\dx\po}(\dx\po),\Tlm} \E_{\tl | A\pp\ul}[ -\log \delta \mu_2(\dx\pl) ], \label{E:inf_sup_expectation2} \\
& \inf_{P_{\rvx\po}(\rvx\po),P_{\dx\po}(\dx\po),\Tlm} \Var_{\tl | A\pp\ul}[ \log \delta \mu_2(\dx\pl) ],& \qquad & \sup_{P_{\rvx\po}(\rvx\po),P_{\dx\po}(\dx\po),\Tlm} \Var_{\tl | A\pp\ul}[ \log \delta \mu_2(\dx\pl) ]. \label{E:inf_sup_variance2} \\[-15pt] \nonumber
\end{alignat}

Our strategy is to consider: \\[-15pt]
\begin{enumerate}[label={--}]
\item Sequences of random variables $(\rvx\pok)\ukN$, $(\dx\pok)\ukN$ corresponding to deterministic distributions $P_{\rvx\pok}(\rvx\pok)$, $P_{\dx\pok}(\dx\pok)$;
\item Sequences of deterministic model parameters $(\Theta\plmk)\ukN$ up to layer $l-1$;
\item Sequences of random variables $(\rvx\plmk)\ukN$ and $(\dx\plmk)\ukN$ obtained by the simultaneous propagation of $(\rvx\pok,\dx\pok)$ with parameters $\Theta\plmk$ up to layer $l-1$;
\item Sequences of random variables $(\rvx\plk)\ukN$ and $(\dx\plk)\ukN$ obtained by the simultaneous propagation at layer $l$ of $(\rvx\plmk,\dx\plmk)$ with random parameters $(\vomega\plk,\vbeta\plk)$;
\item Sequences of geometric increments $(\delta \nu_2(\rvx\plk))\ukN$ and $(\delta \mu_2(\dx\plk))\ukN$, defined as $\delta \nu_2(\rvx\plk)\equiv\frac{\nu_2(\rvx\plk)}{\nu_2(\rvx\plmk)}$ and $\delta \mu_2(\dx\plk)\equiv\frac{\mu_2(\dx\plk)}{\mu_2(\dx\plmk)}$;
\item Sequences of events $(B\ulk)\ukN$, $(B\pp\ulk)\ukN$, $(A\ulk)\ukN$, $(A\pp\ulk)\ukN$ appropriately defined with respect to $\delta \nu_2(\rvx\plk)$ and $\delta \mu_2(\dx\plk)$.
\end{enumerate}

We will finally consider sequences such that $\E_{\tl | A\ulk}[-\log \delta \nu_2(\rvx\plk) ]$, $\Var_{\tl | A\ulk}[ \log \delta \nu_2(\rvx\plk) ]$, $\E_{\tl | A\pp\ulk}[-\log \delta \mu_2(\dx\plk) ]$, $\Var_{\tl | A\pp\ulk}[ \log \delta \mu_2(\dx\plk) ]$ converge to the infima and suprima of \eq{E:inf_sup_expectation1}, \eq{E:inf_sup_variance1}, \eq{E:inf_sup_expectation2}, \eq{E:inf_sup_variance2} as $k\to\infty$. 

\emph{We start by focusing on $\delta \nu_2(\rvx\pl)$ and the reasoning will be easily extended to $\delta \mu_2(\dx\pl)$.}

\subsubsection{Weakly Convergent Subsequence}

By \eq{E:moments_vanilla_Theorem6}, under $B\ulk \cap A\ulmk$: \\[-12pt]
\begin{align*}
 \delta \nu_2(\rvx\plk) \notin ]a,b] & \implies \big( a \geq w\umink \big) \lor \big( w\umaxk > b \big), \\[-15pt] \nonumber
\end{align*}

with $\wedge$ the logical \emph{and}, $\lor$ the logical \emph{or}, and with $w\umink$, $w\umaxk$ defined as in \eq{E:moments_vanilla_Theorem6} with respect to $\delta \nu_2(\rvx\plk)$. Then $\sP_{\tl | A\ulk }\big[ \delta \nu_2(\rvx\plk) \notin ]a,b] \big] = \sP_{\tl | B\ulk \cap A\ulmk}\big[ \delta \nu_2(\rvx\plk) \notin ]a,b] \big]$ can be bounded as \\[-9pt]
\begin{align*}
\sP_{\tl | A\ulk }\big[ \delta \nu_2(\rvx\plk) \notin ]a,b] \big]
 & \leq \sP_{w \sim \frac{1}{N\ul R\ul} \operatorname{Chi-Squared}(1)} \big[ w \leq a \big] + \sP_{w \sim \frac{2}{N\ul} \operatorname{Chi-Squared}(N\ul R\ul)} \big[ w > b \big]. \\[-15pt] \nonumber
\end{align*}

Thus $\forall \epsilon$, $\exists a_\epsilon, b_\epsilon$ such that \\[-12pt]
\begin{align*}
\forall k: \, \sP_{\tl | A\ulk }\big[ \delta \nu_2(\rvx\plk) \notin ]a_\epsilon,b_\epsilon] \big] & \leq \epsilon, \\
\inf\nolimits\uk \sP_{\tl | A\ulk }\big[ \delta \nu_2(\rvx\plk) \in ]a_\epsilon,b_\epsilon] \big] & \geq 1-\epsilon, \\[-15pt] \nonumber
\end{align*}

which means that the sequence $\big(\delta \nu_2(\rvx\plk) | A\ulk\big)\ukN$ of random variables is tight. By \lemma{L:prokhorov_weak_convergence}, it follows that there exists a sequence of strictly increasing indices $(i\uk)\ukN$ and a random variable $X$ such that $\big(\delta \nu_2(\rvx\plik) | A\ulik\big)\ukN$ converges weakly to $X$: $\delta \nu_2(\rvx\plik)| A\ulik \Rightarrow X$.

If $\E_{\tl | A\ulk}[ -\log \delta \nu_2(\rvx\plk) ]$, $\Var_{\tl | A\ulk}[ \log \delta \nu_2(\rvx\plk) ]$ have well-defined limits equal to the infima and suprima of \eq{E:inf_sup_expectation1} and \eq{E:inf_sup_variance1}, then $\E_{\tl | A\ulik}[ -\log \delta \nu_2(\rvx\plik) ]$, $\Var_{\tl | A\ulik}[ \log \delta \nu_2(\rvx\plik) ]$ have the same limits. For simplicity of notations and without loss of generality, $(\delta \nu_2(\rvx\plik))\ukN$ may thus be renamed as $(\delta \nu_2(\rvx\plk))\ukN$ such that $\delta \nu_2(\rvx\plk)| A\ulk \Rightarrow X$.

We have that for all continuity points $a>0$ of the function $x\mapsto \sP[X\leq x]$: \\[-9pt]
\begin{align}
\sP_{\tl | A\ulk }\big[ \delta \nu_2(\rvx\plk) \leq a \big] & 
\leq \sP_{w \sim \frac{1}{N\ul R\ul} \operatorname{Chi-Squared}(1)} \big[ w \leq a \big], \nonumber \\
\sP \big[ X \leq a \big] & 
\leq \sP_{w \sim \frac{1}{N\ul R\ul} \operatorname{Chi-Squared}(1)} \big[ w \leq a \big], \label{E:moments_vanilla_Theorem10} \\[-15pt] \nonumber
\end{align}

where we used the definition of weak convergence: $\sP_{\tl | A\ulk }\big[ \delta \nu_2(\rvx\plk) \leq a \big] \xrightarrow{k\rightarrow\infty} \sP[ X \leq a ]$. 

Now let us show that the set of discontinuity points of the cumulative distribution function (c.d.f.) $x\mapsto \sP[X\leq x]$ on $[0,1]$ has Borel measure equal to $0$. Since c.d.f. are always non-decreasing and right-continuous, the set of discontinuity points is the set of non-left-continuity points, i.e. $D=\big\{ x\in[0,1]: \lim_{x\pp \to x^-} \sP[X\leq x\pp] < \sP[X\leq x] \big\}$. Let us denote $D_p \equiv\big\{ x\in[0,1]: \sP[X\leq x]- \lim_{x\pp \to x^-} \sP[X\leq x\pp] \geq \frac{1}{p} \big\}$. Then the function $\1_{D_p}$ converges point-wise to $\1_{D}$, i.e. $\forall x \in [0,1]$: $\1_{D_p}(x)  \xrightarrow{p\rightarrow\infty}  \1_{D}(x)$, and the dominated convergence theorem gives\\[-9pt]
\begin{align*}
\int^1_0 \1_{D_p}(x) dx \xrightarrow{p\rightarrow\infty} \int^1_0 \1_{D}(x) dx. \\[-15pt] \nonumber
\end{align*}

On the other hand, since $x\mapsto \sP[X\leq x]$ is non-decreasing and $0 \leq \sP[X\leq x] \leq 1$, it follows that $D_p$ is comprised of at most $p$ points, implying that $\int^1_0 \1_{D_p}(x) dx=0$. We deduce that $\int^1_0 \1_{D}(x)=0$, i.e. that $D$ has Borel measure equal to $0$. 

It follows that we can find a sequence of continuity points $a_p>0$ of $x\mapsto \sP[X\leq x]$ such that $a_p \xrightarrow{p\rightarrow\infty} 0$. We then obtain $\sP \big[ X = 0 \big] \leq \sP \big[ X \leq a_p \big] \xrightarrow{p\rightarrow\infty} 0$ by \eq{E:moments_vanilla_Theorem10}, and thus $\sP \big[ X = 0 \big]= 0$. Without loss of generality, we may assume $X > 0$ surely (if this is not the case, simply replace $X$ by a constant arbitrary value $>0$ under the zero-probability event $\{ X=0 \}$). 

Now if we consider the function $h$ such that $h(x)=\log x$ if $x>0$, and $h(x)=0$ otherwise, then \lemma{L:function_weak_convergence} implies that $h(\delta \nu_2(\rvx\plk)) | A\ulk \Rightarrow h(X)$, i.e. $\log \delta \nu_2(\rvx\plk) | A\ulk \Rightarrow \log X$. If we consider $h(x)=x^2$, we further deduce that $\delta \nu_2(\rvx\plk)^2 | A\ulk \Rightarrow X^2$ and that $\big( \log \delta \nu_2(\rvx\plk) \big)^2 | A\ulk \Rightarrow (\log X)^2$. \\

\subsubsection{Uniform Integrability}

Since both $x \mapsto \1_{\{ x \geq M\}} x$ and $x \mapsto \1_{\{ x^2 \geq M\}} x^2$ are non-decreasing for $x >0$, \eq{E:moments_vanilla_Theorem6} implies that \\[-9pt]
\begin{align*}
\sup_k \E_{\tl | A\ulk} \big[ \1_{\{ \delta \nu_2(\rvx\plk) \geq M\}} \delta \nu_2(\rvx\plk) \big] & \leq  \E_{w \sim \frac{2}{N\ul} \operatorname{Chi-Squared}(N\ul R\ul)} \big[ \1_{\{ w \geq M\}} w \big] \xrightarrow{M\rightarrow\infty} 0, \\
\sup_k \E_{\tl | A\ulk} \big[ \1_{\{ \delta \nu_2(\rvx\plk)^2 \geq M\}} \delta \nu_2(\rvx\plk)^2 \big] & 
 \leq  \E_{w \sim \frac{2}{N\ul} \operatorname{Chi-Squared}(N\ul R\ul)} \big[ \1_{\{ w^2 \geq M\}} w^2 \big] \xrightarrow{M\rightarrow\infty} 0.
\end{align*}

Since $\delta \nu_2(\rvx\plk) \geq 0$, it follows that both $\big(\delta \nu_2(\rvx\plk) | A\ulk \big)\ukN$ and $\big(\delta \nu_2(\rvx\plk)^2 | A\ulk\big)\ukN$ are uniformly integrable, implying by \lemma{L:integrable_weak_convergence} that\\[-9pt]
\begin{align*}
\E_{\tl | A\ulk}\big[ \delta \nu_2(\rvx\plk) \big]   \xrightarrow{k\rightarrow\infty} \E \big[ X \big], \qquad
\E_{\tl | A\ulk}\big[ \delta \nu_2(\rvx\plk)^2 \big]   \xrightarrow{k\rightarrow\infty} \E \big[ X^2 \big].
\end{align*}

Again since $x \mapsto \1_{\{ x \geq M\}} x$ is non-decreasing for $x >0$, \eq{E:moments_vanilla_Theorem6} implies that under $B\ulk \cap \big\{ w\umink>0 \big\} \cap \big\{ w\umaxk>0 \big\}$:\\[-12pt]
\begin{align*}
 \log w\umink & \leq  \log \delta \nu_2(\rvx\plk) \leq  \log w\umaxk,  \\ 
  \big|  \log \delta \nu_2(\rvx\plk) \big|
 & \leq \max \Big( \big|\log w\umink \big|,  \big|\log w\umaxk \big| \Big),  \\
 \1_{\{ |\log \delta \nu_2(\rvx\plk)| \geq M \}} \big|  \log \delta \nu_2(\rvx\plk) \big|
& \leq \max \Big(  \1_{\{ |\log w\umink| \geq M \}} \big|\log w\umink \big|, \1_{\{ |\log w\umaxk| \geq M \}}  \big|\log w\umaxk \big| \Big),  \\
\1_{\{ |\log \delta \nu_2(\rvx\plk)| \geq M \}} \big|  \log \delta \nu_2(\rvx\plk) \big|
 & \leq \1_{\{ |\log w\umink| \geq M \}}  \big|\log w\umink \big| + \1_{\{ |\log w\umaxk| \geq M \}}  \big| \log w\umaxk \big|.  \\[-15pt] \nonumber
\end{align*}

Similarly, we have that under $B\ulk \cap \big\{ w\umink>0 \big\} \cap \big\{ w\umaxk>0 \big\}$:\\[-9pt]
\begin{align*}
 \1_{\{ (\log \delta \nu_2(\rvx\plk))^2 \geq M \}} \big(  \log \delta \nu_2(\rvx\plk) \big)^2
& \leq \max \Big(  \1_{\{ (\log w\umink)^2 \geq M \}} \big(\log w\umink \big)^2, \1_{\{ (\log w\umaxk)^2 \geq M \}}  \big(\log w\umaxk \big)^2 \Big),  \\
\1_{\{ (\log \delta \nu_2(\rvx\plk))^2 \geq M \}} \big(  \log \delta \nu_2(\rvx\plk) \big)^2
 & \leq \1_{\{ (\log w\umink)^2 \geq M \}}  \big(\log w\umink \big)^2 + \1_{\{ (\log w\umaxk)^2 \geq M \}}  \big( \log w\umaxk \big)^2.  \\[-15pt] \nonumber
\end{align*}

Using $\sP_{\tl}[w\umink = 0]=0$ and $\sP_{\tl}[w\umaxk = 0]=0$, and denoting $\operatorname{Chi-Squared}(1)^*$ and $\operatorname{Chi-Squared}(N\ul R\ul)^*$ the chi-squared distributions excluding zero values, we get that \\[-12pt]
\begin{align*}
& \E_{\tl | A\ulk}  \Big[ \1_{\{ |\log \delta \nu_2(\rvx\plk)| \geq M \}} \big|  \log \delta \nu_2(\rvx\plk) \big| \Big] \\
& \qquad \leq \E_{w \sim \frac{1}{N\ul R\ul} \operatorname{Chi-Squared}(1)^*} \Big[ \1_{\{ |\log w| \geq M\}} \big|\log w \big| \Big]
+ \E_{w \sim \frac{2}{N\ul} \operatorname{Chi-Squared}(N\ul R\ul)^*} \Big[ \1_{\{ |\log w| \geq M\}} \big|\log w \big| \Big],\\
& \E_{\tl | A\ulk}  \Big[ \1_{\{ (\log \delta \nu_2(\rvx\plk))^2 \geq M \}} \big(  \log \delta \nu_2(\rvx\plk) \big)^2 \Big] \\
& \qquad \leq \E_{w \sim \frac{1}{N\ul R\ul} \operatorname{Chi-Squared}(1)^*} \Big[ \1_{\{ (\log w)^2 \geq M\}} \big(\log w \big)^2 \Big]
+ \E_{w \sim \frac{2}{N\ul} \operatorname{Chi-Squared}(N\ul R\ul)^*} \Big[ \1_{\{ (\log w)^2 \geq M\}} \big(\log w \big)^2 \Big].
\end{align*}

It follows that \\[-12pt]
\begin{align*}
\sup_k \E_{\tl | A\ulk}  \Big[ \1_{\{ |\log \delta \nu_2(\rvx\plk)| \geq M \}} \big|  \log \delta \nu_2(\rvx\plk) \big| \Big]  & \xrightarrow{M\rightarrow\infty} 0,\\
\sup_k \E_{\tl | A\ulk}  \Big[ \1_{\{ (\log \delta \nu_2(\rvx\plk))^2 \geq M \}} \big(  \log \delta \nu_2(\rvx\plk) \big)^2 \Big] & \xrightarrow{
M\rightarrow\infty} 0. \\[-15pt] \nonumber
\end{align*}

Thus both $\big(\log \delta \nu_2(\rvx\plk) | A\ulk\big)\ukN$ and $\big(( \log \delta \nu_2(\rvx\plk))^2 | A\ulk\big)\ukN$ are uniformly integrable, and by \lemma{L:integrable_weak_convergence}: \\[-9pt]
\begin{align*}
\E_{\tl | A\ulk}\big[ \log \delta \nu_2(\rvx\plk) \big]   \xrightarrow{k\rightarrow\infty} \E \big[ \log X \big], \qquad
\E_{\tl | A\ulk}\big[ \big(  \log \delta \nu_2(\rvx\plk) \big)^2 \big]   \xrightarrow{k\rightarrow\infty} \E \big[ ( \log X )^2 \big]. \\[-15pt] \nonumber \\ \nonumber
\end{align*}

\subsubsection{Bounding Moments of $\delta \nu_2(\rvx\plk)$}

First let us bound $\Var_{\tl | A\ulmk}\big[ \delta \nu_2(\rvx\plk) \big]$ from above. For each channel, the variance is bounded as \\[-12pt]
\begin{align}
\Var_{\tl | A\ulmk} \Bigg[ \frac{\nu_{2,\rc}(\rvx\plk)}{\nu_{2}(\rvx\plmk)} \Bigg] & \leq \E_{\tl | A\ulmk} \Bigg[ \frac{\nu_{2,\rc}(\rvx\plk)^2}{\nu_{2}(\rvx\plmk)^2} \Bigg] \leq  \E_{\tl | A\ulmk}\Bigg[ \frac{\nu_{2,\rc}(\rvy\plk)^2}{\nu_{2}(\rvx\plmk)^2} \Bigg] \nonumber \\
& \leq R^2\ul \, \E_{\tl | A\ulmk} \Big[ \sum\nolimits_{i\neq i\pp} \big( \hat{\mW}\pl_{\rc,i} \big)^2 \big( \hat{\mW}\pl_{\rc,i\pp} \big)^2 \hat{\lambda}_i \hat{\lambda}_{i\pp} + \sum\nolimits_{i} \big( \hat{\mW}\pl_{\rc,i} \big)^4 \hat{\lambda}^2_i  \Big] \nonumber \\
& \leq R^2\ul \sum\nolimits_{i,i\pp} \Bigg(\frac{2}{R\ul}\Bigg) \Bigg(\frac{2}{R\ul}\Bigg) 3 \hat{\lambda}_i \hat{\lambda}_{i\pp} = 12. \nonumber
\end{align}

Since the different channels are independent, we get that \\[-12pt]
\begin{align*}
\Var_{\tl | A\ulmk} \big[ \delta \nu_2(\rvx\plk) \big] & = \Var_{\tl | A\ulmk} \Bigg[ \frac{1}{N\ul} \sum_\rc \frac{\nu_{2,\rc}(\rvx\plk)}{\nu_{2}(\rvx\plmk)} \Bigg]
 = \frac{1}{N^2\ul} \sum_\rc \Var_{\tl | A\ulmk} \Bigg[ \frac{\nu_{2,\rc}(\rvx\plk)}{\nu_{2}(\rvx\plmk)} \Bigg] \leq \frac{12}{N\ul}. \\[-15pt] \nonumber
\end{align*}

Next we bound $\big|\E_{\tl| A\ulk} \big[ \delta \nu_2(\rvx\plk) \big] - 1 \big|$. Using $\E_{\tl | A\ulmk} \big[ \delta \nu_2(\rvx\plk) \big]=1$ by \eq{E:moments_vanilla_Theorem5}:  \\[-9pt]
\begin{align}
\big|\E_{\tl| A\ulk} & \big[ \delta \nu_2(\rvx\plk) \big] - 1 \big| \nonumber \\
& = \big| \E_{\tl| A\ulk} \big[ \delta \nu_2(\rvx\plk) \big] - \E_{\tl | A\ulmk} \big[ \delta \nu_2(\rvx\plk) \big] \big| \nonumber \\
& = \Bigg| \Bigg(\frac{1}{\sP_{\tl | A\ulmk}[A\ulk]} - 1 \Bigg) \, \E_{\tl | A\ulmk} \big[\1_{A\ulk} \delta \nu_2(\rvx\plk) \big]
- \E_{\tl | A\ulmk} \big[ \1_{A^{\textnormal{c}}\ulk} \delta \nu_2(\rvx\plk) \big] \Bigg| \nonumber \\
& \leq \frac{\sP_{\tl | A\ulmk}[A^{\textnormal{c}}\ulk]}{\sP_{\tl | A\ulmk}[A\ulk]} \big| \E_{\tl | A\ulmk} \big[ \1_{A\ulk} \delta \nu_2(\rvx\plk) \big]  \big|
+ \sP_{\tl | A\ulmk} \big[A^{\textnormal{c}}\ulk\big]^\frac{1}{2} \E_{\tl | A\ulmk} \big[ \delta \nu_2(\rvx\plk)^2 \big]^\frac{1}{2} \label{E:moments_vanilla_Theorem13} \\
& \leq \Bigg( \frac{\sP_{\tl | A\ulmk}[A^{\textnormal{c}}\ulk]}{\sP_{\tl | A\ulmk}[A\ulk]} + \sP_{\tl | A\ulmk} \big[A^{\textnormal{c}}\ulk\big]^\frac{1}{2} \Bigg)\, \E_{\tl | A\ulmk} \big[ \delta \nu_2(\rvx\plk)^2 \big]^\frac{1}{2} \label{E:moments_vanilla_Theorem14} \\
& \leq \Bigg( \frac{2^{-N\ul}}{1-2^{-N\ul}}+2^{-\frac{N\ul}{2}} \Bigg) \Big(1+\Var_{\tl | A\ulmk} \big[ \delta \nu_2(\rvx\plk) \big] \Big)^\frac{1}{2} 
\leq \epsilon\ul \Bigg( 1 + \frac{12}{N\ul} \Bigg)^\frac{1}{2} \leq 2 \epsilon\ul, \label{E:moments_vanilla_Theorem15} \\[-18pt] \nonumber
\end{align}

where we applied Cauchy-Schwarz inequality in \eq{E:moments_vanilla_Theorem13} and \eq{E:moments_vanilla_Theorem14}, defined $\epsilon\ul\equiv\frac{2^{-N\ul}}{1-2^{-N\ul}}+2^{-\frac{N\ul}{2}}$ and used $\Big( 1 + \frac{12}{N\ul} \Big)^\frac{1}{2} \leq 2$ under the large width assumption. \\[-9pt]

We are then able to bound $\Var_{\tl | A\ulk}\big[ \delta \nu_2(\rvx\plk) \big]$ from above:  \\[-12pt]
\begin{align}
\Var_{\tl| A\ulk} \big[ \delta \nu_2(\rvx\plk) \big] &  = \E_{\tl| A\ulk} \big[ \delta \nu_2(\rvx\plk)^2 \big] - \E_{\tl| A\ulk} \big[ \delta \nu_2(\rvx\plk) \big]^2 \nonumber \\
 & \leq \frac{\E_{\tl| A\ulmk} \big[  \1_{A\ulk} \delta \nu_2(\rvx\plk)^2 \big]}{\sP_{\tl | A\ulmk}[A\ulk]} - 1 + 1 - \E_{\tl| A\ulk} \big[ \delta \nu_2(\rvx\plk) \big]^2 \nonumber \\
 & \leq \frac{\E_{\tl| A\ulmk} \big[  \delta \nu_2(\rvx\plk)^2 \big] - 1}{1-2^{-N\ul}} + \frac{1}{1-2^{-N\ul}} -1
+ \big|\E_{\tl| A\ulk} \big[ \delta \nu_2(\rvx\plk) \big] - 1 \big| \big| \E_{\tl| A\ulk} \big[ \delta \nu_2(\rvx\plk) \big] + 1 \big|  \nonumber  \\
 & \leq \frac{\Var_{\tl| A\ulmk} \big[  \delta \nu_2(\rvx\plk) \big]}{1-2^{-N\ul}} + \frac{ 2^{-N\ul} }{1-2^{-N\ul}} + 2 \epsilon\ul \Bigg(\frac{\E_{\tl| A\ulmk} \big[ \delta \nu_2(\rvx\plk) \big]}{1-2^{-N\ul}} + 1 \Bigg) \nonumber \\
 & \leq \frac{1}{1-2^{-N\ul}} \Bigg( \frac{12}{N\ul} + 2^{-N\ul} \Bigg) + 2 \epsilon\ul \Bigg(\frac{1}{1-2^{-N\ul}} + 1 \Bigg) \leq  \frac{24}{N\ul}, \label{E:moments_vanilla_Theorem12} \\[-15pt] \nonumber
\end{align}
where we used again the fact that the terms in $2^{-N\ul}$ are negligible with respect to $\frac{12}{N\ul}$ under the large width assumption.

Finally let us bound $\Var_{\tl | A\ulk}\big[ \delta \nu_2(\rvx\plk) \big]$ from below. In the remaining of this calculation, the conditionality on \mbox{$\big\{ || \mW\pl ||^2_F>0 \big\}$} is assumed but omitted for simplicity of notation. This conditionality has no effect on expectations and probabilities since $\big\{ || \mW\pl ||^2_F>0 \big\}$ has probability one. 

We first note that $\Big(\1_{B\ulk}$, $\frac{\delta \nu_2(\rvx\plk)}{|| \mW\pl ||^2_F}\Big)$ is fully determined by $\Big((\tilde{v}\pl_\rc)_{1\leq\rc\leq N\ul}, \frac{\mW\pl}{|| \mW\pl ||^2_F}\Big)$, which is itself independent from $|| \mW\pl ||^2_F$. It follows that $\Big(\1_{B\ulk},\frac{\delta \nu_2(\rvx\plk)}{|| \mW\pl ||^2_F}\Big)$ is independent from $|| \mW\pl ||^2_F$, and thus that \\[-9pt]
\begin{align*}
 & \Var_{\tl | A\ulmk \cap B\ulk } \big[ \delta \nu_2(\rvx\plk) \big] &  \nonumber \\
 & \quad = \frac{\E_{\tl | A\ulmk }\Big[ \1_{B\ulk} \frac{\delta \nu_2(\rvx\plk)^2}{|| \mW\pl ||^4_F} || \mW\pl ||^4_F \Big]}{\sP_{\tl | A\ulmk}[B\ulk]} - \frac{\E_{\tl | A\ulmk }\Big[ \1_{B\ulk} \frac{\delta \nu_2(\rvx\plk)}{|| \mW\pl ||^2_F} || \mW\pl ||^2_F \Big]^2}{\sP_{\tl | A\ulmk}[B\ulk]^2} \nonumber \\
& \quad = \frac{\E_{\tl | A\ulmk }\Big[ \1_{B\ulk} \frac{\delta \nu_2(\rvx\plk)^2}{|| \mW\pl ||^4_F} \Big]}{\sP_{\tl | A\ulmk }[B\ulk]}  \E_{\tl | A\ulmk }\big[ || \mW\pl ||^4_F\big] - \frac{\E_{\tl | A\ulmk }\Big[ \1_{B\ulk} \frac{\delta \nu_2(\rvx\plk)}{|| \mW\pl ||^2_F} \Big]^2}{\sP_{\tl | A\ulmk }[B\ulk]^2} \E_{\tl | A\ulmk}\big[|| \mW\pl ||^2_F \big]^2 \nonumber \\
& \quad = \E_{\tl | A\ulmk \cap B\ulk }\Bigg[ \frac{\delta \nu_2(\rvx\plk)^2}{|| \mW\pl ||^4_F} \Bigg] \E_{\tl | A\ulmk}\big[ || \mW\pl ||^4_F\big] -  \E_{\tl | A\ulmk \cap B\ulk}\Bigg[ \frac{\delta \nu_2(\rvx\plk)}{|| \mW\pl ||^2_F} \Bigg]^2 \E_{\tl | A\ulmk}\big[|| \mW\pl ||^2_F \big]^2 \nonumber \\
& \quad \geq \E_{\tl | A\ulmk \cap B\ulk}\Bigg[ \frac{\delta \nu_2(\rvx\plk)}{|| \mW\pl ||^2_F} \Bigg]^2 \Big( \E_{\tl | A\ulmk}\big[ || \mW\pl ||^4_F\big] -  \E_{\tl | A\ulmk}\big[|| \mW\pl ||^2_F \big]^2 \Big) \nonumber \\
& \quad \geq \E_{\tl | A\ulmk \cap B\ulk}\Bigg[ \frac{\delta \nu_2(\rvx\plk)}{|| \mW\pl ||^2_F} \Bigg]^2 \Var_{\tl | A\ulmk}\big[|| \mW\pl ||^2_F \big]. \nonumber \\[-15pt] \nonumber
\end{align*}

Due to $\sP_{\tl | A\ulmk \cap B\ulk}[A\ulk]= 1$ and $A\ulk \subseteq A\ulmk \cap B\ulk$, the conditionality on $A\ulmk \cap B\ulk$ can be replaced by the conditionality on $A\ulk$:  \\[-9pt]
\begin{align}
 \Var_{\tl | A\ulk } \big[ \delta \nu_2(\rvx\plk) \big] \geq \E_{\tl | A\ulk}\Bigg[ \frac{\delta \nu_2(\rvx\plk)}{|| \mW\pl ||^2_F} \Bigg]^2 \Var_{\tl | A\ulmk}\big[|| \mW\pl ||^2_F \big]. \label{E:moments_vanilla_Theorem16} \\[-15pt] \nonumber
\end{align}

It remains to bound the terms $\E_{\tl | A\ulk}\Big[ \frac{\delta \nu_2(\rvx\plk)}{|| \mW\pl ||^2_F} \Big]^2$ and $\Var_{\tl | A\ulmk}\big[|| \mW\pl ||^2_F \big]$. A computation similar to \eq{E:moments_vanilla_Theorem14} gives \\[-9pt]
\begin{align}
\Bigg|\E_{\tl | A\ulk}\Bigg[ \frac{\delta \nu_2(\rvx\plk)}{|| \mW\pl ||^2_F} \Bigg] - \E_{\tl | A\ulmk}\Bigg[ \frac{\delta \nu_2(\rvx\plk)}{|| \mW\pl ||^2_F} \Bigg] \Bigg|
& \leq \epsilon\ul \, \E_{\tl | A\ulmk} \Bigg[ \frac{\delta \nu_2(\rvx\plk)^2}{|| \mW\pl ||^4_F} \Bigg]^\frac{1}{2}. \label{E:moments_vanilla_Theorem17}
\end{align}

The term $\E_{\tl | A\ulmk} \Big[ \frac{\delta \nu_2(\rvx\plk)^2}{|| \mW\pl ||^4_F} \Big]^\frac{1}{2}$ of \eq{E:moments_vanilla_Theorem17} can be bounded using \eq{E:moments_vanilla_Theorem2}:\\[-6pt]
\begin{align*}
\frac{\delta \nu_2(\rvx\plk)}{|| \mW\pl ||^2_F}
 = \frac{R\ul}{N\ul} \sum_\rc u\pl_\rc \sum_i \big( \hat{\mW}\pl_{\rc,i} \big)^2 \hat{\lambda}_i \frac{1}{|| \mW\pl ||^2_F}
 & \leq \frac{R\ul}{N\ul} \sum_\rc \sum_i \big( \hat{\mW}\pl_{\rc,i} \big)^2 \frac{1}{|| \mW\pl ||^2_F} = \frac{R\ul}{N\ul}, \\
\E_{\tl | A\ulmk} \Bigg[ \frac{\delta \nu_2(\rvx\plk)^2}{|| \mW\pl ||^4_F} \Bigg]^\frac{1}{2}
 & \leq \frac{R\ul}{N\ul}. \\[-18pt] \nonumber
\end{align*}

As for the term $\E_{\tl | A\ulmk}\Big[ \frac{\delta \nu_2(\rvx\plk)}{|| \mW\pl ||^2_F} \Big]$ of \eq{E:moments_vanilla_Theorem17}, we get by independence of $|| \mW\pl ||_F$ and $\frac{\delta \nu_2(\rvx\plk)}{|| \mW\pl ||^2_F}$ that \\[-6pt]
\begin{align*}
 \E_{\tl | A\ulmk}\Bigg[ \frac{\delta \nu_2(\rvx\plk)}{|| \mW\pl ||^2_F} \Bigg] \E_{\tl | A\ulmk}\big[ || \mW\pl ||^2_F \big] & =\E_{\tl | A\ulmk}\big[ \delta \nu_2(\rvx\plk) \big] = 1, \\
  \E_{\tl | A\ulmk}\Bigg[ \frac{\delta \nu_2(\rvx\plk)}{|| \mW\pl ||^2_F} \Bigg] & = \frac{1}{\E_{\tl | A\ulmk}\big[ || \mW\pl ||^2_F \big]} = \frac{1}{\frac{2}{R\ul} N\ul R\ul} = \frac{1}{2 N\ul}.
\end{align*}

We have $\epsilon\ul=\frac{2^{-N\ul}}{1-2^{-N\ul}}+2^{-\frac{N\ul}{2}} \ll \frac{1}{2 N\ulm n^d} \leq \frac{1}{2R\ul}$ under the large width assumption. Then, by \eq{E:moments_vanilla_Theorem17}: \\[-9pt]
\begin{align}
\Bigg|\E_{\tl | A\ulk}\Bigg[ \frac{\delta \nu_2(\rvx\plk)}{|| \mW\pl ||^2_F} \Bigg] - \frac{1}{2N\ul} \Bigg|
 & \ll \frac{1}{2R\ul} \frac{R\ul}{N\ul} = \frac{1}{2N\ul}, \\
\E_{\tl | A\ulk}\Bigg[ \frac{\delta \nu_2(\rvx\plk)}{|| \mW\pl ||^2_F} \Bigg] & \geq \frac{1}{4N\ul}. \label{E:moments_vanilla_Theorem18}
\end{align}

The variance $\Var_{\tl | A\ulmk}\big[|| \mW\pl ||^2_F \big]$ is given by\\[-9pt]
\begin{align}
\Var_{\tl | A\ulmk}\big[|| \mW\pl ||^2_F \big] & = \Bigg(\frac{2}{R\ul}\Bigg)^2 \Bigg( \E_{\tl | A\ulmk}\Bigg[ \sum\nolimits_{(\rc,i),(\rc\pp,i\pp)} \Bigg(\frac{2}{R\ul}\Bigg)^{-1}\big(\mW\pl_{\rc,i}\big)^2 \Bigg(\frac{2}{R\ul}\Bigg)^{-1} \big(\mW\pl_{\rc\pp,i\pp}\big)^2 \Bigg] - N^2\ul R^2\ul \Bigg)  \nonumber \\
& =  \Bigg(\frac{2}{R\ul}\Bigg)^2 \Bigg( \Big(\sum\nolimits_{(\rc,i) \neq (\rc\pp,i\pp)} 1\Big) + \Big(\sum\nolimits_{(\rc,i)} 3\Big) - N^2\ul R^2\ul \Bigg) = \frac{8N\ul}{R\ul}. \label{E:moments_vanilla_Theorem19} \\[-15pt] \nonumber
\end{align}

Finally combining \eq{E:moments_vanilla_Theorem16}, \eq{E:moments_vanilla_Theorem18} and \eq{E:moments_vanilla_Theorem19}: \\[-12pt]
\begin{align}
\Var_{\tl | A\ulk}\big[ \delta \nu_2(\rvx\plk) \big] & \geq \Bigg(\frac{1}{4N\ul}\Bigg)^2  \frac{8N\ul}{R\ul} = \frac{1}{2 N\ul R\ul}. \label{E:moments_vanilla_Theorem20} \\ \nonumber
\end{align}

\subsubsection{Consequence for $m\umin$, $m\umax$, $v\umin$, $v\umax$}

Using \eq{E:moments_vanilla_Theorem15} and taking the limit $k\to\infty$: \\[-12pt]
\begin{align*}
 \big|\E_{\tl| A\ulk} \big[ \delta \nu_2(\rvx\plk) \big] - 1 \big| & \leq 2 \epsilon\ul, \\
 \big|\E \big[ X \big] - 1 \big| & \leq 2 \epsilon\ul.  \\[-15pt] \nonumber
\end{align*}

Similarly, using \eq{E:moments_vanilla_Theorem12} and \eq{E:moments_vanilla_Theorem20} and taking the limit $k\to\infty$:
\begin{align*}
  \frac{1}{2 N\ul R\ul} \leq \Var_{\tl | A\ulk}\big[ \delta \nu_2(\rvx\plk) \big] & = \E_{\tl | A\ulk}\big[ \delta \nu_2(\rvx\plk)^2 \big] - \E_{\tl | A\ulk}\big[ \delta \nu_2(\rvx\plk) \big]^2 \leq  \frac{24}{N\ul},  \\
 \frac{1}{2 N\ul R\ul} \leq \Var \big[ X \big] & = \E \big[ X^2 \big] - \E \big[ X \big]^2 \leq  \frac{24}{N\ul}.  \\[-15pt] \nonumber
\end{align*}

Thus $\big|\E \big[ X \big]-1\big|$ is exponentially small in $N\ul$, while the standard deviation of $X$ behaves as a power-law of $N\ul$: $\frac{1}{\sqrt{2 N\ul R\ul}} \leq \Var \big[ X \big]^\frac{1}{2} \leq \sqrt{\frac{24}{N\ul}}$. This means that $\big|\E \big[ X \big]-1\big|$ is much smaller than the effect of the log-concavity: \\[-6pt]
\begin{align*}
\big| \E \big[ X \big] - 1 \big| \ll \log \E \big[ X \big] - \E \big[ \log X \big] \leq \E \big[ X \big] - 1 - \E \big[ \log X \big] & \; \implies \;
  \big| \E \big[ X \big] - 1 \big| < \E \big[ X \big] - 1 - \E \big[ \log X \big] \\
   & \; \implies \; \big| \E \big[ X \big] - 1 \big| - \big( \E \big[ X \big] - 1 \big) <  - \E \big[ \log X \big] \\
 & \; \implies \; 0 <  \E \big[ -\log X \big]. \\[-15pt] \nonumber
\end{align*}

In addition, $X$ has small standard deviation around $\E \big[ X \big]$ since $\Var \big[ X \big]^\frac{1}{2} \ll 1$ under the large width assumption. This implies that \\[-9pt]
\begin{gather*}
0 < \lim_{k\to\infty} \E_{\tl| A\ulk} [ - \log \delta \nu_2(\rvx\plk) ] = \E \big[ - \log X \big] \ll 1,\\ 
0 < \lim_{k\to\infty} \Var_{\tl| A\ulk} [ \log \delta \nu_2(\rvx\plk) ] = \Var \big[ \log X \big] \ll 1.  \\[-15pt] \nonumber
\end{gather*}

Now if we alternately consider sequences $\big(\delta \nu_2(\rvx\plk) | A\ulk\big)\ukN$ corresponding to distributions $P_{\rvx\pok}(\rvx\pok)$, $P_{\dx\pok}(\dx\pok)$, and parameters $\Theta^{l-1,k}$ up to layer $l-1$, such that \\[-9pt]
\begin{align*}
 \lim_{k\to\infty} \E_{\tl| A\ulk} [ - \log \delta \nu_2(\rvx\plk) ]=\inf_{P_{\rvx\po}(\rvx\po),P_{\dx\po}(\dx\po),\Tlm} \E_{\tl | A\ul}[ -\log \delta \nu_2(\rvx\pl) ],  \\
\lim_{k\to\infty} \E_{\tl| A\ulk} [ - \log \delta \nu_2(\rvx\plk) ]=\sup_{P_{\rvx\po}(\rvx\po),P_{\dx\po}(\dx\po),\Tlm} \E_{\tl | A\ul}[ -\log \delta \nu_2(\rvx\pl) ], \\
\lim_{k\to\infty} \Var_{\tl| A\ulk} [ \log \delta \nu_2(\rvx\plk) ] 
=\inf_{P_{\rvx\po}(\rvx\po),P_{\dx\po}(\dx\po),\Tlm} \Var_{\tl | A\ul}[ \log \delta \nu_2(\rvx\pl) ],  \\
\lim_{k\to\infty} \Var_{\tl| A\ulk} [ \log \delta \nu_2(\rvx\plk) ] 
=\sup_{P_{\rvx\po}(\rvx\po),P_{\dx\po}(\dx\po),\Tlm} \Var_{\tl | A\ul}[ \log \delta \nu_2(\rvx\pl) ],\\[-15pt]
\end{align*}

then we obtain that \\[-12pt]
\begin{gather*}
0 < \inf_{P_{\rvx\po}(\rvx\po),P_{\dx\po}(\dx\po),\Tlm} \E_{\tl | A\ul}[ - \log \delta \nu_2(\rvx\pl) ] \ll 1, \\ 
0 < \sup_{P_{\rvx\po}(\rvx\po),P_{\dx\po}(\dx\po),\Tlm} \E_{\tl | A\ul}[ - \log \delta \nu_2(\rvx\pl) ] \ll 1, \\ 
0 < \inf_{P_{\rvx\po}(\rvx\po),P_{\dx\po}(\dx\po),\Tlm} \Var_{\tl | A\ul}[ \log \delta \nu_2(\rvx\pl) ] \ll 1, \\
0 < \sup_{P_{\rvx\po}(\rvx\po),P_{\dx\po}(\dx\po),\Tlm} \Var_{\tl | A\ul}[ \log \delta \nu_2(\rvx\pl) ] \ll 1. \\[-15pt] \nonumber
\end{gather*}

The final remaining dependency is the dependency in $N\ul$ and $R\ul$. Since $R\ul=K^d\ul N\ulm \leq n^d N\ulm$, and since $(N\ul)_{l\in\sN}$ is bounded, it follows that $(R\ul)_{l\in\sN}$ is also bounded. If we denote \\[-12pt]
\begin{gather*}
N\umin\equiv\min_l N\ul, \quad N\umax\equiv\max_l N\ul, \quad R\umin\equiv\min_l R\ul, \quad R\umax\equiv\max_l R\ul, \\
\gI_N\equiv\{N\umin, \dots,N\umax\}, \quad \gI_R\equiv\{R\umin, \dots,R\umax\},  \\[-15pt]
\end{gather*}

then we finally get \\[-9pt]
\begin{gather*}
0 < \min_{N\ul \in \gI_N, R\ul \in \gI_R} \, \inf_{P_{\rvx\po}(\rvx\po),P_{\dx\po}(\dx\po),\Tlm} \E_{\tl | A\ul}[ - \log \delta \nu_2(\rvx\pl) ] \ll 1, \\
0 < \max_{N\ul \in \gI_N, R\ul \in \gI_R} \, \sup_{P_{\rvx\po}(\rvx\po),P_{\dx\po}(\dx\po),\Tlm} \E_{\tl | A\ul}[ -\log \delta \nu_2(\rvx\pl) ] \ll 1, \\ 
0 < \min_{N\ul \in \gI_N, R\ul \in \gI_R} \, \inf_{P_{\rvx\po}(\rvx\po),P_{\dx\po}(\dx\po),\Tlm} \Var_{\tl | A\ul}[ \log \delta \nu_2(\rvx\pl) ] \ll 1, \\
0 < \max_{N\ul \in \gI_N, R\ul \in \gI_R} \, \sup_{P_{\rvx\po}(\rvx\po),P_{\dx\po}(\dx\po),\Tlm} \Var_{\tl | A\ul}[ \log \delta \nu_2(\rvx\pl) ] \ll 1. \\[-15pt] \nonumber
\end{gather*}

The whole reasoning can immediately be transposed to $\mu_2(\dx\pl)$ to get\\[-9pt]
\begin{gather*}
0 < \min_{N\ul \in \gI_N, R\ul \in \gI_R} \, \inf_{P_{\rvx\po}(\rvx\po),P_{\dx\po}(\dx\po),\Tlm} \E_{\tl | A\pp\ul}[ - \log \delta \mu_2(\dx\pl) ] \ll 1, \\
0 < \max_{N\ul \in \gI_N, R\ul \in \gI_R} \, \sup_{P_{\rvx\po}(\rvx\po),P_{\dx\po}(\dx\po),\Tlm} \E_{\tl | A\pp\ul}[ -\log \delta \mu_2(\dx\pl) ] \ll 1, \\ 
0 < \min_{N\ul \in \gI_N, R\ul \in \gI_R} \, \inf_{P_{\rvx\po}(\rvx\po),P_{\dx\po}(\dx\po),\Tlm} \Var_{\tl | A\pp\ul}[ \log \delta \mu_2(\dx\pl) ] \ll 1, \\
0 < \max_{N\ul \in \gI_N, R\ul \in \gI_R} \, \sup_{P_{\rvx\po}(\rvx\po),P_{\dx\po}(\dx\po),\Tlm} \Var_{\tl | A\pp\ul}[ \log \delta \mu_2(\dx\pl) ] \ll 1. \\[-15pt] \nonumber
\end{gather*}

It follows that there exists small positive constants $1\gg m\umin, m\umax, v\umin, v\umax >0$ such that $\forall l$:\\[-9pt]
\begin{alignat}{6}
m\umin & \leq \E_{\tl | A\ul}[ - \log \delta \nu_2(\rvx\pl) ] && \leq m\umax, &\qquad v\umin & \leq \Var_{\tl | A\ul}[ \log \delta \nu_2(\rvx\pl) ]&& \leq v\umax, \label{E:moments_vanilla_Theorem23} \\
m\umin & \leq \E_{\tl | A\pp\ul}[ - \log \delta \mu_2(\dx\pl) ] && \leq m\umax, & \qquad v\umin & \leq \Var_{\tl | A\pp\ul}[ \log \delta \mu_2(\dx\pl) ] &&\leq v\umax. \label{E:moments_vanilla_Theorem24} \\[-15pt] \nonumber
\end{alignat}

\subsubsection{Proof Conclusion}

Again we start by focusing on $\delta \nu_2(\rvx\pl)$ and the reasoning will easily be extended to $\delta \mu_2(\dx\pl)$. Let us define under $A\uk$:  \\[-6pt]
\begin{alignat*}{6}
X\uk & \equiv\log \delta \nu_2(\rvx\pk), & \qquad Y\uk & \equiv\E_{\tk | A\uk}[ \log \delta \nu_2(\rvx\pk)],& \qquad Z\uk\equiv &\log \delta \nu_2(\rvx\pk)-\E_{\tk | A\uk}[ \log \delta \nu_2(\rvx\pk)].  \\[-15pt]
\end{alignat*}

Using \eq{E:moments_vanilla_Theorem23}, we have that under $A\uk$: \\[-6pt]
\begin{alignat*}{5}
m\umin \leq & -Y\uk \leq  m\umax,& \qquad & v\umin \leq \Var_{\tk|A\uk}[Z\uk]  \leq v\umax, \\
-m\umax \leq &  \; Y\uk  \leq  -m\umin,&\qquad &v\umin\leq \Var_{\tk|A\uk}[Z\uk] \leq v\umax. \\[-15pt] \nonumber
\end{alignat*}

By Lemma \ref{L:moments_vanilla}, we deduce that there exist random variables $m\ul$, $s\ul$ such that under $A\ul$: \\[-6pt]
\begin{align*}
 \sum\pl_{k=1} \log \delta \nu_2(\rvx\pk) & = l m\ul + \sqrt{l} s\ul, \quad -m\umax \leq m\ul \leq -m\umin, \quad \E_{\Tl|A\ul}[s\ul]=0, \quad v\umin \leq \Var_{\Tl|A\ul}[s\ul] \leq v\umax, \\
\log \Bigg( \frac{\nu_2(\rvx\pl)}{\nu_2(\rvx\po)}  \Bigg) & = l m\ul + \sqrt{l} s\ul, \quad -m\umax \leq m\ul \leq -m\umin, \quad \E_{\Tl|A\ul}[s\ul]=0, \quad v\umin \leq \Var_{\Tl|A\ul}[s\ul] \leq v\umax. \\[-15pt] \nonumber
\end{align*}

Finally changing the variable $m\ul$ to $-m\ul$, we get that under $A\ul$: \\[-6pt]
\begin{align*}
\log \Bigg( \frac{\nu_2(\rvx\pl)}{\nu_2(\rvx\po)}  \Bigg) & = -l m\ul + \sqrt{l} s\ul, \quad m\umin \leq m\ul \leq m\umax, \quad \E_{\Tl|A\ul}[s\ul]=0, \quad v\umin \leq \Var_{\Tl|A\ul}[s\ul] \leq v\umax.\\[-15pt] \nonumber
\end{align*}

Applying the exact same reasoning to $\mu_2(\dx\pl)$, we deduce that there exist random variables $m\pp\ul$, $s\pp\ul$ such that under $A\pp\ul$: \\[-6pt]
\begin{align*}
\log \Bigg( \frac{\mu_2(\dx\pl)}{\mu_2(\dx\po)}  \Bigg)& = -l m\pp\ul + \sqrt{l} s\pp\ul, \quad m\umin \leq m\pp\ul \leq m\umax, \quad \E_{\Tl|A\pp\ul}[s\pp\ul]=0, \quad v\umin \leq \Var_{\Tl|A\pp\ul}[s\pp\ul] \leq v\umax.\\[-15pt] \nonumber
\end{align*}

\subsubsection{Illustration}
Let us give an illustration in the fully-connected case with constant width, $N\ul=N=100$ and $R\ul=N=100$. The bounds $m\umin$, $m\umax$, $v\umin$, $v\umax$ are obtained by considering the extreme cases for $u\pl_{\rc}$ and $R\ul \sum\nolimits_i (\hat{\mW}\pl_{\rc,i})^2 \hat{\lambda}_i$ in \eq{E:moments_vanilla_Theorem2}: \\[-15pt]
\begin{enumerate}[label={--}]
\item We obtain \emph{minimum bounds} by considering $u\pl_\rc \sim 1/2$ and $R\ul \sum\nolimits_i (\hat{\mW}\pl_{\rc,i})^2 \hat{\lambda}_i \sim 2 \operatorname{Chi-Squared}(N) /N$, leading to $\delta \nu_2(\rvx\pl)$, $\delta \mu_2(\dx\pl) \sim \operatorname{Chi-Squared}(N^2) /N^2$;
\item We obtain \emph{maximum bounds} by considering $u\pl_\rc \sim \operatorname{Bernouilli}(1/2)$ and $R\ul \sum\nolimits_i (\hat{\mW}\pl_{\rc,i})^2 \hat{\lambda}_i \sim 2 \operatorname{Chi-Squared}(1)$.
\end{enumerate}

We numerically find $m\umin\simeq 9.7 {\scriptstyle \times} 10^{-5}$ and $v\umin\simeq 2.0 {\scriptstyle \times} 10^{-4}$ as minimum bounds and $m\umax\simeq 2.5 {\scriptstyle \times} 10^{-2}$ and $v\umax\simeq 5.2 {\scriptstyle \times} 10^{-2}$ as maximum bounds. \\[-6pt]

\subsection{The Conditionality on $A\ul$ is Highly Negligible}
 \label{S:negligible_conditionality}

The events $A\ul$, $A\pp\ul$ defined in \theorem{T:moments_vanilla} have probabilities equal to $\prod^l_{k=1} \big( 1 - 2^{-N_k } \big)$. Thus \\[-6pt]
\begin{align*}
-\sP_\Tl[A\pc\ul] &\simeq\log \big(\, 1-\sP_\Tl[A\pc\ul] \, \big)=\log \sP_\Tl[A\ul] = \sum\nolimits^l_{k=1} \log \big( 1 - 2^{-N_k} \big) \simeq  -\sum\nolimits^l_{k=1} 2^{-N_k},  \\[-15pt] \nonumber
\end{align*}

implying that $\sP_\Tl[A\pc\ul] =\sP_\Tl[A\ppc\ul]\simeq \sum^l_{k=1} 2^{-N_k}$. It follows that $\sP_\Tl[A\pc\ul]$, $\sP_\Tl[A\ppc\ul]$ grow linearly in the depth but decay exponentially in the width. 

In practice, $\sP_\Tl[A\pc\ul]$, $\sP_\Tl[A\ppc\ul]$ are thus highly negligible and the conditionality on $A\ul$, $A\pp\ul$ is also highly negligible. For example, in the case of constant width $N\ul=100$ and total depth $L=200$, we numerically find $\sP_{\Theta^L}[A\pc_L]=\sP_{\Theta^L}[A\ppc_L]\simeq 3.2 \times 10^{-28}$. \\[-9pt]

\subsection{Relation to the Terms $\overline{m}$, $\underline{m}$, $\underline{s}$ of \Section{S:model_parameters_randomness}}
 \label{S:moments_vanilla_terms_evolution}

Here we relate \theorem{T:moments_vanilla} to the terms $\overline{m}$, $\underline{m}$, $\underline{s}$ defined in \Section{S:model_parameters_randomness}, under the conditionality $A\uk$, $A\pp\uk$. By \eq{E:moments_vanilla_Theorem15}, we have that $|\E_{\tk |A\uk}[ \delta \nu_2(\rvx\pk) ] - 1 | \leq 2 \epsilon\uk \ll 1$. This implies that under $A\uk$: \\[-6pt]
\begin{align*}
| \overline{m}[\nu_2(\rvx\pk)] | & = \big| \log \E_{\tk |A\uk}[ \delta \nu_2(\rvx\pk) ] \big| \simeq \big|\E_{\tk |A\uk}[ \delta \nu_2(\rvx\pk) ] - 1 \big| \leq 2 \epsilon\uk. \\[-15pt]
\end{align*}

Similarly, we have that $|\E_{\tk |A\pp\uk}[ \delta \mu_2(\dx\pk) ] - 1 | \leq 2 \epsilon\uk \ll 1$, and that under $A\pp\uk$: \\[-6pt]
\begin{align*}
| \overline{m}[\mu_2(\dx\pk)] | & = \big| \log \E_{\tk |A\pp\uk}[ \delta \mu_2(\dx\pk) ] \big| \simeq \big|\E_{\tk |A\pp\uk}[ \delta \mu_2(\dx\pk) ] - 1 \big| \leq 2 \epsilon\uk.  \\[-15pt]
\end{align*}

The terms $\overline{m}[\nu_2(\rvx\pk)]$, $\overline{m}[\mu_2(\dx\pk)]$ are thus exponentially small in $N\uk$, implying that the evolution with depth of $\nu_2(\rvx\pl), \mu_2(\dx\pl)$ is dominated by the negative drift terms: \mbox{$\underline{m} [\nu_2(\rvx\pl) ]<0$, $\underline{m} [\mu_2(\dx\pl) ]<0$} and the diffusion terms: $\underline{s} [\nu_2(\rvx\pl) ]$, $\underline{s} [\mu_2(\dx\pl) ]$.  \\[-9pt]

\subsection{Proof of \theorem{T:S_increments_vanilla}}
 \label{S:proof_S_increments_vanilla}

\textbf{\theorem{T:S_increments_vanilla}} (normalized sensitivity increments of vanilla nets)\textbf{.}\emph{
Denoting $\rvy^{l,\pm}\equiv\max\big(\pm\rvy\pl,0\big)$, the dominating term under $\{ \mu_2(\rvx\plm) > 0\}$ in the evolution of $\S\pl$ is \\[-12pt]
\begin{gather*}
\delta \S\pl \simeq \exp(\overline{m}[\S\pl]) = \lpar 1 - \E_{\rc,\tl} \Bigg[ \frac{\nu_{1,\rc}  \big( \rvy\plplus \big) \nu_{1,\rc} \big( \rvy\plminus \big)}{\mu_{2}(\rvx\plm)} \Bigg]\rpar^{-\frac{1}{2}}.\\[-38pt]
\end{gather*}
\begin{align*}
\hspace{84pt}\underbrace{\qquad\qquad\qquad\qquad\qquad\qquad\qquad\quad\;}_{
\hspace{2pt} \in [1,\sqrt{2}]} \\[-21pt]
\end{align*}
}

\begin{proof}\textbf{.} 
The dominating term in the evolution of $\S\pl$ is given by  \\[-9pt]
\begin{align}
\delta \S\pl & \simeq \exp(\overline{m}[\S\pl]) = \Bigg( \frac{\E_\tl[ \delta \mu_2(\dx\pl)]}{\E_\tl[ \delta \mu_2(\rvx\pl)]} \Bigg)^\frac{1}{2}. \label{E:S_increments_vanilla_Theorem1} \\[-15pt] \nonumber
\end{align}

First we consider the term $\E_{\tl}[ \delta \mu_2(\rvx\pl)]$. Again we use the definitions and notations from \Section{S:complementary_definitions_notations}. We further denote $(\ve_1,\dots,\ve_{R\ul})$ and $(\lambda_1,\dots,\lambda_{R\ul})$ respectively the orthogonal eigenvectors and eigenvalues of $\mC_{\rvx,\valpha}[\rf(\rvx\plm,\valpha)]$ and $\hat{\mW}\pl\equiv\mW\pl (\ve_1,\dots,\ve_{R\ul})$. Using these notations, we get that $\forall \rc$: \\[-6pt]
 \begin{align}
\mu_{2,\rc}(\rvy\pl) & = \E_{\rvx,\valpha} \Big[ \fmc(\rvy\pl,\valpha)^2_\rc \Big]
= \E_{\rvx,\valpha} \Big[ \big(\mW\pl_{\rc,:} \rfc(\rvx\plm,\valpha) \big)^2 \Big] \nonumber \\
& = \sum\nolimits_i \big( \hat{\mW}\pl_{\rc,i} \big)^2 \lambda_i \label{E:S_increments_vanilla_Theorem2}. \\[-15pt] \nonumber
\end{align}

Then due to $\mW\pl_{\rc,:} \sim_\tl \hat{\mW}\pl_{\rc,:} \sim_\tl \mathcal{N}(0,2\,/\,R\ul \mI)$: \\[-6pt]
 \begin{align}
\E_{\tl}[ \mu_{2,\rc}(\rvy\pl) ] & = \frac{2}{R\ul} \sum\nolimits_i \lambda_i = \frac{2}{R\ul} \Tr \mC_{\rvx,\valpha}[\rf(\rvx\plm,\valpha)] = 2 \mu_2(\rvx\plm) \label{E:S_increments_vanilla_Theorem3}. \\[-15pt] \nonumber
\end{align}

where we used Corollary \ref{C:statistics_preserving_constant} in \eq{E:S_increments_vanilla_Theorem3}. The symmetric propagation gives \\[-6pt]
\begin{align}
\mu_{2,\rc}(\rvx\pl)+\mu_{2,\rc}(\bx\pl) & =\E_{\rvx,\valpha}\big[(\rvy\plplus_{\valpha,\rc})^2\big]-\E_{\rvx,\valpha}\big[\rvy\plplus_{\valpha,\rc}\big]^2+\E_{\rvx,\valpha}\big[(\rvy\plminus_{\valpha,\rc})^2\big]-\E_{\rvx,\valpha}\big[\rvy\plminus_{\valpha,\rc}\big]^2 \nonumber \\
&=\nu_{2,\rc}(\rvy\plplus)-\nu_{1,\rc}(\rvy\plplus)^2+\nu_{2,\rc}(\rvy\plminus)-\nu_{1,\rc}(\rvy\plminus)^2 \nonumber \\
&=\nu_{2,\rc}(\rvy\pl)-\Big(\nu_{1,\rc}(\rvy\plplus)^2+\nu_{1,\rc}(\rvy\plminus)^2\Big). \label{E:S_increments_vanilla_Theorem4} \\[-15pt] \nonumber
\end{align}

Since $\rvy\pl=\rvy\plplus-\rvy\plminus$ and $|\rvy\pl|=\rvy\plplus+\rvy\plminus$, we can express $\nu_{1,\rc}(\rvy\pl)$ and $\nu_{1,\rc}(|\rvy\pl|)$ as
\begin{align}
\nu_{1,\rc}(\rvy\pl)^2 & = \big(\nu_{1,\rc}(\rvy\plplus)- \nu_{1,\rc}(\rvy\plminus) \big)^2 
= \nu_{1,\rc}(\rvy\plplus)^2 + \nu_{1,\rc}(\rvy\plminus)^2 - 2 \nu_{1,\rc}(\rvy\plplus) \nu_{1,\rc}(\rvy\plminus), \label{E:S_increments_vanilla_Theorem5} \\
\nu_{1,\rc}(|\rvy\pl|)^2 & = \big(\nu_{1,\rc}(\rvy\plplus)+ \nu_{1,\rc}(\rvy\plminus) \big)^2 
= \nu_{1,\rc}(\rvy\plplus)^2 + \nu_{1,\rc}(\rvy\plminus)^2 + 2 \nu_{1,\rc}(\rvy\plplus) \nu_{1,\rc}(\rvy\plminus). \label{E:S_increments_vanilla_Theorem6} \\[-15pt] \nonumber
\end{align}

Using \eq{E:S_increments_vanilla_Theorem5}, we can then rewrite \eq{E:S_increments_vanilla_Theorem4} as\\[-9pt]
\begin{align}
\mu_{2,\rc}(\rvx\pl)+\mu_{2,\rc}(\bx\pl)
&=\nu_{2,\rc}(\rvy\pl)-\nu_{1,\rc}(\rvy\pl)^2-2\nu_{1,\rc}(\rvy\plplus) \nu_{1,\rc}(\rvy\plminus) \nonumber\\
&=\mu_{2,\rc}(\rvy\pl)-2\nu_{1,\rc}(\rvy\plplus) \nu_{1,\rc}(\rvy\plminus).  \label{E:S_increments_vanilla_Theorem7} \\[-15pt] \nonumber
\end{align}

Combining \eq{E:S_increments_vanilla_Theorem3} and \eq{E:S_increments_vanilla_Theorem7}: \\[-6pt]
\begin{align}
\E_{\tl}[ \mu_{2,\rc}(\rvx\pl)+\mu_{2,\rc}(\bx\pl) ]
&=2\mu_2(\rvx\plm)-2\E_\tl[\nu_{1,\rc}(\rvy\plplus) \nu_{1,\rc}(\rvy\plminus)], \nonumber \\
2 \E_{\tl}[ \mu_{2,\rc}(\rvx\pl) ] & = 2 \mu_2(\rvx\plm)-2 \E_\tl[\nu_{1,\rc}(\rvy\plplus) \nu_{1,\rc}(\rvy\plminus)], \label{E:S_increments_vanilla_Theorem8} \\
\E_{\tl}[ \mu_{2,\rc}(\rvx\pl) ] & = \mu_2(\rvx\plm)\Bigg(1-\E_\tl\Bigg[ \frac{\nu_{1,\rc}(\rvy\plplus) \nu_{1,\rc}(\rvy\plminus)}{\mu_2(\rvx\plm)} \Bigg]  \Bigg). \nonumber \\[-15pt] \nonumber
\end{align}

where \eq{E:S_increments_vanilla_Theorem8} was obtained by symmetry of the propagation. We then get  \\[-6pt]
\begin{align}
\E_{\tl}[ \mu_2(\rvx\pl) ] & = \E_\rc\big[ \E_\tl[ \mu_{2,\rc}(\rvx\pl) ] \big]  \nonumber \\
& = \mu_2(\rvx\plm)\Bigg(1-\E_{\rc,\tl}\Bigg[ \frac{\nu_{1,\rc}(\rvy\plplus) \nu_{1,\rc}(\rvy\plminus)}{\mu_2(\rvx\plm)} \Bigg]  \Bigg), \nonumber \\
\E_{\tl}[ \delta \mu_2(\rvx\pl) ] & = 1-\E_{\rc,\tl}\Bigg[ \frac{\nu_{1,\rc}(\rvy\plplus) \nu_{1,\rc}(\rvy\plminus)}{\mu_2(\rvx\plm)} \Bigg]. \nonumber \\[-15pt] \nonumber
\end{align}

Combining with \eq{E:S_increments_vanilla_Theorem1} and $\E_\tl[ \delta \mu_2(\dx\pl)] = 1$ by \eq{E:moments_vanilla_Theorem7} in the proof of \theorem{T:moments_vanilla}, we finally get  \\[-6pt]
\begin{align*}
\delta \S\pl & \simeq \exp(\overline{m}[\S\pl]) = \lpar 1 - \E_{\rc,\tl} \Bigg[ \frac{\nu_{1,\rc}  \big( \rvy\plplus  \big) \nu_{1,\rc} \big( \rvy\plminus \big)}{\mu_{2}(\rvx\plm)} \Bigg]\rpar^{-\frac{1}{2}}.
\end{align*}

To obtain the bounds on $\exp(\overline{m}[\S\pl])$, we use \eq{E:S_increments_vanilla_Theorem5} and \eq{E:S_increments_vanilla_Theorem6}: \\[-9pt]
\begin{align}
4 \nu_{1,\rc}(\rvy\plplus) \nu_{1,\rc}(\rvy\plminus) + \nu_{1,\rc}(\rvy\pl)^2  = \nu_{1,\rc}(|\rvy\pl|)^2 & \leq \nu_{2,\rc}(|\rvy\pl|) = \nu_{2,\rc}(\rvy\pl), \nonumber \\
4 \nu_{1,\rc}(\rvy\plplus) \nu_{1,\rc}(\rvy\plminus) & \leq \nu_{2,\rc}(\rvy\pl) - \nu_{1,\rc}(\rvy\pl)^2 = \mu_{2,\rc}(\rvy\pl). \label{E:S_increments_vanilla_Theorem9} \\[-15pt] \nonumber
\end{align}

Given $\E_{\tl}[ \mu_{2,\rc}(\rvy\pl) ] = 2 \mu_2(\rvx\plm)$ by \eq{E:S_increments_vanilla_Theorem3}, we deduce that $4 \E_{\tl}[ \nu_{1,\rc}(\rvy\plplus) \nu_{1,\rc}(\rvy\plminus) ] \leq 2 \mu_2(\rvx\plm)$, and thus that\\[-9pt]
\begin{align*}
\E_{\rc,\tl} \Bigg[ \frac{\nu_{1,\rc}  \big( \rvy\plplus \big) \nu_{1,\rc} \big( \rvy\plminus \big)}{\mu_{2}(\rvx\plm)} \Bigg] & \leq \frac{1}{2}, \\
 1 \leq \exp(\overline{m}[\S\pl]) = \Bigg( 1 - \E_{\rc,\tl} \Bigg[ \frac{\nu_{1,\rc}  \big( \rvy\plplus \big) \nu_{1,\rc} \big( \rvy\plminus \big)}{\mu_{2}(\rvx\plm)} \Bigg] \Bigg)^{-\frac{1}{2}} & \leq \sqrt{2}.\tag*{\qed} \\[-15pt]
\end{align*}
\end{proof}

\subsection{If the Drift of $\S\pl$ Is Larger than Diffusion and if $\nu_2(\rvx\pl)$, $\mu_2(\dx\pl)$ are Lognormal, then $\mu_2(\rvx\pl) \, / \, \nu_2(\rvx\pl)\to0$  a.s.}
 \label{S:ratio_cv_as}

\vspace{3pt}
\begin{Lemma} \label{L:Borel_Cantelli}
For a sequence of random variables $(X_l)_{l\in \sN}$ and a random variable $X$, if $\forall \epsilon>0: \sum\nolimits^\infty_{l=1} \sP[|X_l-X|>\epsilon]<\infty$, then 
\begin{align*}
X_l \xrightarrow{l\to\infty} X \textnormal{ a.s.}
\end{align*}
\end{Lemma}

\begin{proof}\textbf{.} For given $\epsilon>0$, denote $N_\epsilon$ the number of times that the event $\{|X_l-X|>\epsilon\}$ occurs such that $N_\epsilon=\sum\nolimits^\infty_{l=1} \1_{\{|X_l-X|>\epsilon\}}$. Fubini's theorem implies that $\E[N_\epsilon]=\sum\nolimits^\infty_{l=1} \sP[|X_l-X|>\epsilon]<\infty$, \mbox{implying that $N_\epsilon$ is finite a.s.}

Now let us reason by contradiction and suppose that $\exists E$ with $\sP[E]>0$ such that under $E$: $X_l \centernot{\xrightarrow{l\to\infty}} X$. Under $E$, $\exists \epsilon$ random variable, and $\exists(k\ul)_{l\in \sN}$ random strictly increasing sequence such that $\forall l$: $|X_{k\ul}-X|>\epsilon$. This implies in turn that $\exists E\pp$ with $\sP[E\pp]>0$ and $\exists \epsilon\pp>0$ non-random, such that under $E\pp$: $\exists(k\ul)_{l\in \sN}$ random strictly increasing sequence with $\forall l$: $|X_{k\ul}-X|>\epsilon\pp$. Thus $N_{\epsilon\pp}$ has non-zero probability to be infinite: $\sP[N_{\epsilon\pp}=\infty] \geq \sP[E\pp]>0$, which is a contradiction. We deduce that $X_l \xrightarrow{l\to\infty} X$ a.s. \qed
\end{proof} \\

\begin{Proposition}
\label{P:ratio_cv_as}
Suppose that:
\begin{enumerate}[leftmargin=*,itemsep=8pt,parsep=0pt,label={(\roman*)}]
\item We can neglect the events $A\ul$, $A\pp\ul$ of probability exponentially small in the width (see \Section{S:negligible_conditionality} for justification);
\item The event $D$ under which $\S\pl$ has drift larger than diffusion has probability $\sP[D]>0$; 
\item $\nu_2(\rvx\pl)$, $\mu_2(\dx\pl)$ are lognormal. \\[-9pt]
\end{enumerate}

Then, under $D$:  \\[-9pt]
\begin{align*}
\frac{\mu_2(\rvx\pl)}{\nu_2(\rvx\pl)}\xrightarrow{l\to\infty}0 \textnormal{\, a.s.} \\[-6pt]
\end{align*}
\end{Proposition}

\begin{proof}\textbf{.}
Neglecting the events $A\ul$, $A\pp\ul$, \theorem{T:moments_vanilla} implies that $\exists m\ul, m\pp\ul$, $s\ul,s\pp\ul$ such that \\[-6pt]
\begin{alignat*}{4}
 \log \nu_2(\rvx\pl)  & = - l m\ul + \sqrt{l} s\ul + \log \nu_2(\rvx\po), && \quad m\umin \leq m\ul \leq m\umax, && \quad \E_\Tl [s\ul]=0, && \quad v\umin \leq \Var_\Tl[s\ul] \leq v\umax, \\
 \log \mu_2(\dx\pl) & = - l m\pp\ul + \sqrt{l} s\pp\ul + \log \mu_2(\dx\po), && \quad m\umin \leq m\pp\ul \leq m\umax, && \quad \E_\Tl [s\pp\ul ]=0, && \quad v\umin \leq \Var_\Tl [s\pp\ul] \leq v\umax. \\[-15pt] \nonumber
\end{alignat*} 

On the other hand, under standard initialization: \\[-6pt]
\begin{align*}
\E_\Tl[\nu_2(\rvx\pl)] & = \E_\Tlm \E_\tl\big[\nu_2(\rvx\plm) \cdot \delta\nu_2(\rvx\pl)\big] = \E_\Tlm\big[\nu_2(\rvx\plm) \cdot \E_\tl[\delta\nu_2(\rvx\pl)] \big] = \E_\Tlm\big[\nu_2(\rvx\plm)\big], \\
\E_\Tl[\mu_2(\dx\pl)] & = \E_\Tlm \E_\tl\big[\mu_2(\dx\plm) \cdot \delta\mu_2(\dx\pl)\big] = \E_\Tlm\big[\mu_2(\dx\plm) \cdot \E_\tl[\delta\mu_2(\dx\pl)] \big] = \E_\Tlm\big[\mu_2(\dx\plm)\big], \\[-15pt] \nonumber
\end{align*}

implying by induction that $\E_\Tl[\nu_2(\rvx\pl)]=\nu_2(\rvx\po)$ and $\E_\Tl[\mu_2(\dx\pl)]=\mu_2(\dx\po)$.

Since $\log \nu_2(\rvx\pl)$, $\log \mu_2(\dx\pl)$ are Gaussian by the assumption of lognormality, and since a logormal variable $\exp(X)$ with $X\sim \mathcal{N}(\mu, \sigma^2)$ has expectation equal to $\E[\exp(X)]=\exp(\mu+\sigma^2/2)$, it follows that $\exists S_l, S\pp_l$ random variables and $\exists M_l, M\pp_l>0$ constants such that \\[-6pt]
\begin{alignat*}{4}
 \log \nu_2(\rvx\pl) & =S\ul- M\ul +\log\nu_2(\rvx\po), && \quad S\ul \sim_\Tl \mathcal{N}(0,2M\ul), && \quad lm\umin \leq M\ul \leq lm\umax, \\
 \log \mu_2(\dx\pl) & = S\pp\ul-M\pp\ul + \log \mu_2(\dx\po), && \quad S\pp\ul \sim_\Tl \mathcal{N}(0,2M\pp\ul), && \quad lm\umin \leq M\pp\ul \leq lm\umax. \\[-15pt] \nonumber
\end{alignat*} 

Now let us make more precise the conditionality on $D$. We may assume that $\exists m>\frac{1}{2}\big(m\umax-m\umin\big)$ such that $\forall l$ under $D$: $\log \S\pl\geq lm$.

The ratio $\mu_2(\rvx\pl)/\nu_2(\rvx\pl)$ can be expressed as \\[-6pt]
\begin{align*}
\frac{\mu_2(\rvx\pl)}{\nu_2(\rvx\pl)} & = \Bigg(  \frac{\mu_2(\dx\po)}{\mu_2(\rvx\po)} \frac{\mu_2(\rvx\pl)}{\mu_2(\dx\pl)}  \Bigg)
 \Bigg( \frac{\mu_2(\rvx\po)}{\mu_2(\dx\po)} \frac{\mu_2(\dx\pl)} {\nu_2(\rvx\pl)} \Bigg) = \frac{1}{(\S\pl)^2}  \frac{\mu_2(\rvx\po)} {\mu_2(\dx\po)} \frac{\mu_2(\dx\pl)} {\nu_2(\rvx\pl)}. \\[-15pt] \nonumber
\end{align*}

This gives with logarithms that, under $D$:  \\[-6pt]
\begin{align*}
\log \mu_2(\rvx\pl) & - \log \nu_2(\rvx\pl) = -2\log \S\pl +  \log \mu_2(\dx\pl) -\log \mu_2(\dx\po) -\log \nu_2(\rvx\pl) +\log \mu_2(\rvx\po) \\
 & \leq -2lm + \big(S\pp\ul-M\pp\ul \big) - \big(S\ul-M\ul+\log\nu_2(\rvx\po)\big) +\log \mu_2(\rvx\po) \\
 & \leq -2lm + lm\umax - lm\umin-\log\nu_2(\rvx\po) +\log \mu_2(\rvx\po)+S\pp\ul-S\ul \\
 & \leq -lM + C+S\pp\ul-S\ul,  \\[-15pt] \nonumber
\end{align*}

where we defined $M\equiv2m-m\umax+m\umin>0$ and $C\equiv-\log\nu_2(\rvx\po)+\log \mu_2(\rvx\po)$. Then for given $\epsilon$, under $D$:  \\[-6pt]
\begin{align*}
\frac{\mu_2(\rvx\pl)}{\nu_2(\rvx\pl)} > \epsilon & \implies \log \epsilon < -lM+C+S\pp\ul-S\ul \\
& \implies  \Big( S\pp\ul \geq \frac{\log \epsilon + lM - C}{2} \Big) \lor \Big( -S\ul \geq \frac{\log \epsilon + lM - C}{2} \Big)\\
& \implies  \Big( \tilde{S}\pp\ul \geq \frac{\log \epsilon + lM - C}{2\sqrt{2 M\pp\ul}} \Big) \lor \Big( -\tilde{S}\ul \geq \frac{\log \epsilon + lM - C}{2\sqrt{2 M\ul}} \Big) \\
& \implies \Big( \tilde{S}\pp\ul \geq \frac{\log \epsilon + lM - C}{2\sqrt{2lm\umax}} \Big) \lor \Big( -\tilde{S}\ul \geq \frac{\log \epsilon + lM - C}{2\sqrt{2lm\umax}} \Big), \\[-15pt] \nonumber
\end{align*}

where we denoted $\lor$ the logical \emph{or}, $\tilde{S}\ul\equiv S\ul/\sqrt{2M\ul}$ and $\tilde{S}\pp\ul\equiv S\pp\ul/\sqrt{2M\pp\ul}$, and supposed $l$ large enough such that $\log \epsilon + lM - C\geq 0$. Then $\exists C_\epsilon > 0$ such that for $l$ large enough, under $D$:  \\[-6pt]
\begin{align}
\frac{\mu_2(\rvx\pl)}{\nu_2(\rvx\pl)} > \epsilon 
& \implies \Big( \tilde{S}\pp\ul \geq \sqrt{l} C_\epsilon \Big) \lor \Big( -\tilde{S}\ul \geq \sqrt{l} C_\epsilon \Big). \nonumber \\[-15pt] \nonumber
\end{align}

It follows that for $l$ large enough: \\[-6pt]
\begin{align}
\sP_{\Tl|D} \Bigg[\frac{\mu_2(\rvx\pl)}{\nu_2(\rvx\pl)} > \epsilon\Bigg] & \leq \sP_{\Tl|D}\Big[\tilde{S}\pp\ul \geq \sqrt{l} C_\epsilon \Big]+ \sP_{\Tl|D}\Big[-\tilde{S}\ul \geq \sqrt{l} C_\epsilon \Big]  \nonumber  \\
& \leq \frac{1}{\sP_\Tl[D]} \sP_\Tl\Big[D \cap \big\{\tilde{S}\pp\ul \geq \sqrt{l} C_\epsilon \big\} \Big] + \frac{1}{\sP_\Tl[D]} \sP_\Tl\Big[D \cap \big\{-\tilde{S}\ul \geq \sqrt{l} C_\epsilon \big\} \Big] \nonumber  \\
& \leq \frac{1}{\sP_\Tl[D]} \erfc\Bigg( \sqrt{\frac{l}{2}} C_\epsilon \Bigg) \label{E:ratio_cv_proba1} \\
& \leq \frac{1}{\sP_\Tl[D]} \exp\Big( -\frac{l}{2} C_\epsilon^2 \Big), \label{E:ratio_cv_proba2}
\end{align}

where \eq{E:ratio_cv_proba1} is obtained using $\tilde{S}\ul,\tilde{S}\pp\ul \sim_\Tl \mathcal{N}(0,1)$, while \eq{E:ratio_cv_proba2} is obtained using $\erfc(x)\leq \exp(-x^2)$ \citep{Chiani03}. It follows from \eq{E:ratio_cv_proba2} that  \\[-6pt]
\begin{align*}
\sum^\infty_{l=1} \sP_{D} \Bigg[\frac{\mu_2(\rvx\pl)}{\nu_2(\rvx\pl)} > \epsilon\Bigg]=\sum^\infty_{l=1} \sP_{\Tl|D} \Bigg[\frac{\mu_2(\rvx\pl)}{\nu_2(\rvx\pl)} > \epsilon\Bigg]<\infty. \\[-15pt] \nonumber
\end{align*}

By Lemma \ref{L:Borel_Cantelli}, we finally deduce that, under $D$:  \\[-6pt]
\begin{align*}
\frac{\mu_2(\rvx\pl)}{\nu_2(\rvx\pl)} \xrightarrow{l\to\infty} 0 \textnormal{ a.s.} \tag*{\qed}
\end{align*}
\end{proof}

\subsection{If $\exp\big( \overline{m}[\S\pl]\big) \to 1$ and if Moments of $\tx\pl$ Are Bounded, then $\rvx\pl$ Converges to One-Dimensional Signal Pathology}
 \label{S:one_dimensional_pathology}

\begin{Proposition}
\label{P:one_dimensional_pathology}
Again we adopt the notation: $\tx\pl\equiv\rvx\pl/\sqrt{\mu_2(\rvx\pl)}$, and the usual notation: \\[-6pt] 
\begin{align*}
X_l=\bigO(Y_l) \iff \exists M >0, \forall l: \; X_l\leq M Y_l.  \\[-15pt] \nonumber
\end{align*}

We further suppose that:
\begin{enumerate}[leftmargin=*,itemsep=8pt,parsep=0pt,label={(\roman*)}]
\item $\tx\pl$ is well-defined with bounded moments: $\nu_p(|\tx\pl|)=\bigO(1)$, implying in particular $\nu_2(\rvx\pl)/\mu_2(\rvx\pl)\centernot{\xrightarrow{l\to\infty}} \infty$ and thus $\mu_2(\rvx\pl)/\nu_2(\rvx\pl)\centernot{\xrightarrow{l\to\infty}} 0\,$, i.e. that $\rvx\pl$ does not converge to zero-dimensional signal pathology;
\item $\delta \S\pl \simeq \exp\big( \overline{m}[\S\pl]\big) \xrightarrow{l\to\infty} 1$. 
\end{enumerate}

Then $\rvx\pl$ converges to one-dimensional signal pathology. \\
\end{Proposition}

\begin{proof}\textbf{.}
Again we use the notations from \Section{S:complementary_definitions_notations} and we denote:  \\[-6pt]
\begin{align*}
\vnu\pl_\varphi & \equiv \E_{\rvx,\valpha}\big[ \fm(\tx\pl,\valpha) \big]=\big(\nu_{1,\rc}(\tx\pl)\big)_{1\leq\rc\leq N\ul}, \\
\vnu\pl_\rho & \equiv \E_{\rvx,\valpha}\big[ \rf(\tx\pl,\valpha) \big].  \\[-15pt] \nonumber
\end{align*}

The statistic-preserving property implies $\frac{1}{N\ul} ||\vnu\pl_\varphi||^2_2=\frac{1}{R\ul} ||\vnu\pl_\rho||^2_2$, in turn implying that \\[-6pt]
\begin{align*}
\nu_2(\tx\pl) & =\frac{1}{N\ul} \Big( \sum_\rc \mu_{2,\rc}(\tx\pl)+\nu_{1,\rc}(\tx\pl)^2 \Big) \\
& = \mu_2(\tx\pl)+\frac{1}{N\ul} ||\vnu_\varphi\pl||_2^2 \\
& =1+\frac{1}{N\ul} ||\vnu_\varphi\pl||_2^2 = 1+\frac{1}{R\ul} ||\vnu_\rho\pl||_2^2, \\[-15pt] \nonumber 
\end{align*}

i.e. that $||\vnu_\rho\pl||_2^2=R\ul \big(\nu_2(\tx\pl)-1\big)$. Combined with $\nu_2(\tx\pl)=\bigO(1)$, we deduce that $||\vnu_\rho\pl||_2=\bigO(1)$.

\emph{Now let us reason by contradiction and suppose that $\reff(\rvx\pl)=\reff(\tx\pl)\centernot{\xrightarrow{l\to\infty}}1$, implying that $\exists \eta>0$ and $\exists (k\ul)_{l\in \sN}$ strictly increasing sequence with $\forall l$: $\reff(\tx\pkl)\geq 1+\eta$.}

This directly implies that $\exists \eta\pp>0$ such that $\forall l$:  \\[-6pt]
\begin{align*}
\exists \rvv_\varphi\pkl\in \R^{N\ukl} \perp \vnu_\varphi\pkl,\quad ||\rvv_\varphi\pkl||_2=1: \quad \Var_{\rvx,\valpha}\big[ \langle \fm(\tx\pkl,\valpha), \rvv_\varphi\pkl \rangle  \big]=\E_{\rvx,\valpha}\big[ \langle \fm(\tx\pkl,\valpha), \rvv_\varphi\pkl \rangle^2  \big] \geq \eta\pp,  \\[-15pt] \nonumber
\end{align*}

i.e. that $\fm(\tx\pkl,\valpha)$ has a direction of variance $> \eta\pp$ which is orthogonal to its mean vector $\vnu_\varphi\pkl$. By padding this direction appropriately with zeros, it follows that $\exists \eta\pp>0$ such that $\forall l$:  \\[-6pt]
\begin{align*}
\exists \rvv_\rho\pkl\in \R^{R\ukl} \perp \vnu_\rho\pkl, \quad ||\rvv_\rho\pkl||_2=1: \quad \Var_{\rvx,\valpha}\big[ \langle \rf(\tx\pkl,\valpha), \rvv_\rho\pkl \rangle  \big]=\E_{\rvx,\valpha}\big[ \langle \rf(\tx\pkl,\valpha), \rvv_\rho\pkl \rangle^2  \big] \geq \eta\pp. \\[-15pt] \nonumber 
\end{align*}

Let us denote $\tilde{\mW}\pklp$ such that $\forall \rc:$ $\tilde{\mW}\pklp_{\rc,:}\equiv\mW\pklp_{\rc,:}/||\mW\pklp_{\rc,:}||_2$ and $\tnu_\rho\pkl\equiv\vnu_\rho\pkl/||\vnu_\rho\pkl||_2$. Let us further decompose $\tilde{\mW}\pklp_{\rc,:}$ as  \\[-9pt]
\begin{align*}
\tilde{\mW}\pklp_{\rc,:} = w_\rvv \big(\rvv_\rho\pkl\big)^T + \sqrt{1-w^2_\rvv} \rvw^T, \qquad \rvw \perp \rvv_\rho\pkl, \qquad ||\rvw|| = 1. \\[-15pt] \nonumber
\end{align*}

Then we get  \\[-9pt]
\begin{align*}
\E_{\rvx,\valpha}& \Big[ \big(\tilde{\mW}\pklp_{\rc,:} \rf(\tx\pkl,\valpha) \big)^2 \Big] \\
 & = \E_{\rvx,\valpha}\Big[  w^2_\rvv  \langle \rf(\tx\pkl,\valpha), \rvv_\rho\pkl \rangle^2 + (1-w^2_\rvv) \langle \rf(\tx\pkl,\valpha), \rvw \rangle^2  + 2 w_\rvv \sqrt{1-w^2_\rvv} \langle \rf(\tx\pkl,\valpha), \rvv_\rho\pkl \rangle   \langle \rf(\tx\pkl,\valpha), \rvw \rangle  \Big] \\
 & \geq w^2_\rvv \eta\pp + 2 w_\rvv \sqrt{1-w^2_\rvv} \E_{\rvx,\valpha}\Big[ \langle \rf(\tx\pkl,\valpha), \rvv_\rho\pkl \rangle   \langle \rf(\tx\pkl,\valpha), \rvw \rangle  \Big] \\
  & \geq w^2_\rvv \eta\pp - 2 w_\rvv \sqrt{1-w^2_\rvv} \E_{\rvx,\valpha}\Big[ \langle \rf(\tx\pkl,\valpha), \rvv_\rho\pkl \rangle^2 \Big]^\frac{1}{2} \E_{\rvx,\valpha}\Big[ \langle \rf(\tx\pkl,\valpha), \rvw \rangle^2 \Big]^\frac{1}{2} \\
    & \geq w^2_\rvv \eta\pp - 2 w_\rvv \sqrt{1-w^2_\rvv} \E_{\rvx,\valpha}\Bigg[ \Big\langle \rf(\tx\pkl,\valpha), \frac{\rf(\tx\pkl,\valpha)}{||\rf(\tx\pkl,\valpha)||} \Big\rangle^2 \Bigg]^\frac{1}{2} \E_{\rvx,\valpha}\Bigg[ \Big\langle \rf(\tx\pkl,\valpha), \frac{\rf(\tx\pkl,\valpha)}{||\rf(\tx\pkl,\valpha)||} \Big\rangle^2 \Bigg]^\frac{1}{2} \\
   & \geq w^2_\rvv \eta\pp - 2 w_\rvv \sqrt{1-w^2_\rvv} \E_{\rvx,\valpha}\Big[ \sum\nolimits_i \rf(\tx\pkl,\valpha)^2_i \Big] \\
     & \geq w^2_\rvv \eta\pp - 2 w_\rvv \sqrt{1-w^2_\rvv} R\ukl \nu_2(\tx\pkl), \\
\E_{\rvx,\valpha} & \Big[ \tilde{\mW}\pklp_{\rc,:} \rf(\tx\pkl,\valpha) \Big]^2  \\
& = (1-w^2_\rvv) \E_{\rvx,\valpha} \Big[ \langle \rf(\tx\pkl,\valpha), \rvw \rangle   \Big]^2   \\
& \leq (1-w^2_\rvv) \E_{\rvx,\valpha}\Bigg[ \Big\langle \rf(\tx\pkl,\valpha), \frac{\rf(\tx\pkl,\valpha)}{||\rf(\tx\pkl,\valpha)||} \Big\rangle^2 \Bigg] \\
& \leq (1-w^2_\rvv) R\ukl \nu_2(\tx\pkl). \\[-15pt] \nonumber
\end{align*}

Given that $\nu_2(\tx\pkl)=\bigO(1)$, this implies by spherical symmetry that $\forall \epsilon>0$, $\exists p_\epsilon>0$ such that $\forall l$:  \\[-9pt]
\begin{align}
\sP_{\tklp}\Bigg[ \Big(\E_{\rvx,\valpha}\Big[ \big( \tilde{\mW}\pklp_{\rc,:} \rf(\tx\pkl,\valpha)\big)^2 \Big]^2 \geq \eta^{\prime \vss 2} -\epsilon \Big) 
\wedge \Big( \E_{\rvx,\valpha}\Big[ \tilde{\mW}\pklp_{\rc,:} \rf(\tx\pkl,\valpha) \Big]^2  \leq \epsilon \Big) \Bigg]
\geq p_\epsilon, \label{E:one_dimensional_pathology1} \\[-15pt] \nonumber
\end{align}

with $\wedge$ the logical \emph{and}. 

On the other hand, by Cauchy-Schwarz inequality:   \\[-6pt]
\begin{align}
\E_{\rvx,\valpha}\Big[ \big( \tilde{\mW}\pklp_{\rc,:} \rf(\tx\pkl,\valpha) \big)^2 \Big]^2 & \leq 
\E_{\rvx,\valpha}\Big[ \big| \tilde{\mW}\pklp_{\rc,:} \rf(\tx\pkl,\valpha) \big| \Big]
\E_{\rvx,\valpha}\Big[ \big| \tilde{\mW}\pklp_{\rc,:} \rf(\tx\pkl,\valpha) \big|^3 \Big]. \label{E:one_dimensional_pathology2} \\[-15pt] \nonumber
\end{align}

The second term on the right-hand side can be bounded as  \\[-6pt]
\begin{align}
\E_{\rvx,\valpha}\Big[ & \big| \tilde{\mW}\pklp_{\rc,:} \rf(\tx\pkl,\valpha) \big|^3 \Big] \nonumber \\
 & \leq \E_{\rvx,\valpha}\Bigg[ \Big\langle \rf(\tx\pkl,\valpha), \frac{\rf(\tx\pkl,\valpha)}{||\rf(\tx\pkl,\valpha)||_2} \Big\rangle^3 \Bigg] 
 = \E_{\rvx,\valpha}\Big[ ||\rf(\tx\pkl,\valpha)||^3_2 \Big] 
 = \E_{\rvx,\valpha}\Bigg[ \Big(\sum^{R\ukl}_{i=1} \rf(\tx\pkl,\valpha)^2_{i} \Big)^{3/2} \Bigg] \nonumber \\
& \leq \E_{\rvx,\valpha}\Bigg[ \sum_{i_1,i_2,i_3} \rf(\tx\pkl,\valpha)^2_{i_1} \rf(\tx\pkl,\valpha)^2_{i_2} \rf(\tx\pkl,\valpha)^2_{i_3} \Bigg]^{1/2} \label{E:one_dimensional_pathology4} \\
& \leq \sum_{i_1,i_2,i_3}  \E_{\rvx,\valpha}\big[ \rf(\tx\pkl,\valpha)^4_{i_1} \big]^{1/4} \E_{\rvx,\valpha}\big[ \rf(\tx\pkl,\valpha)^8_{i_2} \big]^{1/8} \E_{\rvx,\valpha}\big[ \rf(\tx\pkl,\valpha)^8_{i_3} \big]^{1/8} \label{E:one_dimensional_pathology5} \\
& \leq R^3\ukl N^{1/2}\ukl \nu_4(\tx\pkl)^{1/4} \nu_8(\tx\pkl)^{1/4}, \label{E:one_dimensional_pathology6}  \\[-15pt] \nonumber
\end{align}

where \eq{E:one_dimensional_pathology4} and \eq{E:one_dimensional_pathology5} were obtained by applying Cauchy-Schwarz inequality, while \eq{E:one_dimensional_pathology6} was obtained with $\displaystyle \forall i,\forall p$: $\displaystyle \E_{\rvx,\valpha}\big[ \rf(\tx\pkl,\valpha)^p_i \big] \leq \sum\nolimits_\rc \E_{\rvx,\valpha}\big[ \fm(\tx\pkl,\valpha)^p_\rc \big]=N\ukl \nu_p(\tx\pkl)$. 

It then follows from \eq{E:one_dimensional_pathology2} and the hypothesis that all moments are bounded $\nu_p(|\tx\pl|)=\bigO(1)$ that   \\[-6pt]
\begin{align}
\E_{\rvx,\valpha}\Big[ \big( \tilde{\mW}\pklp_{\rc,:} \rf(\tx\pkl,\valpha) \big)^2 \Big]^2 = \bigO\Big(  \E_{\rvx,\valpha}\Big[ \big| \tilde{\mW}\pklp_{\rc,:} \rf(\tx\pkl,\valpha) \big| \Big] \Big). \label{E:one_dimensional_pathology7}  \\[-15pt] \nonumber
\end{align}

Combining \eq{E:one_dimensional_pathology1} and \eq{E:one_dimensional_pathology7}, we deduce that $\exists \eta\ppp>0$ with $\forall \epsilon>0$, $\exists p\pp_\epsilon >0$ such that $\forall l$:   \\[-6pt]
\begin{align*}
\sP_{\tklp}\Bigg[ \Big(\E_{\rvx,\valpha}\Big[ \big|  \tilde{\mW}\pklp_{\rc,:} \rf(\tx\pkl,\valpha) \big| \Big] \geq \eta\ppp -\epsilon \Big) 
\wedge \Big( \E_{\rvx,\valpha}\Big[  \tilde{\mW}\pklp_{\rc,:} \rf(\tx\pkl,\valpha) \Big]^2  \leq \epsilon \Big) \Bigg]
\geq p\pp_\epsilon.  \\[-18pt] \nonumber
\end{align*}

Under standard initialization: $\mW\pklp_{\rc,:} \sim_\tklp \mathcal{N}(0,2\,/\,R\ukl \mI)$, the variables $\tilde{\mW}\pklp_{\rc,:}$ and $||\mW\pklp_{\rc,:}||_2$ are independent and $\sP_{\tklp}\big[1 \leq ||\mW\pklp_{\rc,:}||_2 \leq 2 \big]>0$ does not depend on $l$. Therefore $\forall \epsilon>0$, $\exists p\ppp_\epsilon >0$ such that $\forall l$:   \\[-9pt]
\begin{align}
\sP_{\tklp}\Bigg[ \Big( \E_{\rvx,\valpha}\Big[ \big| \mW\pklp_{\rc,:} \rf(\tx\pkl,\valpha) \big|  \Big] \geq \eta\ppp-\epsilon \Big) 
\wedge \Big( \E_{\rvx,\valpha}\Big[ \mW\pklp_{\rc,:} \rf(\tx\pkl,\valpha) \Big]^2  \leq 4 \epsilon \Big) \Bigg]
\geq p\ppp_\epsilon. \\[-18pt] \nonumber
\end{align}

Now by noting that \\[-9pt]
\begin{align*}
\E_{\rvx,\valpha}\Big[ \big| \mW\pklp_{\rc,:} \rf(\tx\pkl,\valpha) \big| \Big] & =
 \E_{\rvx,\valpha}\Big[ \big(\mW\pklp_{\rc,:} \rf(\tx\pkl,\valpha)\big)^+ \Big] +
\E_{\rvx,\valpha}\Big[ \big(\mW\pklp_{\rc,:} \rf(\tx\pkl,\valpha)\big)^- \Big],\\
\E_{\rvx,\valpha}\Big[ \mW\pklp_{\rc,:} \rf(\tx\pkl,\valpha) \Big]^2 & = 
\Big( \E_{\rvx,\valpha}\Big[ \big(\mW\pklp_{\rc,:} \rf(\tx\pkl,\valpha)\big)^+ \Big] -
\E_{\rvx,\valpha}\Big[ \big(\mW\pklp_{\rc,:} \rf(\tx\pkl,\valpha)\big)^- \Big] \Big)^2, \\[-18pt] \nonumber
\end{align*}

we deduce that $\exists \eta\pppp>0$, $\exists p>0$ such that $\forall l$:   \\[-9pt]
\begin{align*}
\sP_{\tklp}\Bigg[ \Big( \E_{\rvx,\valpha}\Big[ \big(\mW\pklp_{\rc,:} \rf(\tx\pkl,\valpha)\big)^+ \Big] \geq \eta\pppp \Big)
\wedge \Big( \E_{\rvx,\valpha}\Big[ \big(\mW\pklp_{\rc,:} \rf(\tx\pkl,\valpha)\big)^- \Big] \geq \eta\pppp \Big) \Bigg]
& \geq p,  \\
\sP_{\tklp}\Bigg[ \Big( \nu_{1,\rc}(\rvy\pklpplus) \geq \eta\pppp \sqrt{\mu_2(\rvx\pkl)} \Big)
\wedge \Big( \nu_{1,\rc}(\rvy\pklpminus) \geq \eta\pppp \sqrt{\mu_2(\rvx\pkl)}  \Big) \Bigg]
& \geq p, \\
\sP_{\tklp}\Bigg[ \frac{\nu_{1,\rc}(\rvy\pklpplus) \nu_{1,\rc}(\rvy\pklpminus)}{\mu_2(\rvx\pkl)} \geq (\eta\pppp)^2  \Bigg]
& \geq p, \\
\E_{\rc,\tklp}\Bigg[ \frac{\nu_{1,\rc}(\rvy\pklpplus) \nu_{1,\rc}(\rvy\pklpminus)}{\mu_2(\rvx\pkl)} \Bigg]
& \geq p (\eta\pppp)^2. \\[-18pt] \nonumber
\end{align*}

Thus by \theorem{T:S_increments_vanilla}, $\exists \eta\ppppp>0$ such that $\forall l$: $\exp\big(\overline{m}[\S\pklp]\big) \geq 1+\eta\ppppp$, contradicting the hypothesis $\exp\big(\overline{m}[\S\pl]\big)\xrightarrow{l\to\infty}1$. 

\emph{We deduce that $\reff(\rvx\pl)\xrightarrow{l\to\infty}1$, i.e. that $\rvx\pl$ converges to one-dimensional signal pathology}. \qed
\end{proof}

\subsection{If $\exp\big( \overline{m}[\S\pl]\big) \to 1$, then each Additional Layer $l$ Becomes Arbitrarily Well Approximated by a Linear Mapping}
 \label{S:pseudo_linearity}

We suppose that $\forall l$: $\mu_2(\rvx\pl)>0$ and that $\exp\big( \overline{m}[\S\pl]\big)\to 1$. Denoting $\ty\pl=\rvy\pl/\sqrt{\mu_2(\rvx\plm)}$ and $\ty^{l,\pm}\equiv\max\big(\pm \ty\pl,0\big)$, \theorem{T:S_increments_vanilla} implies that  \\[-12pt]
\begin{align}
\E_{\rc,\tl} \big[ \nu_{1,\rc}(\ty\plplus)\nu_{1,\rc}(\ty\plminus) \big] \to 0, \nonumber \\ 
\E_{\rc,\tl} \Big[ \min\big(\nu_{1,\rc}(\ty\plplus),\nu_{1,\rc}(\ty\plminus) \big)^2 \Big] \to 0, \nonumber \\
\forall \epsilon>0: \quad \sP_{\rc,\tl} \big[ \min\big(\nu_{1,\rc}(\ty\plplus),\nu_{1,\rc}(\ty\plminus) \big) > \epsilon \big] \to 0, \nonumber \\
\forall \epsilon>0: \quad \sP_{\tl} \big[ \exists \rc:\; \min\big(\nu_{1,\rc}(\ty\plplus),\nu_{1,\rc}(\ty\plminus) \big) > \epsilon \big] \to 0, \nonumber\\
\forall \epsilon>0: \quad \sP_{\tl} \big[ \forall \rc:\; \min\big(\nu_{1,\rc}(\ty\plplus),\nu_{1,\rc}(\ty\plminus) \big) \leq \epsilon \big] \to 1. \label{E:pseudo_linearity1} \\[-15pt] \nonumber
\end{align}

Now let us fix a channel $\rc$ and suppose that $\min\big(\nu_{1,\rc}(\ty\plplus)\nu_{1,\rc}(\ty\plminus) \big) \leq \epsilon$. Given that $\ty\plminus=|\ty\plplus-\ty\pl|$, we have that \\[-9pt]
\begin{align*}
\min\big(\nu_{1,\rc}(\ty\plplus)\nu_{1,\rc}(\ty\plminus) \big) = \min\big( \nu_{1,\rc}(|\ty\plplus-0|), \nu_{1,\rc}(|\ty\plplus-\ty\pl|) \big)  \leq \epsilon. \\[-15pt] 
\end{align*}

Both $\nu_{1,\rc}(|\ty\plplus-0|)$ and $\nu_{1,\rc}(|\ty\plplus-\ty\pl|)$ correspond to the mean absolute error incurred when approximating the rescaled signal $\rvx\pl/\mu_2(\rvx\plm)=\rvy\plplus/\mu_2(\rvx\plm)=\ty\plplus$ in channel $\rc$ by a linear function. So there exists a linear function $f_\rc: \R^{n\times \dots \times n \times N\ulm} \rightarrow \R^{n\times \dots \times n}$ such that\\[-9pt] 
\begin{align*}
 \E_{\rvx,\valpha}\big[ |\ty\plplus_{\valpha,\rc}-f_\rc(\rvx\plm)_\valpha | \big]  \leq \epsilon. \\[-15pt] 
\end{align*}

If $\forall \rc$: $\min\big(\nu_{1,\rc}(\ty\plplus)\nu_{1,\rc}(\ty\plminus) \big) \leq \epsilon$, and if we define the linear function $f: \R^{n\times \dots \times n \times N\ulm} \rightarrow \R^{n\times \dots \times n \times N\ul}$ such that $\forall \valpha,\rc$: $f(\rvx\plm)_{\valpha,\rc} = f_\rc(\rvx\plm)_\valpha$, then we get  \\[-9pt]
\begin{align*}
 \nu_1(|\ty\plplus-f(\rvx\plm)|) = \E_{\rvx,\valpha,\rc}\big[ |\ty\plplus_{\valpha,\rc}-f(\rvx\plm)_{\valpha,\rc} | \big]
 = \E_\rc \E_{\rvx,\valpha}\big[ |\ty\plplus_{\valpha,\rc}-f_\rc(\rvx\plm)_{\valpha} | \big]  \leq \epsilon. \\[-15pt] 
 \end{align*}

Combined with \eq{E:pseudo_linearity1}, this means that $\rvx\pl/\mu_2(\rvx\plm)=\ty\plplus$ can be approximated arbitrarily well by a linear function of $\rvx\plm$ with probability arbitrarily close to $1$ in $\tl$. 

We have shown that $\rvx\pl/\mu_2(\rvx\plm)$ is arbitrarily well approximated by a linear function of $\rvx\plm$ when normalizing with respect to $\rvx\plm$. Now let us show that $\tx\pl=\rvx\pl/\mu_2(\rvx\pl)$ is arbitrarily well approximated by a linear function of $\rvx\plm$ when normalizing with respect to $\rvx\pl$.

Let us denote $(\ve_1,\dots,\ve_{R\ul})$ and $(\lambda_1,\dots,\lambda_{R\ul})$ respectively the orthogonal eigenvectors and eigenvalues of $\mC_{\rvx,\valpha}[\rf(\rvx\plm,\valpha)]$ and $\hat{\mW}\pl\equiv\mW\pl (\ve_1,\dots,\ve_{R\ul})$. By Corollary~\ref{C:statistics_preserving_constant} there is at least one eigenvalue $\lambda_i$ such that $\lambda_i \geq \mu_2(\rvx\plm)$, which gives combined with \eq{E:S_increments_vanilla_Theorem2} that $\forall \rc$:  \\[-9pt]
 \begin{align*}
\mu_{2,\rc}(\ty\pl) & = \frac{1}{\mu_2(\rvx\plm)} \sum\nolimits_i \big( \hat{\mW}\pl_{\rc,i} \big)^2 \lambda_i, \nonumber\\
\mu_{2,\rc}(\ty\pl) & \geq X, \qquad X \sim_\tl \frac{2}{R\ul} \operatorname{Chi-Squared}(1).\\[-18pt] \nonumber 
\end{align*}

Using \eq{E:S_increments_vanilla_Theorem7}, we then get \\[-9pt]
  \begin{align}
   \mu_{2,\rc}(\rvx\pl)+\mu_{2,\rc}(\bx\pl) & =\mu_{2,\rc}(\rvy\pl)-2\nu_{1,\rc}(\rvy\plplus) \nu_{1,\rc}(\rvy\plminus)=\mu_2(\rvx\plm) \Big( \mu_{2,\rc}(\ty\pl)-2\nu_{1,\rc}(\ty\plplus) \nu_{1,\rc}(\ty\plminus) \Big),\nonumber \\
 \mu_{2,\rc}(\rvx\pl)+\mu_{2,\rc}(\bx\pl) & \geq \mu_2(\rvx\plm) \big( X - Y \big), \qquad X \sim_\tl \frac{2}{R\ul} \operatorname{Chi-Squared}(1), \qquad \forall \epsilon: \sP_\tl[|Y|>\epsilon] \xrightarrow{l\rightarrow\infty} 0. \label{E:pseudo_linearity2}\\[-15pt] \nonumber
\end{align}

Similarly to the proof of \theorem{T:moments_vanilla}, we define \\[-9pt]
\begin{align*}
w\pl_\rc \equiv \left\{
    \begin{array}{cl}
        0 & \mbox{if }  \mu_{2,\rc}(\rvx\pl)<\mu_{2,\rc}(\bx\pl) \vspace{3pt} \\
        1 & \mbox{if }  \mu_{2,\rc}(\rvx\pl)>\mu_{2,\rc}(\bx\pl) \;, \vspace{3pt} \\
        \tilde{w}\pl_\rc & \mbox{if }  \mu_{2,\rc}(\rvx\pl)=\mu_{2,\rc}(\bx\pl) \\
    \end{array}
\right.
\end{align*}

with $\tilde{w}\pl_\rc \sim \operatorname{Bernouilli}(1/2)$ independent of $\vomega\pl$ and $\vbeta\pl$. 

Since $C\ul$ is independent from $\mu_{2,\rc}(\rvx\pl)+\mu_{2,\rc}(\bx\pl)$, it follows from \eq{E:pseudo_linearity2} that $\forall p>0, \exists \eta,\eta\pp>0$ such that for $l$ large enough: \\[-9pt]
 \begin{align}
\sP_{\tl | C\ul}\Bigg[ \mu_{2}(\rvx\pl) \geq \frac{1}{2N\ul} \mu_2(\rvx\plm) \eta \Bigg] & > 1-p,\\
\sP_{\tl | C\ul}\Bigg[ \frac{\mu_2(\rvx\pl)}{\mu_2(\rvx\plm)} \geq \eta\pp \Bigg] & > 1-p. \label{E:pseudo_linearity3} \\[-18pt] \nonumber
\end{align}

Now let us fix $p,\epsilon>0$ and consider $\eta\pp$ as in \eq{E:pseudo_linearity3}. If we suppose that $\forall \rc$: $\min\big(\nu_{1,\rc}(\ty\plplus)\nu_{1,\rc}(\ty\plminus) \big) \leq \sqrt{\eta\pp} \epsilon$, and that $\frac{\mu_2(\rvx\pl)}{\mu_2(\rvx\plm)} \geq \eta\pp$, then there exists a linear function $f: \R^{n\times \dots \times n \times N\ulm} \rightarrow \R^{n\times \dots \times n \times N\ul}$ such that \\[-9pt]
\begin{align*}
 \nu_1(|\ty\plplus-f(\rvx\plm)|) & \leq \sqrt{\eta\pp} \epsilon, \\
 \nu_1(|\tx\pl-\tilde{f}(\rvx\plm)|) & \leq \sqrt{\frac{\mu_2(\rvx\plm)}{\mu_2(\rvx\pl)}} \sqrt{\eta\pp} \epsilon \leq \frac{1}{\sqrt{\eta\pp}}\sqrt{\eta\pp} \epsilon=\epsilon, \\[-15pt] 
 \end{align*}

where we defined $\tilde{f}(\rvx\plm)=\sqrt{\frac{\mu_2(\rvx\plm)}{\mu_2(\rvx\pl)}}f(\rvx\plm)$. Given \eq{E:pseudo_linearity1}, this means that $\tx\pl$ can be approximated with error $\epsilon$ by a linear function of $\rvx\plm$ with probability arbitrarily close to $(1-p)\sP_\tl[C\ul]=(1-p)\big(1-2^{-N\ul}\big)$. Thus $\tx\pl$ can be approximated arbitrarily well by a linear function of $\rvx\plm$ with probability arbitrarily close to $\sP_\tl[C\ul]=1-2^{-N\ul}$. Furthermore $\sP_\tl[C\ul]$ is itself nearly indistinguishable from $1$.

\section{Details of \Section{S:BN_feedforward_nets}}

\subsection{Proof of \theorem{T:S_increments_bn_ff}}
 \label{S:proof_S_increments_bn_ff}

\textbf{\theorem{T:S_increments_bn_ff}} (normalized sensitivity increments of batch-normalized feedforward nets)\textbf{.}\emph{
The dominating term in the evolution of $\S\pl$ can be decomposed as
\begin{align*}
\delta \S\pl = \delta\ubn \S\pl \;\cdot\;  \delta\uphi \S\pl \simeq \exp\big( \overline{m}[ \S\pl ] \big) & = \exp\big( \overline{m}\uBN[ \S\pl ] \big) \,\cdot\, \exp\big( \overline{m}\uPHI[ \S\pl ] \big), \\
\exp\big( \overline{m}\uBN[ \S\pl ] \big) & \equiv \Bigg(\frac{\mu_2(\dx\plm)}{\mu_2(\rvx\plm)}\Bigg)^{-\frac{1}{2}} \E_{\rc,\tl} \Bigg[ \frac{\mu_{2,\rc}(\dy\pl)}{\mu_{2,\rc}(\rvy\pl)} \Bigg]^\frac{1}{2}, \\
\exp\big( \overline{m}\uPHI[ \S\pl ] \big) & \equiv \Big( 1- 2 \vss \E_{\rc,\tl} [ \nu_{1,\rc}  ( \rvz\plplus  ) \nu_{1,\rc} ( \rvz\plminus ) ] \Big)^{-\frac{1}{2}}.\\[-40pt]
\end{align*} 
\begin{align*}
\hspace{158pt}\underbrace{\qquad\qquad\qquad\qquad\qquad\qquad\qquad\;}_{
\hspace{2pt} \in [1,\sqrt{2}]} \\[-20pt]
\end{align*}
}

\begin{proof}\textbf{.} First let us decompose $\delta \S\pl$ as the product of $\delta\ubn \S\pl$ and $\delta\uphi \S\pl$: \\[-9pt]
\begin{align*}
\delta\ubn \S\pl & \equiv \Bigg(\frac{\mu_2(\dz\pl)}{\mu_2(\rvz\pl)}\Bigg)^{\frac{1}{2}} \Bigg(\frac{\mu_2(\dx\plm)}{\mu_2(\rvx\plm)}\Bigg)^{-\frac{1}{2}},\\
\delta\uphi \S\pl & \equiv \Bigg(\frac{\mu_2(\dx\pl)}{\mu_2(\rvx\pl)}\Bigg)^{\frac{1}{2}} \Bigg(\frac{\mu_2(\dz\pl)}{\mu_2(\rvz\pl)}\Bigg)^{-\frac{1}{2}}, \\
\delta \S\pl & = \Bigg(\frac{\mu_2(\dx\pl)}{\mu_2(\rvx\pl)}\Bigg)^{\frac{1}{2}} \Bigg(\frac{\mu_2(\dx\plm)}{\mu_2(\rvx\plm)}\Bigg)^{-\frac{1}{2}} 
= \delta\ubn \S\pl \;\cdot\; \delta\uphi \S\pl.
\end{align*} 

Next let us decompose $\exp\big( \overline{m}[ \S\pl ] \big)$ as the product of two terms:  \\[-9pt]
\begin{align*}
\exp\big( \overline{m}\uBN\big[ \S\pl \big] \big) & = \Bigg(\frac{\E_\tl[\mu_2(\dz\pl)]}{\E_\tl[\mu_2(\rvz\pl)]}\Bigg)^{\frac{1}{2}} \Bigg(\frac{\mu_2(\dx\plm)}{\mu_2(\rvx\plm)}\Bigg)^{-\frac{1}{2}}, \\
\exp\big( \overline{m}\uPHI\big[ \S\pl \big] \big) & =  \Bigg(\frac{\E_\tl[\mu_2(\dx\pl)]}{\E_\tl[\mu_2(\rvx\pl)]}\Bigg)^{\frac{1}{2}} \Bigg(\frac{\E_\tl[\mu_2(\dz\pl)]}{\E_\tl[\mu_2(\rvz\pl)]}\Bigg)^{-\frac{1}{2}}, \\
\exp\big( \overline{m}[ \S\pl ] \big) & = \Bigg(\frac{\E_\tl[\mu_2(\dx\pl)]}{\E_\tl[\mu_2(\rvx\pl)]}\Bigg)^{\frac{1}{2}} \Bigg(\frac{\mu_2(\dx\plm)}{\mu_2(\rvx\plm)}\Bigg)^{-\frac{1}{2}} \\
	& = \exp\big( \overline{m}\uBN[ \S\pl ] \big) \,\cdot\, \exp\big( \overline{m}\uPHI[ \S\pl ] \big).  \\[-15pt] \nonumber
\end{align*}

The term $\exp\big( \overline{m}\uBN[ \S\pl ] \big)$ approximates the geometric increment $\delta\ubn \S\pl$ from $(\rvx\plm,\dx\plm)$ to $(\rvz\pl,\dz\pl)$ such that \mbox{$\exp\big( \overline{m}\uBN[ \S\pl ] \big) \simeq \delta\ubn \S\pl$}, while the term $\exp\big( \overline{m}\uPHI[ \S\pl ] \big)$ approximates the geometric increment $\delta\uphi \S\pl$ from $(\rvz\pl,\dz\pl)$ to $(\rvx\pl,\dx\pl)$ such that \mbox{$\exp\big( \overline{m}\uPHI[ \S\pl ] \big) \simeq \delta\uphi \S\pl$}. \mbox{These terms can be seen (slightly simplistically)} as the direct contribution of respectively batch normalization and the nonlinearity $\phi$ to $\delta \S\pl$. Now let us explicitate both terms. 

\textbf{Term $\exp\big( \overline{m}\uBN\big[ \S\pl \big] \big)$.} First let us note that batch normalization directly gives $\mu_2(\rvz\pl)=1$, and thus $\E_\tl[\mu_2(\rvz\pl)]=1$. Next let us explicitate $\E_{\tl}[\mu_2(\dz\pl)]$:  \\[-9pt]
\begin{gather*}
\forall \rc: \; \dz\pl_{:,\rc} = \frac{\dy\pl_{:,\rc}}{\sqrt{\mu_{2,\rc}(\rvy\pl)}}, \qquad \forall \rc: \; \mu_{2,\rc}(\dz\pl) = \frac{\mu_{2,\rc}(\dy\pl)}{\mu_{2,\rc}(\rvy\pl)}, \\
\E_{\tl}[\mu_2(\dz\pl)] = \E_{\rc,\tl}[\mu_{2,\rc}(\dz\pl)] = \E_{\rc,\tl}\Bigg[ \frac{\mu_{2,\rc}(\dy\pl)}{\mu_{2,\rc}(\rvy\pl)} \Bigg]. \\[-15pt] \nonumber
\end{gather*}

All together, we get that \\[-9pt]
\begin{align*}
\exp\big( \overline{m}\uBN\big[ \S\pl \big] \big) & = \Bigg(\frac{\mu_2(\dx\plm)}{\mu_2(\rvx\plm)}\Bigg)^{-\frac{1}{2}} \E_{\rc,\tl}\Bigg[ \frac{\mu_{2,\rc}(\dy\pl)}{\mu_{2,\rc}(\rvy\pl)} \Bigg]^\frac{1}{2}. \\[-15pt] \nonumber
\end{align*}

\textbf{Term $\exp\big( \overline{m}\uPHI\big[ \S\pl \big] \big)$.} We consider the symmetric propagation for batch-normalized feedforward nets, introduced in \Section{S:complementary_definitions_notations}. From \eq{E:symmetric_bn_mus}, we deduce that \\[-12pt]
\begin{align}
\E_\tl[ \mu_2(\dx\pl)] + \E_\tl[ \mu_2(\bdx\pl)] & = \E_\tl[ \mu_2(\dz\pl)], \nonumber \\
2\E_\tl[ \mu_2(\dx\pl)] & = \E_\tl[ \mu_2(\dz\pl)]  \label{E:S_increments_bn_ff_Theorem1}, \\[-15pt] \nonumber
\end{align}

where \eq{E:S_increments_bn_ff_Theorem1} is obtained by symmetry of the propagation. Next we turn to the symmetric propagation of the signal:  \\[-9pt]
\begin{align}
\mu_{2,\rc}(\rvx\pl)+\mu_{2,\rc}(\bx\pl) & =\E_{\rvx,\valpha}\big[(\rvz\plplus_{\valpha,\rc})^2\big]-\E_{\rvx,\valpha}\big[\rvz\plplus_{\valpha,\rc}\big]^2+\E_{\rvx,\valpha}\big[(\rvz\plminus_{\valpha,\rc})^2\big]-\E_{\rvx,\valpha}\big[\rvz\plminus_{\valpha,\rc}\big]^2 \label{E:S_increments_bn_ff_Theorem2} \\
&=\nu_{2,\rc}(\rvz\plplus)-\nu_{1,\rc}(\rvz\plplus)^2+\nu_{2,\rc}(\rvz\plminus)-\nu_{1,\rc}(\rvz\plminus)^2 \nonumber \\
&=\nu_{2,\rc}(\rvz\pl)-\Big(\nu_{1,\rc}(\rvz\plplus)^2+\nu_{1,\rc}(\rvz\plminus)^2\Big), \nonumber \\[-15pt] \nonumber
\end{align}

where \eq{E:S_increments_bn_ff_Theorem2} follows from \eq{E:symmetric_bn_identity}. Due to the constraints $\nu_{1,\rc}(\rvz\pl)=0$ and $\nu_{2,\rc}(\rvz\pl)=1$, imposed by batch normalization: \\[-9pt]
\begin{gather}
\mu_{2,\rc}(\rvx\pl)+\mu_{2,\rc}(\bx\pl) =1-\Big(\nu_{1,\rc}(\rvz\plplus)^2+\nu_{1,\rc}(\rvz\plminus)^2\Big), \label{E:S_increments_bn_ff_Theorem3} \\
\nu_{1,\rc}(\rvz\pl)=\nu_{1,\rc}(\rvz\plplus)-\nu_{1,\rc}(\rvz\plminus)=0, \nonumber \\
\big( \nu_{1,\rc}(\rvz\plplus)-\nu_{1,\rc}(\rvz\plminus) \big)^2 = \nu_{1,\rc}(\rvz\plplus)^2+\nu_{1,\rc}(\rvz\plminus)^2-2\nu_{1,\rc}(\rvz\plplus)\nu_{1,\rc}(\rvz\plminus) = 0. \label{E:S_increments_bn_ff_Theorem4}  \\[-15pt] \nonumber
\end{gather}

Using \eq{E:S_increments_bn_ff_Theorem3}, \eq{E:S_increments_bn_ff_Theorem4} and the symmetry of the propagation:   \\[-9pt]
\begin{align}
\mu_{2,\rc}(\rvx\pl)+\mu_{2,\rc}(\bx\pl) & = 1 -  2\nu_{1,\rc}(\rvz\plplus)\nu_{1,\rc}(\rvz\plminus),  \nonumber \\
2\E_\tl[ \mu_2(\rvx\pl)]  & = 1 -  2 \vss \E_{\rc,\tl} [ \nu_{1,\rc}  \big( \rvz\plplus  \big) \nu_{1,\rc} \big( \rvz\plminus \big) ].\label{E:S_increments_bn_ff_Theorem5} \\[-15pt] \nonumber
\end{align}

Finally combining \eq{E:S_increments_bn_ff_Theorem1}, \eq{E:S_increments_bn_ff_Theorem5} and $\E_{\theta\pl}[\mu_2(\rvz\pl)]=1$: \\[-9pt]
\begin{align*}
\exp\big( \overline{m}\uPHI\big[ \S\pl \big] \big) & =  \Bigg(\frac{\E_{\tl}[\mu_2(\dx\pl)]}{\E_{\tl}[\mu_2(\rvx\pl)]}\Bigg)^\frac{1}{2} \Bigg(\frac{\E_{\theta\pl}[\mu_2(\dz\pl)]}{\E_{\tl}[\mu_2(\rvz\pl)]}\Bigg)^{-\frac{1}{2}} \\
& = \Big( 1 - 2 \vss \E_{\rc,\tl} [ \nu_{1,\rc}  \big( \rvz\plplus  \big) \nu_{1,\rc} \big( \rvz\plminus \big) ] \Big)^{-\frac{1}{2}}.
\end{align*}

To obtain the bounds on $\exp\big( \overline{m}\uPHI\big[ \S\pl \big] \big)$, the same reasoning as \eq{E:S_increments_vanilla_Theorem9} may be applied to $\rvz\pl$ instead of $\rvy\pl$: \\[-9pt]
\begin{gather*}
4 \nu_{1,\rc}(\rvz\plplus) \nu_{1,\rc}(\rvz\plminus) \leq \mu_{2,\rc}(\rvz\pl)  = 1, \qquad
2 \E_{\rc,\tl} [ \nu_{1,\rc}  ( \rvz\plplus  ) \nu_{1,\rc} ( \rvz\plminus ) ] \leq \frac{1}{2}, \\
 1 \leq  \exp\big( \overline{m}_{\phi}[ \S\pl ] \big) = \Big( 1- 2 \vss \E_{\rc,\tl} [ \nu_{1,\rc}  ( \rvz\plplus  ) \nu_{1,\rc} ( \rvz\plminus ) ] \Big)^{-\frac{1}{2}} \leq \sqrt{2}.\tag*{\qed} \\[-15pt]
\end{gather*}

\end{proof}

\subsection{In the First Step of the Propagation, $\exp\big( \overline{m}\uBN\big[ \S^1 \big] \big)\geq1$}
 \label{S:bn_first_step}

Using again the notations from \Section{S:complementary_definitions_notations}, we may explicitate the second-order moment in channel $\rc$ of $\dy^1$:  \\[-6pt]
\begin{align}
\mu_{2,\rc}(\dy^1) & = \E_{\rvx,\dx,\valpha}\Big[ \fmc(\dy^1,\valpha)^2_\rc \Big] = \E_{\rvx,\dx,\valpha}\Big[ \fm(\dy^1,\valpha)^2_\rc \Big] = \E_{\rvx,\dx,\valpha} \Big[ \big(\mW^1_{\rc,:} \rf(\dx,\valpha) \big)^2 \Big] \label{E:bn_first_step1} \\
& = \sum\nolimits_{i,j} \mW^1_{\rc,i} \mW^1_{\rc,j} \E_{\dx,\valpha}[\rf(\dx,\valpha)_i \vss \rf(\dx,\valpha)_j] \nonumber \\
& = \mu_2(\dx\po) \, \sum\nolimits_i \big( \mW^1_{\rc,i} \big)^2 = \mu_2(\dx\po) \, ||\mW^1_{\rc,:}||^2_2. \label{E:bn_first_step2} \\[-15pt] \nonumber
\end{align}

where \eq{E:bn_first_step1} follows from $\dy^1$ being centered, while \eq{E:bn_first_step2} follows from the white noise property $\E_{\dx}[ \dx_{i} \dx_{j}]= \sigma^2_{\dx} \delta_{ij}= \mu_2(\dx\po) \, \delta_{ij}$, implying $\forall \valpha$: $\E_\dx[\rf(\dx,\valpha)_i \vss \rf(\dx,\valpha)_j] = \mu_2(\dx\po)\, \delta_{ij}$ under periodic boundary conditions.

Now we turn to the second-order moment in channel $\rc$ of $\rvy^1$. Denoting $(\ve_1,\dots,\ve_{R_1})$ and $(\lambda_1,\dots,\lambda_{R_1})$ respectively the orthogonal eigenvectors and eigenvalues of $\mC_{\rvx,\valpha}[\rf(\rvx,\valpha)]$ and $\hat{\mW}^1=\mW^1 (\ve_1,\dots,\ve_{R_1})$, we get that  \\[-6pt]
\begin{align}
\mu_{2,\rc}(\rvy^1) & = \E_{\rvx,\valpha} \Big[ \fmc(\rvy^1,\valpha)^2_\rc \Big] = \E_{\rvx,\valpha} \Big[ \big(\mW^1_{\rc,:} \rfc(\rvx,\valpha) \big)^2 \Big] = \sum\nolimits_i \big( \hat{\mW}^1_{\rc,i} \big)^2 \lambda_i  \nonumber \\
 & = ||\mW^1_{\rc,:}||^2_2 \sum\nolimits_i \big( \tilde{\mW}^1_{\rc,i} \big)^2 \lambda_i = \frac{\mu_{2,\rc}(\dy^1)}{\mu_2(\dx\po)} \sum\nolimits_i \big( \tilde{\mW}^1_{\rc,i} \big)^2 \lambda_i,\label{E:bn_first_step3} \\[-15pt] \nonumber
\end{align}

where we defined $\tilde{\mW}^1$ such that $\forall \rc$: $\tilde{\mW}^1_{\rc,:}=\hat{\mW}^1_{\rc,:} \,\,/\,\, ||\hat{\mW}^1_{\rc,:}||$ and we used \eq{E:bn_first_step2}. Under standard initialization, the distribution of $\hat{\mW}^1$ is spherically symmetric, implying that for all channels $\rc$ the distribution of $\tilde{\mW}^1_{\rc,:}$ is uniform on the unit sphere of $\R^{R_1}$. In turn, this implies that  \\[-9pt]
\begin{gather}
\forall i:\; \E_{\theta^1} \Big[\big(\tilde{\mW}^1_{\rc,i}\big)^2 \Big] =  \frac{1}{R_1}, \nonumber\\ 
\forall \rc:\; \E_{\theta^1}\Big[ \sum\nolimits_i \big( \tilde{\mW}^1_{\rc,i} \big)^2 \lambda_i \Big] = \frac{1}{R_1} \sum\nolimits_i \lambda_i,\qquad \E_{\rc,\theta^1}\Big[ \sum\nolimits_i \big( \tilde{\mW}^1_{\rc,i} \big)^2 \lambda_i \Big] = \frac{1}{R_1} \sum\nolimits_i \lambda_i. \label{E:bn_first_step4} \\[-15pt] \nonumber
\end{gather}

Finally we can write $\exp\big( \overline{m}\uBN\big[ \S^1 \big] \big)$ as  \\[-9pt]
\begin{align}
\exp\big( \overline{m}\uBN\big[ \S^1 \big] \big) & = \Bigg(\frac{\mu_2(\dx\po)}{\mu_2(\rvx\po)}\Bigg)^{-\frac{1}{2}} \E_{\rc,{\theta^1}}\Bigg[ \frac{\mu_{2,\rc}(\dy^1)}{\mu_{2,\rc}(\rvy^1)} \Bigg]^\frac{1}{2} \nonumber \\
& = \Bigg(\frac{\mu_2(\dx\po)}{\frac{1}{R_1} \sum_i \lambda_i}\Bigg)^{-\frac{1}{2}} \E_{\rc,{\theta^1}}\Bigg[ \frac{\mu_2(\dx\po)}{\sum\nolimits_i \big( \tilde{\mW}^1_{\rc,i} \big)^2 \lambda_i} \Bigg]^\frac{1}{2} \label{E:bn_first_step5} \\
& = \Big(\frac{1}{R_1} \sum_i \lambda_i \Big)^\frac{1}{2} \E_{\rc,{\theta^1}}\Bigg[ \frac{1}{\sum\nolimits_i \big( \tilde{\mW}^1_{\rc,i} \big)^2 \lambda_i} \Bigg]^\frac{1}{2} \nonumber \\
& \geq \Big(\frac{1}{R_1} \sum_i \lambda_i \Big)^\frac{1}{2}  \Bigg( \E_{\rc,\theta^1}\Big[ \sum\nolimits_i \big( \tilde{\mW}^1_{\rc,i} \big)^2 \lambda_i \Big]^{-1}\Bigg)^\frac{1}{2} = 1. \label{E:bn_first_step6}  \\[-15pt] \nonumber
\end{align}

where \eq{E:bn_first_step5} was obtained using \eq{E:bn_first_step3} and $\mu_2(\rvx\po)=\mu_2(\rvx)=\frac{1}{R_1} \Tr \mC_{\rvx,\valpha}[\rf(\rvx,\valpha)]=\frac{1}{R_1} \sum_i \lambda_i$ by Corollary~\ref{C:statistics_preserving_constant}, while \eq{E:bn_first_step6} was obtained using the convexity of $x\mapsto1/x$ and \eq{E:bn_first_step4}.  \\

\section{Details of \Section{S:BN_resnets}}

\subsection{Adaptation of the Previous Setup to Resnets}
\label{S:adaptation_resnet}
Before proceeding to the analysis, slight adaptations and forewords are necessary. We denote  \\[-9pt]
\begin{alignat*}{4} 
\Tlh & \equiv(\vomega^{1,1}, \vbeta^{1,1}, \dots, \vomega^{1,H}, \vbeta^{1,H},\dots,\vomega^{l,1}, \vbeta^{l,1},\dots,\vomega^{l,h}, \vbeta^{l,h}), 
& \qquad \tlh & \equiv \Tlh | \Tlhm,\\
\Tl & \equiv (\vomega^{1,1}, \vbeta^{1,1}, \dots, \vomega^{1,H}, \vbeta^{1,H},\dots,\vomega^{l,1}, \vbeta^{l,1},\dots,\vomega^{l,H}, \vbeta^{l,H}), 
& \qquad  \tl & \equiv \Tl | \Tlm.  \\[-15pt] \nonumber
\end{alignat*}

In the pre-activation perspective, each residual layer starts with $(\rvy\plhm,\dy\plhm)$ after the convolution and ends with $(\rvy\plh,\dy\plh)$ again after the convolution.  The concrete effect is that $\BN$ and $\phi$ are completely deterministic conditionally on $\Theta\plm$ in the first layer $h=1$ of each residual unit $l$. This occurs again for $h\geq2$ since $\BN$ and $\phi$ are random conditionally on $\Theta\plm$ but completely deterministic conditionally on $\Theta\plhm$. At even larger granularity, due to the aggregation $(\rvy^{l},\dy^{l}) = \sum\nolimits^l_{k=0} (\rvy^{k,H},\dy^{k,H})$, the input $(\rvy\plm,\dy\plm)$ of each residual unit becomes more and more correlated between successive $l$, and less and less dependent on the random parameters $\theta^{l-k}$ of previous residual units. 

Since the evolution of $\S\pl$ is mainly influenced by batch normalization and the nonlinearity $\phi$, this shift can be thought as attributing the parameters and thus the stochasticity of layer $h$ to layer $h-1$. A simple strategy to apply the results of \Section{S:BN_feedforward_nets} is thus to shift back to the post-activation perspective by considering the parameters $\theta\plhm$ and the evolution from $(\rvx\plhm,\dx\plhm)$ to $(\rvx\plh,\dx\plh)$ for layers $2\leq h \leq H$. \theorem{T:S_increments_bn_ff} strictly applies in this case. 

It remains to understand the evolution from $(\rvy^{l,0},\dy^{l,0})=(\rvy\plm,\dy\plm)$ to $(\rvx^{l,1},\dx^{l,1})$ in layer $h=1$, and the evolution from $(\rvx\plH,\dx\plH)$ to $(\rvy\plH,\dy\plH)$ in layer $h=H$. 

By considering the parameter $\Theta\plm$, the dominating term in the evolution from $(\rvy\plm,\dy\plm)$ to $(\rvz^{l,1},\dz^{l,1})$ is  \\[-9pt]
\begin{align*} 
\Bigg(\frac{\E_{\Tlm}[\mu_2(\dz^{l,1})]}{\E_{\Tlm}[\mu_2(\rvz^{l,1})]} \Bigg)^\frac{1}{2} \Bigg(\frac{\E_{\Tlm}[\mu_2(\dy\plm)]}{\E_{\Tlm}[\mu_2(\rvy\plm)]}\Bigg)^{-\frac{1}{2}} = \Bigg(\frac{\E_{\Tlm}[\mu_2(\dy\plm)]}{\E_{\Tlm}[\mu_2(\rvy\plm)]}\Bigg)^{-\frac{1}{2}} \E_{\rc,\Tlm} \Bigg[ \frac{\mu_{2,\rc}(\dy\plm)}{\mu_{2,\rc}(\rvy\plm)} \Bigg]^\frac{1}{2}.  \\[-15pt] \nonumber
\end{align*}

 Under the assumption of well-conditioned noise, this term is again $\gtrsim 1$ by convexity of $x \mapsto 1/x$. For the nonlinearity term, the symmetric propagation with respect to $\Theta\plm$ applies for all terms in the sum $(\rvy\plm,\dy\plm)=\sum\nolimits\plm_{k=0} (\rvy^{k,H},\dy^{k,H})$, except for $(\rvy^{0,H},\dy^{0,H})=(\rvy,\dy)$. The expression of the nonlinearity term $\exp\big( \overline{m}_\phi[ \S\pl ] \big)$ in \theorem{T:S_increments_bn_ff} thus remains approximately valid.

Finally by spherical symmetry, the evolution from $(\rvx\plH,\dx\plH)$ to $(\rvy\plH,\dy\plH)$ in layer $h=H$ has dominating term  \\[-12pt]
\begin{align*}
\Bigg(\frac{\mu_2(\dx\plH)}{\mu_2(\rvx\plH)}\Bigg)^{-\frac{1}{2}} \Bigg(\frac{\E_{\tlH}[\mu_2(\dy\plH)]}{\E_{\tlH}[\mu_2(\rvy\plH)]} \Bigg)^\frac{1}{2} = 1.
\end{align*}

\emph{In summary, \theorem{T:S_increments_bn_ff} remains approximately valid during the feedforward evolution inside residual units}.

\subsection{Lemma on Dot-Products}

\begin{Lemma}
\label{L:residual_dot_product}
It holds that:  \\[-6pt]
\begin{align*}
\E_{\tl}\big[ \E_{\rvy,\valpha,\rc} \big[ \fmc(\rvy\plm)_\rc \vss \fmc(\rvy\plH,\valpha)_\rc \big] \big] & =0, \\
\E_{\tl}\Big[ \E_{\rvy,\valpha,\rc} \big[ \fmc(\rvy\plm)_\rc \vss \fmc(\rvy\plH,\valpha)_\rc \big]^2 \Big] & \leq \frac{1}{N \reff(\rvy\plm)} \mu_2(\rvy\plm) \E_{\tl}[\mu_2(\rvy\plH)], \\
\E_{\tl}\big[ \E_{\rvy,\dy,\valpha,\rc} \big[ \fmc(\dy\plm)_\rc \vss \fmc(\dy\plH,\valpha)_\rc \big] \big] & =0, \\
\E_{\tl}\Big[ \E_{\rvy,\dy,\valpha,\rc} \big[ \fmc(\dy\plm)_\rc \vss \fmc(\dy\plH,\valpha)_\rc \big]^2 \Big] & \leq \frac{1}{N \reff(\dy\plm)} \mu_2(\dy\plm) \E_{\tl}[\mu_2(\dy\plH)]. 
\end{align*}
\end{Lemma}

\begin{proof}\textbf{.} 
 By spherical symmetry, the moments of $\fmc(\rvy\plH,\valpha)_\rc$ and $\fmc(-\rvy\plH,\valpha)_\rc=-\fmc(\rvy\plH,\valpha)_\rc$ have the same distribution with respect to $\tl$. 

 It follows that  \\[-9pt]
 \begin{align*}
\E_{\tl}\big[ \E_{\rvy,\valpha,\rc} \big[ \fmc(\rvy\plm,\valpha)_\rc \vss \fmc(\rvy\plH,\valpha)_\rc \big] \big]&= 
\E_{\tl}\big[ \E_{\rvy,\valpha,\rc} \big[ \fmc(\rvy\plm,\valpha)_\rc \big(-\fmc(\rvy\plH,\valpha)_\rc \big)  \big]\big], \\
\E_{\tl}\big[ \E_{\rvy,\valpha,\rc} \big[ \fmc(\rvy\plm,\valpha)_\rc \vss \fmc(\rvy\plH,\valpha)_\rc \big] \big]&= 0. \\[-15pt] \nonumber
\end{align*}

Next we note that  \\[-9pt]
\begin{align}
 \E_{\rvy,\valpha,\rc} \big[ \fmc(\rvy\plm,\valpha)_\rc \vss \fmc(\rvy\plH,\valpha)_\rc \big] & = \frac{1}{N} \sum\nolimits_\rc \E_{\rvy,\valpha}\big[\fmc(\rvy\plm,\valpha)_\rc \vss \fmc(\rvy\plH,\valpha)_\rc \big] \nonumber \\ 
& =  \frac{1}{N} \E_{\rvy,\valpha}\Big[ \big\langle \fmc(\rvy\plm,\valpha), \fmc(\rvy\plH,\valpha) \big\rangle \Big], \label{E:Lemma_residual2} \\[-15pt] \nonumber
\end{align}

with $\langle \, , \rangle$ the standard dot product in $\R^N$. 

Let us denote $(\ve_1,\dots,\ve_N)$ and $(\lambda_1,\dots,\lambda_N)$ respectively the orthogonal eigenvectors and eigenvalues of $\mC_{\rvy,\valpha}[\fm(\rvy\plm,\valpha)]$. Let us further denote $u_i$ the unit-variance components of $\fmc(\rvy\plm,\valpha)$ in the basis $(\ve_1,\dots,\ve_N)$ and $y_i$ the components of $\fmc(\rvy\plH,\valpha)$ in the basis $(\ve_1,\dots,\ve_N)$. Then we get that \\[-6pt]
\begin{gather*}
\fmc(\rvy\plm,\valpha) = \sum\nolimits_i \sqrt{\lambda_i} u_i \ve_i, 
\qquad \forall i: \E_{\rvy,\valpha}\big[ u^2_i \big] = 1, 
\qquad \forall j\neq i: \E_{\rvy,\valpha}\big[ u_i u_j \big]=0, \\
\fmc(\rvy\plH,\valpha) = \sum\nolimits_i y_i \ve_i. \\[-15pt]
\end{gather*}

Now we decompose each component $y_i$ of $\rvy\plH$ as  \\[-6pt]
\begin{align*}
\forall j:\, \alpha_{i,j} \equiv \E_{\rvy,\valpha}[ y_i u_j ], \qquad y_i = \sum\nolimits_j \alpha_{i,j} u_j + z_i. \\[-15pt]
\end{align*}

From this definition, we get that \\[-6pt]
\begin{gather}
\forall j: \, \E_{\rvy,\valpha}\big[ z_i u_j \big]=0, \qquad
\E_{\rvy,\valpha}\big[ y_i u_i \big] = \alpha_{i,i}, \qquad 
\E_{\rvy,\valpha}\big[ y^2_i \big] = \sum\nolimits_j \alpha^2_{i,j} + \E_{\rvy,\valpha}\big[ z_i^2  \big], \nonumber \\
\mu_2( \rvy\plH ) = \frac{1}{N} \E_{\rvy,\valpha} \Big[ \big\langle \fmc(\rvy\plH,\valpha), \fmc(\rvy\plH,\valpha) \big\rangle \Big]  = \frac{1}{N} \sum\nolimits_i  \E_{\rvy,\valpha}\big[ y^2_i \big] = \frac{1}{N} \Big( \sum\nolimits_{i,j} \alpha^2_{i,j} + \sum\nolimits_i \E_{\rvy,\valpha}\big[ z_i^2  \big] \Big), \label{E:Lemma_residual3} \\[-15pt] \nonumber
\end{gather}

where the dot product in \eq{E:Lemma_residual3} was computed in the orthogonal basis $(\ve_1,\dots,\ve_N)$. 

Now computing the dot product of $\fmc(\rvy\plm,\valpha)$ and $\fmc(\rvy\plH,\valpha)$ in the orthogonal basis $(\ve_1,\dots,\ve_N)$:  \\[-9pt]
\begin{gather*}
\E_{\rvy,\valpha}\Big[ \big\langle \fmc(\rvy\plm,\valpha) \fmc(\rvy\plH,\valpha) \big\rangle \Big] = \sum\nolimits_i \sqrt{\lambda_i} \E_{\rvy,\valpha}[ y_i u_i] =  \sum\nolimits_i \sqrt{\lambda_i} \alpha_{i,i}. \\[-15pt] \nonumber
\end{gather*}

Spherical symmetry implies that the moments of $y_1 \ve_1 + \dots + y_i \ve_i + \dots + y_N \ve_N$ and $y_1 \ve_1 + \dots - y_i \ve_i + \dots + y_N \ve_N$ have the same distribution with respect to $\theta\pl$. Thus $\forall j\neq i$:  \\[-9pt]
\begin{align*}
\E_{\rvy,\valpha}\big[ y_i u_i \big] \E_{\rvy,\valpha}\big[ y_j u_j \big] & \sim_\tl \E_{\rvy,\valpha}\big[ - y_i u_i \big] \E_{\rvy,\valpha}\big[ y_j u_j \big], \\
\alpha_{i,i} \alpha_{j,j} & \sim_\tl (-\alpha_{i,i}) \alpha_{j,j},  \\
\E_\tl \big[ \alpha_{i,i} \alpha_{j,j} \big] & = 0. \\[-15pt]
\end{align*}

We deduce that   \\[-9pt]
\begin{gather*}
\E_\tl \Bigg[ \E_{\rvy,\valpha}\Big[ \big\langle \fmc(\rvy\plm,\valpha) \fmc(\rvy\plH,\valpha) \big\rangle \Big]^2 \Bigg] = \sum\nolimits_i \lambda_i \E_\tl \big[ \alpha^2_{i,i} \big]. \\[-18pt] \nonumber
\end{gather*}

Spherical symmetry also implies that the distribution of $\alpha_{i,j}$ with respect to $\theta\pl$ does not depend on $i$. Denoting $(\beta_j)$ such that $\forall i,j$: $\beta_j\equiv\E_\tl \big[ \alpha^2_{i,j} \big]$, we get combined with \eq{E:Lemma_residual3}:  \\[-6pt]
\begin{align*}
\E_\tl\big[ \mu_2( \rvy\plH ) \big] \geq \frac{1}{N} \sum\nolimits_{i,j} \E_\tl\big[\alpha^2_{i,j} \big] 
& = \frac{1}{N} \sum\nolimits_{i,j} \beta_j \geq \sum\nolimits_i \beta_i, \\
\E_\tl \Bigg[ \E_{\rvy,\valpha}\Big[ \big\langle \fmc(\rvy\plm,\valpha) \fmc(\rvy\plH,\valpha) \big\rangle \Big]^2 \Bigg] = \sum\nolimits_i \lambda_i \beta_i & \leq \lambda\umax \Big( \sum\nolimits_i \beta_i \Big) \\
& \leq  \lambda\umax \E_\tl\big[ \mu_2( \rvy\plH ) \big]. \\[-15pt] \nonumber
\end{align*}

Finally combining with \eq{E:Lemma_residual2}:  \\[-6pt]
\begin{align}
\E_{\tl}\Bigg[ \E_{\rvy,\valpha,\rc} \Big[ \fmc(\rvy\plm,\valpha)_\rc \vss \fmc(\rvy\plH,\valpha)_\rc \Big]^2 \Bigg] & = \frac{1}{N^2} \E_\tl \Bigg[ \E_{\rvy,\valpha}\Big[ \big\langle \fmc(\rvy\plm,\valpha) \fmc(\rvy\plH,\valpha) \big\rangle \Big]^2 \Bigg]\nonumber \\
& \leq \frac{1}{N^2}  \lambda\umax \E_\tl\big[ \mu_2( \rvy\plH ) \big] \label{E:Lemma_residual4} \\
& \leq \frac{1}{N \reff(\rvy\plm)}  \mu_2(\rvy\plm) \E_\tl\big[ \mu_2( \rvy\plH ) \big],\nonumber  \\[-15pt] \nonumber
\end{align}

where we used $\lambda\umax \reff(\rvy\plm) =\sum\nolimits_i \lambda_i=N \mu_2(\rvy\plm)$. 

The same analysis immediately applies to $\fmc(\dy\plm,\valpha)$ and $\fmc(\dy\plH,\valpha)$. \qed
\end{proof}
\\

\begin{corollary}
\label{C:residual_dot_product}
Let us denote the dot products:  \\[-9pt]
\begin{align*}
Y_l & \equiv \E_{\rvy,\valpha,\rc} \big[ \fmc(\rvy\plm,\valpha)_\rc \vss \fmc(\rvy\plH,\valpha)_\rc \big], \\
T_l & \equiv \E_{\rvy,\dy,\valpha,\rc} \big[ \fmc(\dy\plm,\valpha)_\rc \vss \fmc(\dy\plH,\valpha)_\rc \big], \\
Y_{l,l} & \equiv \E_{\rvy,\valpha,\rc} \big[ \fmc(\rvy\plH,\valpha)_\rc \vss \fmc(\rvy\plH,\valpha)_\rc \big] = \mu_2(\rvy\plH), \\
T_{l,l} & \equiv \E_{\rvy,\dy,\valpha,\rc} \big[ \fmc(\dy\plH,\valpha)_\rc \vss \fmc(\dy\plH,\valpha)_\rc \big] = \mu_2(\dy\plH). \\[-15pt] \nonumber
\end{align*}

Then by spherical symmetry $\forall l$, $\forall l\pp \neq l$:  \\[-9pt]
\begin{align*}
\E_{\Tl}[ Y_l ]=0, \qquad \E_{\Theta^{\max(l,l\pp)}}[ Y_l Y_{l\pp}]=0, \\
\E_{\Tl}[ T_l ]=0, \qquad \E_{\Theta^{\max(l,l\pp)}}[ T_l T_{l\pp}]=0. \\[-15pt] \nonumber
\end{align*}

Furthermore given Lemma \ref{L:residual_dot_product} and given $\reff(\rvy\plm)$, $\reff(\dy\plm) \geq 1$, we deduce that  \\[-6pt]
\begin{align*}
\E_{\Tl}\big[ Y^2_l \big] & \leq \frac{1}{N} \E_{\Tlm}\Big[ \mu_2(\rvy\plm) \E_{\tl}[\mu_2(\rvy\plH)] \Big] \\
 & \leq \frac{1}{N} \E_{\Tlm}\Big[ \mu_2(\rvy\plm) \E_{\tl}[Y_{l,l}] \Big], \\
 \E_{\Tl}\Bigg[  \Big( \frac{\mu_2(\rvy\po)}{\mu_2(\dy\po)(\S\plm)^2} T_l \Big)^2 \Bigg] & \leq \frac{1}{N} \E_{\Tlm}\Bigg[ \frac{\mu_2(\rvy\po)}{\mu_2(\dy\po)(\S\plm)^2} \mu_2(\dy\plm) 
 \E_{\tl}\Bigg[ \frac{\mu_2(\rvy\po)}{\mu_2(\dy\po)(\S\plm)^2} \mu_2(\dy\plH) \Bigg] \Bigg] \\
  & \leq \frac{1}{N} \E_{\Tlm}\Bigg[ \mu_2(\rvy\plm)  \E_{\tl}\Bigg[ \frac{\mu_2(\rvy\po)}{\mu_2(\dy\po)(\S\plm)^2} T_{l,l} \Bigg] \Bigg].  \\[-15pt]
\end{align*}
These inequalities will be useful in the proof of \theorem{T:S_increments_bn_res}.

\end{corollary}

\subsection{Proof of \theorem{T:S_increments_bn_res}}
 \label{S:proof_S_increments_bn_res}

\textbf{\theorem{T:S_increments_bn_res}} (normalized sensitivity increments of batch-normalized resnets)\textbf{.}\emph{
Suppose that we can bound signal variances: $\mu\udmin \lesssim \mu_2(\rvy\plH) \lesssim \mu\udmax$ and feedforward increments: $\delta\umin \lesssim \delta \S^{l,h} \lesssim \delta\umax$ for all $l,h$. Further denote $\eta\umin\equiv\big( (\delta\umin)^{2H} \mu\udmin - \mu\udmax \big) / \mu\udmax$ and $\eta\umax\equiv\big( (\delta\umax)^{2H} \mu\udmax - \mu\udmin \big)/ \mu\udmin$, as well as $\tau\umin\equiv\eta\umin/2$ and $\tau\umax\equiv\eta\umax/2$. Then there exist positive constants $C\umin$, $C\umax>0$ such that  \\[-12pt]
\begin{alignat*}{4}
 \Big( 1 + \frac{\eta\umin}{l+1} \Big)^\frac{1}{2} 
 & \lesssim && \,\, \delta \S\pl && \lesssim \Big( 1 + \frac{\eta\umax}{l+1}  \Big)^\frac{1}{2},  \\
C\umin l^{\tau\umin} & \lesssim && \;\; \S\pl \, && \lesssim C\umax l^{\tau\umax}.  \\[-24pt] \nonumber
\end{alignat*}}

\begin{proof}\textbf{.} 
First we introduce the additional constants $\gamma\umin\equiv(\delta\umin)^{2H}$ and $\gamma\umax\equiv(\delta\umax)^{2H}$ such that we can write \mbox{$\eta\umin= \big(\gamma\umin \mu\udmin - \mu\udmax \big)\, / \, \mu\udmax$} and $\eta\umax= \big(\gamma\umax \mu\udmax - \mu\udmin \big)\, / \, \mu\udmin$.

We also remind that we write $a\lesssim b$ when $a(1+\epsilon_a) \leq b (1+\epsilon_b)$ with $|\epsilon_a|\ll 1$, $|\epsilon_b| \ll 1$ with high probability. And we write $a\simeq b$ when $a(1+\epsilon_a) = b (1+\epsilon_b)$ with $|\epsilon_a|\ll 1$, $|\epsilon_b| \ll 1$ with high probability. Denoting $\wedge$ the logical \emph{and}, $\lor$ the logical \emph{or}, the following rules are easily verified:  \\[-6pt]
\begin{align*}
(a \lesssim b) \wedge (b \lesssim a) & \iff (a \simeq b), \\
(a \lesssim b) & \iff (-a \gtrsim -b), \\
(a \lesssim b) \wedge (b \lesssim c) & \implies (a \lesssim c), \\
(a \lesssim b) \wedge (c \lesssim d) \wedge (a > 0) \wedge (c > 0)  & \implies (ac \lesssim bd), \\
(a \lesssim b) \wedge (c \lesssim d) \wedge (a > 0) \wedge (c > 0)  & \implies (a+c \lesssim b+d), \\
(a \lesssim b) \wedge (a > 0) \wedge (b > 0) & \implies (\sqrt{a} \lesssim \sqrt{b}), \\
(a \lesssim b) \wedge (a > 0) \wedge (b > 0) & \implies (1/a \gtrsim 1/b).\\[-15pt]
\end{align*}

Finally let $a$ be a random variable depending on $\Tl$ with well-defined moments and let $b$ be a constant. Let us prove that \\[-6pt]
\begin{align*}
(a \lesssim b) \implies \big( \E_\tl[ a ] \lesssim b \big) \wedge \big( \E_\Tl[ a ] \lesssim b \big).  \\[-15pt]
\end{align*}

Given the assumption $(a \lesssim b)$, there exists an event $A$ with $\sP_\Tl[A] \simeq 1$ such that under $A$: $a (1+\epsilon_a) \leq b (1+\epsilon_b)$ with $|\epsilon_a|\ll 1$, $|\epsilon_b| \ll 1$. Furthermore, using Cauchy-Schwarz inequality:  \\[-6pt]
\begin{align}
\frac{1}{\E_\tl [ a ]^2} \Big( \E_\tl [ a ] - \E_\tl [ \1_A a ] \Big)^2  = \frac{1}{\E_\tl [ a ]^2} \E_\tl[ \1_{A^\textnormal{c}} a ]^2 \leq \sP_\tl [  A^\textnormal{c} ] \frac{\E_\tl[ a^2 ]}{\E_\tl [ a ]^2}. \label{E:S_increments_bn_res_proof4} \\[-15pt] \nonumber
\end{align}

Since $\sP_\Tl[A] \simeq 1$, the complementary event $A^\textnormal{c}$ has probability $\sP_\Tl[A^\textnormal{c}] \ll 1$. Now by contradiction, if there would be non negligible probability with respect to $\Tlm$ that $\sP_\tl [  A^\textnormal{c} ]=\sP_{\Tl | \Tlm} [  A^\textnormal{c} ]$ is non negligible, then we would not have that $\sP_\Tl[A^\textnormal{c}]=\E_{\Tlm} \E_{\Tl| \Tlm} [  \1_{A^\textnormal{c}} ] =\E_{\Tlm} \sP_{\Tl| \Tlm} [  A^\textnormal{c} ]$ is negligible. It follows that $\sP_\tl [  A^\textnormal{c} ] \ll 1$ with high probability with respect to $\Tlm$. 

Combined with \eq{E:S_increments_bn_res_proof4} and the definition of $A$, we get \\[-6pt]
\begin{align*}
\E_\tl [ a ] \simeq  \E_\tl [ \1_A a ] \lesssim b. \\[-15pt]
\end{align*}

A similar reasoning gives  \\[-6pt]
\begin{align*}
\frac{1}{\E_\Tl [ a ]^2} \Big( \E_\Tl [ a ] - \E_\Tl [ \1_A a ] \Big)^2  
\leq \sP_\Tl [  A^\textnormal{c} ] \frac{\E_\Tl [ a^2 ]}{\E_\Tl [ a ]^2}, \qquad \E_\Tl [ a ] \simeq  \E_\Tl [ \1_A a ] \lesssim b. \\[-15pt] \nonumber
\end{align*}

\emph{We keep all these rules in mind in the course of this proof}.  \\

\textbf{Proof of \eq{E:S_increments_bn_res_Theorem1}.} 
Adopting the notations of Corollary \ref{C:residual_dot_product} and using $\rvy\pl=\rvy\plm+\rvy\plH$ by \eq{E:bn_res_propagation1}, we get that \\[-9pt]
\begin{align}
\mu_2(\rvy\pl) & = \E_{\rvy,\valpha,\rc}\Bigg[\Big( \fmc(\rvy\plm,\valpha)_\rc+\fmc(\rvy\plH,\valpha)_\rc \Big)^2 \Bigg] =\mu_2(\rvy\plm)+Y_{l,l} +2Y_l, \label{E:S_increments_bn_res_proof5} \\ 
 \mu_2(\dy\pl) & = \E_{\rvy,\dy,\valpha,\rc}\Bigg[\Big( \fmc(\dy\plm,\valpha)_\rc+\fmc(\dy\plH,\valpha)_\rc \Big)^2 \Bigg] =\mu_2(\dy\plm)+T_{l,l}+2T_l. \nonumber \\[-15pt] \nonumber
\end{align}

Due to the hypothesis $\mu\udmin \lesssim Y_{l,l}=\mu_2(\rvy\plH) \lesssim \mu\udmax$, we have $\mu\udmin \lesssim \mu_2(\rvy\po)=\mu_2(\rvy^{0,H})\lesssim \mu\udmax$. \\[-9pt]

Now let us reason by induction and suppose that $l \mu\udmin \lesssim \mu_2(\rvy\plm) \lesssim l \mu\udmax$. Combined with \eq{E:S_increments_bn_res_proof5}, we get that  \\[-9pt]
\begin{gather*}
l \mu\udmin + \mu\udmin + 2Y_l \lesssim \mu_2 (\rvy\pl) \lesssim l \mu\udmax + \mu\udmax + 2Y_l. \\[-15pt] \nonumber
\end{gather*}

On the other hand, Corollary \ref{C:residual_dot_product} implies that  \\[-9pt]
\begin{gather*}
\E_\Tl \big[ Y^2_l \big] \lesssim \frac{1}{N} l \mu\udmax^2 \leq \frac{1}{N} \frac{1}{l+1} (l+1)^2 \mu\udmax^2. \\[-15pt] \nonumber
\end{gather*}

Further using Chebyshev's inequality, we deduce that \\[-6pt]
\begin{gather}
 \sP_{\Tl}\Big[  \big| Y_l \big| > k \frac{1}{\sqrt{N}} \frac{1}{\sqrt{l+1}} (l+1)\mu\udmax \Big] \lesssim \frac{1}{k^2}. \label{E:S_increments_bn_res_proof6} \\[-15pt] \nonumber
\end{gather}

For large width $N \gg 1$, it follows that $\big| Y_l \big| \ll (l+1) \mu\udmin$ and $\big| Y_l \big| \ll (l+1) \mu\udmax$ with high probability, and thus that  \\[-6pt]
\begin{align}
(l+1) \mu\udmin \lesssim \mu_2(\rvy\pl) \lesssim (l+1) \mu\udmax \label{E:S_increments_bn_res_proof7}. \\[-15pt] \nonumber
\end{align}

Then \Eq{E:S_increments_bn_res_proof7} holds for all $l$, and furthermore $\big|Y_l\big| \ll \mu_2(\rvy\plm)$ with high probability. Now let us write $(\S\pl)^2$ as  \\[-6pt]
\begin{align*}
(\S\pl)^2 & = \Bigg(\frac{\mu_2(\rvy\po)}{\mu_2(\dy\po)}\Bigg)  \Bigg(\frac{\mu_2(\dy\pl)}{\mu_2(\rvy\pl)}\Bigg) = \frac{\mu_2(\rvy\po)}{\mu_2(\dy\po)} \frac{\mu_2(\dy\plm) + T_{l,l} + 2 T_l}{\mu_2(\rvy\plm) + Y_{l,l} + 2 Y_l}, \nonumber \\
(\S\pl)^2 & = (\S\plm)^2 \; \frac{\mu_2(\rvy\plm) + \frac{\mu_2(\rvy\po)}{\mu_2(\dy\po)(\S\plm)^2} T_{l,l} + 2 \frac{\mu_2(\rvy\po)}{\mu_2(\dy\po)(\S\plm)^2} T_l}{\mu_2(\rvy\plm) + Y_{l,l} + 2 Y_l}. \\[-15pt] \nonumber
\end{align*}

Denoting $\tilde{T}_{l,l}\equiv\frac{\mu_2(\rvy\po)}{\mu_2(\dy\po)(\S\plm)^2} T_{l,l}$ and $\tilde{T}_l\equiv\frac{\mu_2(\rvy\po)}{\mu_2(\dy\po)(\S\plm)^2} T_l$, we then get  \\[-6pt]
\begin{align}
(\delta \S\pl)^2 = \frac{(\S\pl)^2}{(\S\plm)^2} =  \frac{\mu_2(\rvy\plm) + \tilde{T}_{l,l} + 2\tilde{T}_l}{\mu_2(\rvy\plm) + Y_{l,l} + 2 Y_l}. \label{E:S_increments_bn_res_proof8} \\[-15pt] \nonumber
\end{align}

We can bound $\tilde{T}_{l,l}$ as  \\[-9pt]
\begin{gather}
\tilde{T}_{l,l} = \frac{\mu_2(\rvy\po)}{\mu_2(\dy\po)(\S\plm)^2} \mu_2(\dy\plH)= \frac{\mu_2(\rvy\po)}{\mu_2(\dy\po)(\S\plm)^2} (\S\plm)^2 \prod\nolimits_h (\delta \S^{l,h})^2 \mu_2(\rvy\plH) \frac{\mu_2(\dy\po)}{\mu_2(\rvy\po)},  \nonumber  \\
\gamma\umin \mu\udmin \lesssim \tilde{T}_{l,l} \lesssim \gamma\umax \mu\udmax. \label{E:S_increments_bn_res_proof10} \\[-15pt] \nonumber
\end{gather}

By Corollary \ref{C:residual_dot_product}, the variance of $\tilde{T}_l$ is bounded as  \\[-6pt]
\begin{align*}
  \E_{\Tl}\big[ \tilde{T}^2_l  \big] & \lesssim \frac{1}{N} \E_{\Tlm}\Big[ \mu_2(\rvy\plm) \E_{\tl}\big[ \tilde{T}_{l,l} \big] \Big] \lesssim \frac{1}{N} \gamma\umax l \mu^2\udmax. \\[-15pt] \nonumber
\end{align*}

The same reasoning as \eq{E:S_increments_bn_res_proof6} implies both $\big|Y_l\big| \ll \mu_2(\rvy\plm)$ and $\big|\tilde{T}_l\big| \ll \mu_2(\rvy\plm)$ with high probability. Finally combining \eq{E:S_increments_bn_res_proof8}, \eq{E:S_increments_bn_res_proof10} and the hypothesis $\mu\udmin \lesssim Y_{l,l} \lesssim \mu\udmax$:  \\[-9pt]
\begin{alignat*}{3}
\frac{\mu_2(\rvy\plm) + \gamma\umin \mu\udmin }{\mu_2(\rvy\plm) +  \mu\udmax} & \lesssim 
(\delta \S\pl)^2 && \lesssim \frac{\mu_2(\rvy\plm) + \gamma\umax \mu\udmax }{\mu_2(\rvy\plm) +  \mu\udmin}, \\
1 + \frac{\gamma\umin \mu\udmin - \mu\udmax }{\mu_2(\rvy\plm) +  \mu\udmax} & \lesssim 
(\delta \S\pl)^2 && \lesssim 1 + \frac{\gamma\umax \mu\udmax - \mu\udmin }{\mu_2(\rvy\plm) +  \mu\udmin}, \\
1 + \frac{\gamma\umin \mu\udmin - \mu\udmax }{ (l+1) \mu\udmax} & \lesssim 
(\delta \S\pl)^2 && \lesssim 1 + \frac{\gamma\umax \mu\udmax - \mu\udmin }{ (l+1) \mu\udmin}, \\
\Big( 1 + \frac{\eta\umin}{l+1} \Big)^\frac{1}{2} & \lesssim \;\; \delta \S\pl && \lesssim \Big( 1 +  \frac{\eta\umax}{l+1}  \Big)^\frac{1}{2}, \\[-15pt]
\end{alignat*} 
where we supposed $\big(\gamma\umin \mu\udmin - \mu\udmax\big)\geq0$ (see \Section{S:adaptation_resnet} and the evolution of \fig{F:BN-FF} for the justification). \qed
\end{proof}

\begin{proof}\textbf{ of \eq{E:S_increments_bn_res_Theorem2}.} 
Expanding \eq{E:S_increments_bn_res_Theorem1}, we get that   \\[-9pt]
\begin{align*}
\prod^l_{k=1} \Big( 1 +  \frac{\eta\umin}{k+1}  \Big)^\frac{1}{2} \lesssim \S\pl = \prod^l_{k=1} \delta \S\pk \lesssim \prod^l_{k=1} \Big( 1 +  \frac{\eta\umax}{k+1}  \Big)^\frac{1}{2}. \\[-15pt] \nonumber
\end{align*}

We can further explicitate the bounds:  \\[-9pt]
\begin{align}
 \sum^{l}_{k=1} & \log \Big( 1 +  \frac{\eta\umax}{k+1}  \Big) \nonumber \\
   & \leq  \int^{l+1}_{1} \log \Big( 1 +  \frac{\eta\umax}{x} \Big) dx \nonumber \\
   & \leq  \int^{l+1}_{1} \log(x+\eta\umax) dx - \int^{l+1}_{1} \log x \, dx \nonumber \\
   & \leq \Big[x \log x - x\Big]^{l+1+\eta\umax}_{1+\eta\umax} - \Big[x \log x-x\Big]^{l+1}_{1} \nonumber \\
   & \leq (l+1+\eta\umax) \log(l+1+\eta\umax)-(1+\eta\umax) \log(1+\eta\umax)-(l+1)\log (l+1) \nonumber \\
   & \leq \eta\umax \log(l+1+\eta\umax) + (l+1) \log\Big(1+\frac{\eta\umax}{l+1}\Big)-(1+\eta\umax) \log(1+\eta\umax) \nonumber \\
   & \leq \eta\umax \log(l+1+\eta\umax) + \eta\umax -(1+\eta\umax) \log(1+\eta\umax), \label{E:S_increments_bn_res_proof11}  \\[-15pt] \nonumber
\end{align}

where we used $\log (1+x) \leq x$ in \eq{E:S_increments_bn_res_proof11}. Considering the integration between $2$ and $l+2$, we similarly get:  \\[-9pt]
\begin{align*}
 \sum^{l}_{k=1} & \log \Big( 1 +  \frac{\eta\umin}{k+1}  \Big) \qquad \\
   & \geq \eta\umin \log(l+2+\eta\umin) + (l+2) \log\Big(1+\frac{\eta\umin}{l+2}\Big)-(2+\eta\umin) \log(2+\eta\umin) + 2 \log 2 \\
   & \geq \eta\umin \log(l+2+\eta\umin) -(2+\eta\umin) \log(2+\eta\umin) + 2 \log 2. \qquad \qquad \qquad \qquad \quad \\[-15pt]
\end{align*}

Let $c\umax\equiv\exp\Big(\eta\umax -(1+\eta\umax) \log(1+\eta\umax)\Big)$ and $c\umin\equiv\exp\Big(-(2+\eta\umin) \log(2+\eta\umin)+2 \log 2\Big)$. Then:  \\[-9pt]
\begin{align*}
\prod^l_{k=1} \Big( 1 +  \frac{\eta\umax}{k+1}  \Big) \leq c\umax (l+1+\eta\umax)^{\eta\umax}, \qquad \prod^l_{k=1} \Big( 1 +  \frac{\eta\umin}{k+1}  \Big) \geq c\umin (l+2+\eta\umin)^{\eta\umin}, \\[-24pt]
\end{align*}
\begin{alignat*}{2}
 \sqrt{c\umin} (l+2+\eta\umin)^{\eta\umin/2} & \lesssim \S\pl \lesssim \sqrt{c\umax} (l+1+\eta\umax)^{\eta\umax/2}, \\
 \sqrt{c\umin} (l+2+\eta\umin)^{\tau\umin} & \lesssim \S\pl \lesssim \sqrt{c\umax} (l+1+\eta\umax)^{\tau\umax}. \\[-15pt]
\end{alignat*}

Since $x\mapsto \Big(\frac{x+2+\eta\umin}{x}\Big)^{\tau\umin}$ and $x\mapsto \Big(\frac{x+1+\eta\umax}{x}\Big)^{\tau\umax}$ are lower-bounded and upper-bounded for $x \geq 1$, there exist positive constants $C\umin$, $C\umax>0$ such that  \\[-9pt]
\begin{align*}
 C\umin l^{\tau\umin} \lesssim \S\pl \lesssim  C\umax l^{\tau\umax}. \tag*{\qed}
\end{align*}
\end{proof}

\subsection{\theorem{T:S_increments_bn_res} Holds for any Choice of $\phi$, with and without Batch Normalization, as long as the Existence of $\mu\udmin$, $\mu\udmax$, $\delta\umin$, $\delta\umax$ is Ensured}
 \label{S:increments_bn_res_other_phi}

The proof of Lemma \ref{L:residual_dot_product} neither requires batch normalization nor does it require any assumption on $\phi$. In addition, the proof still holds up to \eq{E:Lemma_residual4} when replacing $\fmc(\rvy\plm)$, $\fmc(\rvy\plH)$, $\mu_2(\rvy\plm)$, $\mu_2(\rvy\plH)$ by $\fm(\rvy\plm)$, $\fm(\rvy\plH)$, $\nu_2(\rvy\plm)$, $\nu_2(\rvy\plH)$ and eigenvalues of $\mC_{\rvy,\valpha}[\fm(\rvy\plm,\valpha)]$ by eigenvalues of $\mG_{\rvy,\valpha}[\fm(\rvy\plm,\valpha)]$. This gives \\[-9pt]
\begin{align}
\E_{\tl}\Bigg[ \E_{\rvy,\valpha,\rc} \Big[ \fm(\rvy\plm,\valpha)_\rc \vss \fm(\rvy\plH,\valpha)_\rc \Big]^2 \Bigg] 
& \leq \frac{1}{N^2}  \lambda\umax \E_\tl\big[ \nu_2( \rvy\plH ) \big] \nonumber \\
& \leq \frac{1}{N} \nu_2(\rvy\plm) \E_\tl\big[ \nu_2( \rvy\plH ) \big]. \label{E:increments_bn_res_other_phi1} \\[-15pt] \nonumber
\end{align}

Similarly, the proof of \theorem{T:S_increments_bn_res} only depends on batch normalization and the choice of $\phi$ through the constants $\mu\udmin$, $\mu\udmax$, $\delta\umin$, $\delta\umax$. As a consequence, \theorem{T:S_increments_bn_res} holds for any choice of $\phi$, with and without batch normalization, as long as the existence of $\mu\udmin$, $\mu\udmax$, $\delta\umin$, $\delta\umax$ is ensured.

It is therefore interesting to determine in which cases the constants $\mu\udmin$, $\mu\udmax$, $\delta\umin$, $\delta\umax$ exist. In the forthcoming analysis, we will consider the common cases $\phi=\tanh$ and $\phi=\relu$, with and without batch normalization, relating our results and providing extensions to \citet{Yang17}. 

\emph{For the sake of brevity, some results will be established only with an informal proof.}

\subsubsection{Case $\phi=\tanh$, without Batch Normalization} 
From $\rvx\plH=\phi(\rvy\plHm)$, we deduce that $\nu_2(\rvx\plH)$, $\mu_2(\rvx\plH)$ are bounded as  \\[-9pt]
\begin{align*}
\mu_2(\rvx\plH) \leq \nu_2(\rvx\plH) = \E_{\rvx,\valpha}\big[ \phi(\rvy\plHm)^2 \big] \leq 1. \\[-15pt]
\end{align*}

Since $\rvy\plH$ is obtained from $\rvx\plH$ only after a single single convolution step, it follows that $\nu_2(\rvy\plH)$, $\mu_2(\rvy\plH)$ are bounded from above. Let us further admit that $\nu_2(\rvy\plH)$, $\mu_2(\rvy\plH)$ are bounded from below so that the existence of $\mu\udmin$, $\mu\udmax$ is ensured.

Now let us see whether $\delta\umin$, $\delta\umax$ exist in the mean-field limit: $N \to \infty$, where $\rvy\pl$ becomes a Gaussian process and all moment-related quantities become deterministic with the expectation over $\Tl$ equivalent to the average over channels. Using Lemma \ref{L:residual_dot_product} as well as \eq{E:increments_bn_res_other_phi1}, combined with the reasoning of \eq{E:S_increments_bn_res_proof7} on $\nu_2(\rvy\pl)$ and $\mu_2(\rvy\pl)$ for large $N\gg1$: \\[-9pt]
\begin{align*}
\nu_2(\rvy\plm) \propto l, \qquad \mu_2(\rvy\plm) \propto l. \\[-15pt]
\end{align*}

The probability of non-negligible $\phi\pp(\rvy^{l,0})^2=\phi\pp(\rvy\plm)^2$ is equal to the probability that $\rvy\plm$ is roughly $\bigO(1)$, which scales as $\frac{1}{\sqrt{\nu_2(\rvy\plm)}}\frac{1}{\sqrt{2\pi}}\propto \frac{1}{\sqrt{l}}$ for large $l$. Combined with $\dx^{l,1} = \phi\pp (\rvy^{l,0}) \odot \dy^{l,0}$, this implies that  \\[-9pt]
\begin{align*}
\frac{\mu_2(\dx^{l,1})}{\mu_2(\dy^{l,0})} \propto \frac{1}{\sqrt{l}}. \\[-15pt]
\end{align*}

Given $\mu_2(\rvx^{l,1}) \leq \nu_2(\rvx^{l,1}) = \E_{\rvx,\valpha}\big[ \phi(\rvy^{l,0})^2 \big] \leq 1$, we get for the ratio of signal variances: \\[-9pt]
\begin{align*}
\frac{ \mu_2(\rvy^{l,0}) }{\mu_2(\rvx^{l,1})} \geq  \mu_2(\rvy\plm) \propto l. \\[-15pt]
\end{align*}

This gives for the \emph{squared} geometric increment during the nonlinearity step from $(\rvy^{l,0},\dy^{l,0})$ to $(\rvx^{l,1},\dx^{l,1})$: \\[-9pt]
\begin{align*}
\Bigg( \frac{\mu_2(\dx^{l,1})}{\mu_2(\rvx^{l,1})} \Bigg) \Bigg( \frac{ \mu_2(\dy^{l,0}) } {\mu_2(\rvy^{l,0})}\Bigg)^{-1} \geq  \mu_2(\rvy\plm) \frac{\mu_2(\dx^{l,1})}{\mu_2(\dy^{l,0})} \propto \sqrt{l}. \\[-15pt]
\end{align*}

It follows that $\delta\S^{l,1}$ and thus $\delta\S\plh$ are \emph{not} bounded from above and that the existence of $\delta\umax$ is \emph{not} ensured. Now if we replace $\eta\umin$, $\eta\umax$ by $\frac{\mathcal{A}}{2}\sqrt{l+1} \propto \sqrt{l}$ in \eq{E:S_increments_bn_res_Theorem1}: \\[-9pt]
\begin{align*}
\delta \S\pl \simeq \Big( 1 + \frac{\mathcal{A}}{2 \sqrt{l+1}} \Big)^\frac{1}{2}. \\[-15pt]
\end{align*}

Given $\frac{1}{2}\log(1+\frac{\mathcal{A}}{2 \sqrt{x}})~\simeq~\frac{\mathcal{A}}{4\sqrt{x}}$ and $\int_{x_0}^x \frac{\mathcal{A}}{4\sqrt{x\pp}} dx\pp~\simeq~ \frac{\mathcal{A}}{2} \sqrt{x}$ for $x\gg1$, we get that $\S\pl=\prod_k \delta \S\pk=\exp\big(\sum_k \log \delta \S\pl \big)\propto\exp\big(\frac{\mathcal{A}}{2}\sqrt{l}\big)$. Combined with $\mu_2(\rvy\plm) \propto l$ and the definition of $\S\pl$, we deduce that $\mu_2(\dy\plm)\propto\exp\big(\mathcal{A}\sqrt{l}\big)$, which is exactly the scaling found in \citet{Yang17} for the corresponding quantity. 

\emph{In summary, the growth of $\S\pl$ is slightly subexponential but still far from power-law}.

\subsubsection{Case $\phi=\tanh$, with Batch Normalization} 
\label{S:case_tanh_BN}
Batch normalization controls signal variance inside residual units: $\mu_2(\rvz\plH)~=~1$. Since $\rvy\plH$ is obtained from $\rvz\plH$ only after a single nonlinearity step and a single convolution step, the existence of $\mu\udmin$, $\mu\udmax$ is ensured.

Now let us see whether $\delta\umin$, $\delta\umax$ exist and let us first limit our reasoning to the feedforward evolution of \Section{S:BN_feedforward_nets}. Since the reasoning of \Section{S:BN_feedforward_nets} on the effect of batch normalization applies for any choice of $\phi$, the assumption of well-conditioned noise implies that $\exp( \overline{m}\uBN[ \S\pl ] )$ is bounded: (i) from above by considering the signal with worst possible conditioning; (ii) from below by $1$. 

Regarding $\exp( \overline{m}\uphi[ \S\pl ] )$, let us consider again the mean-field limit: $N \to \infty$ such that $\rvz\pl$ is Gaussian with variance equal to $\nu_2(\rvz\pl)=\mu_2(\rvz\pl)=1$. Then $\exp( \overline{m}\uphi[ \S\pl ] )$ is deterministic and constant, implying that $\exp( \overline{m}\uphi[ \S\pl ] )$ is bounded by constants from above and below.

Since the evolution inside residual units is well approximated by the feedforward evolution of \Section{S:BN_feedforward_nets}, it follows that $\delta \S\plh$ is bounded from above and below.

\emph{In summary, \theorem{T:S_increments_bn_res} applies and $\S\pl$ has power-law growth}.

\subsubsection{Case $\phi=\relu$, without Batch Normalization} 
\label{S:case_ReLU_noBN}
Since the evolution inside residual units is well approximated by the feedforward evolution of \Section{S:vanilla_nets}, it follows that $\nu_2(\rvy\plh)$, $\mu_2(\dy\plh)$ are roughly stable and that the increments $\delta \S^{l,h}$ are limited inside residual units. This implies that $\nu_2(\rvy\plH) \simeq \nu_2(\rvy^{l,0})=\nu_2(\rvy\plm)$ and $\mu_2(\rvy\plH) \simeq \mu_2(\rvy^{l,0})=\mu_2(\rvy\plm)$. Combined with \eq{E:S_increments_bn_res_proof5} and the fact that $Y_{l,l}=\mu_2(\rvy\plH)$ and that $Y\ul \ll \mu_2(\rvy\plm)$ with high probability for $N\gg1$, we deduce that \\[-9pt]
\begin{align*}
\mu_2(\rvy\pl)\simeq\mu_2(\rvy\plm)+\mu_2(\rvy\plH)\simeq2\mu_2(\rvy\plm). \\[-15pt]
\end{align*}

Using \eq{E:increments_bn_res_other_phi1}, the same reasoning for non-central moments gives\\[-9pt]
\begin{align*}
\nu_2(\rvy\pl)\simeq\nu_2(\rvy\plm)+\nu_2(\rvy\plH)\simeq2\nu_2(\rvy\plm). \\[-15pt]
\end{align*}

This means that both $\mu_2(\rvy\pl)$ and $\nu_2(\rvy\pl)$ have exponential growth and that the existence of $\mu\udmax$ is not ensured. The exponential growth of $\nu_2(\rvy\pl)$ agrees with the scaling found for the corresponding quantity in \citet{Yang17}. 

Now let us see whether $\delta\umin$, $\delta\umax$ exist. In the feedforward evolution of \Section{S:vanilla_nets}, \theorem{T:S_increments_vanilla} directly ensures that $1 \lesssim \delta\S\pl \lesssim \sqrt{2}$. 

Again since the evolution inside residual units is well approximated by the feedforward evolution of \Section{S:vanilla_nets}, we deduce that $\delta \S\plh$ is bounded from above and below.

\emph{In summary, the existence of $\mu\udmax$ is not ensured and \theorem{T:S_increments_bn_res} does not apply. No conclusion can be made regarding the growth of $\S\pl$}.

\subsubsection{Case $\phi=\relu$, with Batch Normalization}
As in \Section{S:case_tanh_BN}, the existence of $\mu\udmin$, $\mu\udmax$ is ensured by the fact that batch normalization controls signal variance: $\mu_2(\rvz\plH)=1$ and that $\rvy\plH$ is obtained from $\rvz\plH$ only after a single nonlinearity step and a single convolution step.

Now let us see whether $\delta\umin$, $\delta\umax$ exist and again let us first reason in the feedforward evolution of \Section{S:BN_feedforward_nets}. Similarly to \Section{S:case_tanh_BN}, the term $\exp( \overline{m}\uBN[ \S\pl ] )$ is bounded: (i) from above by considering the signal with worst possible conditioning; (ii) from below by $1$. 

\theorem{T:S_increments_bn_ff} further ensures that $1 \leq \exp(\overline{m}\uphi[\S\pl])\leq \sqrt{2}$. 

Since the evolution inside residual units is well approximated by the feedforward evolution of \Section{S:BN_feedforward_nets}, we deduce that $\delta \S\plh$ is bounded from above and below. 

\emph{In summary, \theorem{T:S_increments_bn_res} applies and $\S\pl$ has power-law growth}.

\end{document}